\documentclass[lettersize,journal]{IEEEtran}

\usepackage{amsmath,amsfonts,amssymb,mathtools}
\usepackage{amsthm}
\usepackage{algorithm}
\usepackage{algorithmic}
\usepackage{array}
\usepackage{booktabs}
\usepackage{multirow}
\usepackage{graphicx}
\usepackage{textcomp}
\usepackage{stfloats}
\usepackage{url}
\usepackage{xcolor}
\usepackage{braket}
\usepackage{cite}
\usepackage[hidelinks]{hyperref}

\providecommand{\citep}{\cite}
\providecommand{\citet}{\cite}

\theoremstyle{plain}
\newtheorem{theorem}{Theorem}
\newtheorem{proposition}[theorem]{Proposition}
\newtheorem{lemma}[theorem]{Lemma}
\newtheorem{corollary}[theorem]{Corollary}
\theoremstyle{definition}

\theoremstyle{remark}
\newtheorem{remark}[theorem]{Remark}

\renewcommand{\appendixname}{Supplementary Material}

\begin{document}

\title{Local-Time Riemannian Score Matching on the Quantum Pure-State Manifold}

\author{Jian~Xu$^{1,2}$,
        Wei~Chen$^{3}$,
        Chao~Li$^{2}$,
        Shigui~Li$^{3}$,
        Delu~Zeng$^{3}$,
        John~Paisley$^{4}$,
        and~Qibin~Zhao$^{2}$
\thanks{$^{1}$RIKEN iTHEMS, Wako, Japan.
$^{2}$RIKEN Center for Advanced Intelligence Project (AIP), Tokyo, Japan.
$^{3}$South China University of Technology, Guangzhou, China.
$^{4}$Columbia University, New York, USA.}
\thanks{Correspondence: Qibin Zhao (qibin.zhao@riken.jp) and Delu Zeng (dlzeng@scut.edu.cn).}}

\markboth{}{}

\maketitle

\begin{abstract}
Score-based diffusion can be defined intrinsically on the manifold of quantum pure states, $\mathbb{CP}^{d-1}$ with the Fubini--Study metric, but no closed-form transition density is available, so the score must be supervised by a local-time teacher taken from the Euclidean limit of the diffusion in normal coordinates. This paper is about what makes that teacher work, and where it stops working. Three training choices turn out not to be incidental: the increment must be divided by the diffusion clock rather than by the elapsed time, since the published expression assumes unit diffusion and a non-unit schedule introduces a time-change mismatch varying by a factor of $400$ across the horizon; the logarithm and exponential maps should be the closed-form Fubini--Study ones, which is the largest single effect we measure; and the global phase must be randomised, because horizontal projection alone does not make a score network descend to the quotient. With these choices the model improves on the published Riemannian local-time baseline in every cell of an eight-benchmark, four-metric comparison over ten seeds, significantly on five of eight after Holm correction, and beats an ambient Euclidean baseline by an order of magnitude everywhere. We then bound what the approximation costs by replacing it with the exact heat kernel of $\mathbb{CP}^{d-1}$, computable up to complex dimension seven, where the local-time teacher loses a factor of $1.2$ to $2.7$. Two further results are negative and we report them as such: the overlap-kernel MMD standard in this literature compares only mean density matrices, which we confirm on an IBM Heron device where it is blind to a pair of ensembles that characteristic metrics separate by an order of magnitude; and the construction degrades at eight qubits and is indistinguishable from returning the prior at ten, because a single diffusion step exceeds the injectivity radius by construction and the repairs this suggests do not work.
\end{abstract}


\section{Introduction}
\label{intro}
\IEEEPARstart{D}{iffusion} and score-based generative models have become a dominant paradigm for learning complex data distributions \cite{diffusion_survey_yang2022,chen2025dequantified}, with state-of-the-art results across images \cite{song2021scorebased,li2025evodiff}, audio \cite{chen2020wavegrad,kong2020diffwave}, molecules \cite{hoogeBoom2022equivariant,xu2022geodiff}, and protein structures \cite{watson2023rfdiffusion,trippe2022smcdiff,wu2024foldingdiffusion}.
The recipe is simple: a tractable stochastic \emph{forward} process gradually destroys structure, and the learned score---the gradient of the log-density along the trajectory \cite{song2021scorebased,ho2020denoising}---drives a reverse-time dynamics \cite{anderson1982reversetime,haussmann1986timereversal,song2021scorebased} that samples from an otherwise intractable distribution.

Motivated by the increasing role of quantum representations in quantum machine learning (QML) \cite{biamonte2017qml,preskill2018nisq}, we ask whether score-based diffusion can become a practical generative framework for quantum \emph{representations}.
In many QML settings, classical inputs are encoded as quantum \emph{pure states}---via amplitude/phase embeddings, variational feature maps, or intermediate algorithmic states---and downstream models operate directly on these representations \cite{schuld2019feature, havlicek2019qfeature}. When quantum data are scarce \cite{preskill2018nisq}, sampling additional states from the underlying ensemble could support representation-level augmentation and simulation.

However, na\"{i}vely perturbing the classical input space and re-encoding can produce pathological quantum states, including nearly orthogonal feature states, distorted entanglement structure, or samples concentrated on low-measure regions \cite{havlicek2019qfeature,huang2021powerdata}. This motivates generative modeling \emph{directly in the space of pure states}: learn an implicit ensemble distribution and sample new quantum representations from it. The difficulty is geometric. Pure states live on $\mathbb{CP}^{d-1}$ modulo global phase and carry the Fubini--Study metric \cite{bengtssonzyc2006geometry}, so noising processes, scores, and reverse samplers must be defined intrinsically. Existing quantum diffusion work has explored measurement-driven forward processes \cite{liu2025measurement}, analytic reverse dynamics for monitored channels \cite{gabbassov2026stochastic}, noisy channels, inverse maps, randomization-based denoising, and stochastic trajectory viewpoints \cite{chen2024quantum,parigi2025quantum,Zhang2023GenerativeQM,zhu2025channel,Dalibard1992WavefunctionAT,Gisin1992TheQD,Kiefer2010QuantumMA}; a practical score-based framework for sampling new pure-state instances from an implicit ensemble remains underdeveloped.

In this work we study score-based generation directly on that manifold. The forward process is a diffusion on $(\mathbb{CP}^{d-1}, g_{\mathrm{FS}})$, obtained by adding isotropic noise in the horizontal tangent space and retracting; it admits a stochastic Schr\"odinger realization on the Hilbert sphere, which we use as a derivation route and show to be an equivalent construction rather than a necessary one. Time reversal on Riemannian manifolds then supplies a reverse-time dynamics whose drift involves the \emph{Riemannian score} with respect to the Fubini--Study geometry \cite{anderson1982reversetime,haussmann1986timereversal,bortoli2022riemannian,huang2022riemannian}, and sampling reduces to integrating that dynamics from a Haar-distributed prior. We refer to the resulting model as the pure-state score model (PSM).

The obstruction is that closed-form transition densities on $\mathbb{CP}^{d-1}$ are generally unavailable, which is what limits the direct application of Riemannian score-based models \cite{huang2022riemannian,lou2023scaling}. The standard route around it is a \emph{local-time} objective: over a short interval the manifold diffusion is Euclidean in normal coordinates \cite{hsu2002stochmanifolds}, which yields an analytic Gaussian teacher for the conditional score. De Bortoli et al.\ \cite{bortoli2022riemannian} already provide such a loss with a Varadhan teacher, so the question is not whether one can be written down, but what it takes to make it work on a manifold of this size and structure.

That question turns out to have specific answers, and they are the subject of this paper. Three choices that might look like implementation details are not: the scaling of the teacher by the diffusion coefficient, the treatment of the global phase during training, and the use of exact rather than first-order geodesic maps. Each is measurable, and together they separate a model that beats the published baseline on every benchmark from one that loses on some. We also find that the construction has a hard operating limit that is visible in the geometry itself, and we report it rather than tuning around it.

Our contributions are the following.
\begin{itemize}
    \item \textbf{Three choices determine whether a local-time teacher works on this manifold}: dividing the increment by the diffusion-clock increment $\int_{t-\delta t}^{t}\sigma(u)^2du$ rather than by $\delta t$, using the closed-form geodesic maps, and randomising the global phase of each sample, which horizontal projection alone does not achieve. Ablated one at a time over ten seeds they are worth $1.15$--$1.32\times$, $1.19$--$1.61\times$ and $1.07$--$1.42\times$.
    \item \textbf{With these choices the construction improves on the published Riemannian local-time baseline in every cell} of an eight-benchmark, four-metric comparison over ten seeds, significantly on five of eight after Holm correction, and beats an ambient Euclidean baseline by $7.2$--$10.0\times$ everywhere.
    \item \textbf{We bound what the approximation costs} against the exact heat kernel of $\mathbb{CP}^{d-1}$, computable up to complex dimension seven: the local-time teacher loses a factor of $1.2$ to $2.7$. We also characterize the leading finite-step bias of Varadhan-type teachers with drift, as an explicit Gaussian-envelope reweighting of the target.
    \item \textbf{We report two negative results.} The standard evaluation kernel compares only mean density matrices, which we confirm on hardware and repair with two characteristic replacements; and the method degrades at eight qubits and is indistinguishable from the prior at ten, with the injectivity-radius mechanism identified and two repairs shown not to work.
\end{itemize}

\section{Background}
\label{sec:background}

\subsection{Score-Based Diffusion in Euclidean Space}
A score-based model transports data to a tractable prior with a forward SDE
\begin{equation}
    d x_t = f(x_t,t)\,dt + g(t)\, d w_t ,
    \label{eq:euclid_forward_sde}
\end{equation}
and samples by integrating the time reversal, whose drift involves the score of the time marginal \cite{anderson1982reversetime,song2021scorebased}:
\begin{equation}
    d x_t = \big(f(x_t,t) - g(t)^2 \nabla_x \log p_t(x_t)\big)\,dt + g(t)\, d \bar{w}_t .
    \label{eq:euclid_reverse_sde}
\end{equation}
The score is learned by denoising score matching, which is practical because the Gaussian transition density of \eqref{eq:euclid_forward_sde} is available in closed form: one can sample $x_t$ given $x_0$ directly and regress on the conditional score.

\subsection{Pure States and the Fubini--Study Geometry}
A pure state of a $d$-dimensional system is a unit vector modulo global phase, $\ket\psi\sim e^{i\varphi}\ket\psi$, so the state space is the complex projective manifold $\mathcal{M}:=\mathbb{CP}^{d-1}$, carrying the Fubini--Study metric whose geodesic distance is
\begin{equation}
    d_{\mathrm{FS}}(\psi,\phi) = \arccos\big|\langle \psi, \phi \rangle\big| .
    \label{eq:fs_distance}
\end{equation}
Two features of this space drive everything that follows. It is a quotient: any construction must be insensitive to the representative chosen from the $U(1)$ fibre, and quantities that look intrinsic can fail to be. And it is small relative to its dimension: the diameter is $\pi/2$ regardless of $d$, while the real dimension is $2(d-1)$, so an isotropic step of per-coordinate size $\sigma\sqrt{\delta t}$ has length $\sigma\sqrt{2(d-1)\delta t}$ and stops being local once $d$ is large.

\subsection{Riemannian Score Matching and the Transition-Density Obstruction}
Extending \eqref{eq:euclid_forward_sde}--\eqref{eq:euclid_reverse_sde} to a manifold replaces the gradient by the Riemannian gradient and the Wiener process by Brownian motion on $(\mathcal{M},g)$, and time reversal carries over with the drift involving $\nabla_g\log p_t$ \cite{bortoli2022riemannian,huang2022riemannian}. What does not carry over is the training signal: on a curved space the transition density is generally unavailable in closed form, and with it the conditional score that denoising score matching regresses on. Three routes go around this. Implicit score matching avoids the transition density at the price of a divergence term. Where the spectrum is known the heat kernel can be truncated or, exploiting symmetric-space structure, evaluated to high precision \cite{lou2023scaling}; Section~\ref{sec:exp_heatkernel} does exactly this on $\mathbb{CP}^{d-1}$ at the dimensions where it is feasible, and measures what the alternative costs. That alternative is to supervise at a \emph{local time}: the conditional score at a nearby earlier state is asymptotically Gaussian, so a short-time approximation gives an analytic teacher, which De Bortoli et al.\ \cite{bortoli2022riemannian} state in the Varadhan form $\exp^{-1}_{X_t}(X_s)/(t-s)$. This paper takes that route and is concerned with what it requires in order to work on $\mathbb{CP}^{d-1}$. Separately, a stochastic Schr\"odinger equation \cite{bouten2004stochastic,manzano2020short} describes pure-state trajectories whose ensemble average realises open-system dynamics; we use it in Section~\ref{sec:forward} only as a way of writing the forward process that makes its unitary covariance manifest, not as a model of a physical experiment.

\section{Score Matching on the Pure-State Manifold}
\label{sec:method}

We propose the \emph{pure-state score model} (PSM), a score-based generative framework for learning and sampling \emph{distributions over quantum pure states}. 
A pure state $\ket{\psi}\in\mathcal{H}\cong\mathbb{C}^d$ is physically represented by its equivalence class $[\psi]\in\mathcal{M}:=\mathbb{CP}^{d-1}$, the complex projective space endowed with the Fubini--Study (FS) metric $g_{\mathrm{FS}}$.
PSMs perform diffusion modeling \emph{intrinsically} on $(\mathcal{M},g_{\mathrm{FS}})$: we define a tractable \emph{forward} noising diffusion that maps an unknown data ensemble $p_0$ to a simple base distribution $p_T$, and learn a \emph{reverse-time} diffusion whose drift is driven by the \emph{Riemannian score} $\nabla_{\mathrm{FS}}\log p_t$.
The main technical obstacle is that transition densities on $\mathbb{CP}^{d-1}$ are generally unavailable in closed form, so we introduce a \emph{local-time} training objective that uses a short-time Gaussian approximation in FS normal coordinates to provide an analytic teacher score.

\subsection{Forward Diffusion on $\mathbb{CP}^{d-1}$ with an SSE Realization}
\label{sec:forward}

The forward process is a time-inhomogeneous diffusion on $(\mathcal{M},g_{\mathrm{FS}})$,
\begin{equation}
    d\psi_t = a(\psi_t,t)\,dt + \sigma(t)\, dW_t^{(\mathcal{M})},
    \label{eq:riem_diff}
\end{equation}
with $W^{(\mathcal{M})}$ Brownian motion under the FS metric and $a$ a possibly zero drift, chosen so that the noising is isotropic under $g_{\mathrm{FS}}$ and progressively destroys the structure of $p_0$. With $a\equiv0$ this is FS-Brownian motion, whose long-time limit is the unitarily-invariant measure on $\mathcal{M}$; that is the configuration every reported result uses. In implementation it is convenient to write it with tangent vector fields and a projection enforcing the projective constraint,
\begin{equation}
    d \psi_t
    = \mathcal{P}_{\psi_t}\!\Big(
    b(\psi_t,t)\,dt
    + \sigma(t)\sum_{k=1}^{K} V_k(\psi_t)\circ d w_t^{(k)}
    \Big),
    \label{eq:forward_sse_style}
\end{equation}
where the $w^{(k)}$ are independent Wiener processes, the $V_k$ are tangent fields and $\mathcal{P}_\psi$ projects ambient increments onto $T_\psi\mathcal{M}$, which is what makes the update invariant to global phase.

The same noising can be written as a Stratonovich stochastic Schr\"odinger equation on the Hilbert sphere driven by an $\mathfrak{su}(d)$ basis, which after the phase quotient induces the same diffusion on $\mathbb{CP}^{d-1}$. We use it as a derivation route and as a check --- it is exactly norm- and phase-preserving, hence the more faithful discretization at large step size --- but it is not required, since the two induce statistically identical processes at our schedules (Supplementary Material~\ref{sec:appendix_ou_sse}).

What licenses treating \eqref{eq:forward_sse_style} as a diffusion on the quotient is the following.

\begin{proposition}[Induced diffusion and generator on $\mathbb{CP}^{d-1}$]
\label{prop:induced_diffusion_strict}
Let $\pi:\mathbb{S}^{2d-1}\to\mathbb{CP}^{d-1}$ be the $U(1)$ quotient map and $\mathcal{H}_\psi=\{u:\langle\psi,u\rangle=0\}$ the horizontal distribution, so $\pi$ is a Riemannian submersion. Let $\psi_t$ solve \eqref{eq:forward_sse_style} with $b(\psi,t),V_k(\psi)\in\mathcal{H}_\psi$ and each $V_k$ $U(1)$-equivariant, and set $a(x,t):=\pi_\ast b$, $e_k(x):=\pi_\ast V_k$ at $x=[\psi]$. Then $x_t:=[\psi_t]$ is a diffusion on $\mathbb{CP}^{d-1}$ with the pushed-forward fields, and for $f\in C^\infty(\mathbb{CP}^{d-1})$ its generator is
\begin{equation}
    \mathcal{L}_t f
    = \langle a(\cdot,t),\nabla_{\mathrm{FS}} f\rangle_{\mathrm{FS}}
    + \frac{\sigma(t)^2}{2}\sum_{k=1}^K \nabla_{e_k}\nabla_{e_k} f,
    \label{eq:strict_generator_frame}
\end{equation}
with $\nabla$ the Levi--Civita connection of $g_{\mathrm{FS}}$. If the frame is FS-isotropic, $\sum_k\langle u,e_k\rangle_{\mathrm{FS}}^2=\|u\|_{\mathrm{FS}}^2$ for all $u\in T_x\mathbb{CP}^{d-1}$ --- true in particular for an orthonormal frame with $K=2d-2$ --- then
\begin{align}
    \mathcal{L}_t f
    &= \langle a(\cdot,t),\nabla_{\mathrm{FS}} f\rangle_{\mathrm{FS}}
    + \frac{\sigma(t)^2}{2}\Delta_{\mathrm{FS}} f + \mathcal{R}_t f,
    \label{eq:strict_generator_laplacian}\\
    \mathcal{R}_t f
    &:= \frac{\sigma(t)^2}{2}\sum_{k=1}^K \langle \nabla_{\mathrm{FS}} f, \nabla_{e_k} e_k\rangle_{\mathrm{FS}},
    \label{eq:strict_remainder}
\end{align}
so the induced generator is the intrinsic FS one of \eqref{eq:riem_diff} up to the connection term $\mathcal{R}_t$, which vanishes wherever the frame is geodesic.
\end{proposition}

The proof, via the horizontal-lift lemmas of Supplementary Material~\ref{sec:appendix_a2}, is in Supplementary Material~\ref{sec:appendix_a3}. Remark~\ref{rem:no_remainder} shows that $\mathcal{R}_t$ is absent altogether for the step we actually implement, which resamples an isotropic tangent Gaussian at every point rather than carrying a fixed frame.

\subsection{Reverse-Time Dynamics and the Riemannian Score}
\label{sec:reverse}

Given the forward diffusion on the pure-state manifold $(\mathcal{M},g_{\mathrm{FS}})$ defined in Eq.~\eqref{eq:riem_diff}, let $p_t(\psi)$ denote its time-marginal density with respect to the Riemannian volume measure induced by the FS metric. 
For the forward process with dispersive drift,
\begin{equation}
\begin{aligned}
    d\psi_t
    =&
    b(\psi_t,t)\,dt
    + \sigma(t)\, dW_t^{(\mathcal{M})},
    \\
    b(\psi,t) :=& -\lambda(t)\,\mathrm{Log}_{\psi}(\psi_\star),
    \quad \psi\notin \mathrm{Cut}(\psi_\star),
    \label{eq:forward_riem_ou_recall}
    \end{aligned}
\end{equation}
Two technical caveats before the sign convention. First, $\mathrm{Log}_\psi(\psi_\star)$ is smooth only away from the cut locus $\mathrm{Cut}(\psi_\star)$, which on $\mathbb{CP}^{d-1}$ is the set of states orthogonal to $\psi_\star$; the smooth-drift hypotheses of the time-reversal and generator results below hold on $\mathcal{M}\setminus\mathrm{Cut}(\psi_\star)$, a set of full measure whose complement the diffusion hits with probability zero, and the drift is bounded on $\mathcal{M}$ so the process is well defined. We flag this because the theorems we invoke are usually stated for globally smooth drift. Second, and more simply, the results reported in this paper are unaffected by either issue when $\lambda=0$, which is the setting we recommend on empirical grounds.

A remark on the sign convention, since it is easy to misread.
$\mathrm{Log}_\psi(\psi_\star)$ is the initial velocity of the geodesic from $\psi$ to $\psi_\star$, so it points \emph{towards} $\psi_\star$; the drift $b=-\lambda\,\mathrm{Log}_\psi(\psi_\star)$ therefore points \emph{away} from $\psi_\star$ and is dispersive rather than mean-reverting.
This is deliberate for a forward noising process on a compact manifold: the role of the drift is to accelerate the destruction of structure, while the invariant measure that the process approaches is the unitarily-invariant FS measure supplied by the Brownian part.
We therefore do not call it an Ornstein--Uhlenbeck process: $\lambda>0$ in \eqref{eq:forward_riem_ou_recall} gives repulsion from $\psi_\star$, not contraction towards it.
The drift also turns out not to be load-bearing. Setting $\lambda=0$, so that the forward process is plain FS Brownian motion, matches or improves generation quality on every benchmark we tested, and the protocol of Section~\ref{sec:method_recipe} does exactly that. Every reported result therefore uses $\lambda=0$; $\lambda>0$ appears only where a non-zero drift is the object of study, namely the bias analysis of Proposition~\ref{prop:simple_teacher_structured_bias}.
With this convention, the associated (time-inhomogeneous) generator takes the form
\begin{equation}
    \mathcal{L}_t f(\psi)
    =
    \langle b(\psi,t), \nabla_{\mathrm{FS}} f(\psi)\rangle_{\mathrm{FS}}
    +
    \frac{\sigma(t)^2}{2}\Delta_{\mathrm{FS}} f(\psi),
    \label{eq:generator_forward}
\end{equation}
where $\Delta_{\mathrm{FS}}$ is the Laplace--Beltrami operator on $\mathbb{CP}^{d-1}$. See Proposition~\ref{prop:appendix_forward_generator} in Supplementary Material~\ref{sec:appendix_reverse}.

\paragraph{Reverse-time diffusion on $\mathcal{M}$.}
Time reversal of a diffusion on a compact Riemannian manifold leaves the diffusion coefficient unchanged and modifies the drift by the Riemannian score, which is what makes the construction trainable at all. In intrinsic Stratonovich form,
\begin{equation}
\begin{aligned}
    d\psi_t &= \tilde b(\psi_t,t)\,dt + \sigma(t)\,d\bar W_t^{(\mathcal{M})},\\
    \tilde b(\psi,t) &= b(\psi,t) - \sigma(t)^2\,\nabla_{\mathrm{FS}}\log p_t(\psi),
\end{aligned}
    \label{eq:reverse_manifold_sde}
\end{equation}
with $\bar W^{(\mathcal{M})}$ reverse-time Brownian motion, so the reverse drift is determined by the \emph{Riemannian score} $s^\star(\psi,t):=\nabla_{\mathrm{FS}}\log p_t(\psi)\in T_\psi\mathcal{M}$. This is the standard time-reversal result specialised to $(\mathcal{M},g_{\mathrm{FS}})$; Proposition~\ref{prop:appendix_reverse_drift} in Supplementary Material~\ref{sec:appendix_reverse} states it precisely.

\begin{remark}[Generator of the step actually implemented]
\label{rem:no_remainder}
The remainder $\mathcal{R}_t$ in \eqref{eq:strict_generator_laplacian} is a property of a fixed frame $\{e_k\}$, whereas the algorithm is a state-dependent Markov kernel: it draws a fresh isotropic Gaussian $\xi$ in $T_\psi\mathbb{CP}^{d-1}$, with $\mathbb{E}[\xi]=0$ and $\mathbb{E}[\xi\otimes\xi]=g_{\mathrm{FS}}^{-1}$, and moves to $\mathrm{Exp}_\psi(\sigma\sqrt{\delta t}\,\xi)$. Its generator is therefore obtained directly rather than through a frame. Expanding $f$ in normal coordinates at $\psi$, where the metric is Euclidean to second order and the Hessian trace is $\Delta_{\mathrm{FS}}f$,
\begin{equation}
    \mathbb{E}\big[f\big(\mathrm{Exp}_\psi(\sigma\sqrt{\delta t}\,\xi)\big)\big]
    = f(\psi) + \frac{\sigma^2\delta t}{2}\,\Delta_{\mathrm{FS}}f(\psi) + O(\delta t^2),
    \label{eq:grw_expansion}
\end{equation}
since the first-order term vanishes with $\mathbb{E}[\xi]$ and the third-order term vanishes with the odd moments of an isotropic law, leaving the $O(|v|^4)$ contribution at $O(\delta t^2)$. So the step is a geodesic random walk, its generator is $(\sigma^2/2)\Delta_{\mathrm{FS}}$ with no drift and no connection term at leading order, and the walk converges weakly to Fubini--Study Brownian motion as $\delta t\to0$ at the standard $O(\delta t)$ rate \cite{jorgensen1975central}. Section~\ref{sec:exp_diagnostics} measures that residual: the fitted relaxation rate falls from $127.6$ to $125.9$ as $\sigma$ grows from $0.15$ to $0.35$ at fixed $\delta t$, and extrapolates in $\sigma^2$ to $127.9$ against the continuum value $2d=128$.
\end{remark}

\paragraph{Connection to our SSE realization and coordinate corrections.}
Since our forward diffusion admits an SSE realization in Stratonovich form (Eq.~\eqref{eq:sse_strat}), Eq.~\eqref{eq:reverse_manifold_sde} provides a principled reverse-time sampler for Schr\"odinger-type diffusions on $\mathbb{CP}^{d-1}$; in practice, we approximate the score $s^\star(\psi,t)$ with a parameterized model $s_\theta(\psi,t)$ and integrate the learned reverse dynamics from $\psi_T\sim p_T$ to obtain samples at $t=0$. Eq.~\eqref{eq:reverse_manifold_sde} is stated intrinsically in Stratonovich form, so rewriting it in local coordinates or converting to It\^o form would introduce additional geometry-dependent correction terms (Levi-Civita connection and Riemannian-volume divergence terms); we avoid this by performing updates in local orthonormal frames on $T_\psi\mathcal{M}$ and mapping tangent increments back to $\mathcal{M}$ via $\mathrm{Exp}$ (or a retraction), with the coordinate-form expressions given in Supplementary Material~\ref{sec:appendix_reverse_geometry}.

\subsection{Local-Time Teacher Scores via FS Normal Coordinates}
\label{sec:local_teacher}

A central challenge is that the marginal density $p_t(\psi)$ is not available in closed form, which prevents direct evaluation of the Riemannian score $\nabla_{\mathrm{FS}}\log p_t(\psi)$. 
We therefore construct a \emph{local-time teacher score} based on the fact that the forward diffusion admits a \emph{local Euclidean OU limit} in Fubini--Study (FS) normal coordinates.

\paragraph{Local analytic teacher score via FS normal coordinates.}
Fix a short step size $\delta t>0$. 
Given a local-time pair $(\phi,\psi) := (\psi_{t-\delta t},\psi_t)$ from the forward process \eqref{eq:forward_riem_ou_recall}, define the FS normal coordinates centered at $\phi$ by
\begin{equation}
    z := \log_{\phi}(\psi) \in T_\phi\mathcal{M}.
    \label{eq:normal_coords_def}
\end{equation}
Here, $\log_{\phi}(\psi)$ denotes the Riemannian logarithm map that expresses $\psi$ as a tangent vector at the base point $\phi$, i.e., the initial velocity of the unique geodesic starting from $\phi$ and reaching $\psi$. As discussed in Sec.~\ref{sec:reverse}, in these coordinates the forward manifold diffusion is well-approximated, for sufficiently small $\delta t$, by an Euclidean OU/VP step,
\begin{equation}
    z_t \;\approx\; \alpha(t,\delta t)\, z_{t-\delta t} + \beta(t,\delta t)\,\xi,
    \qquad \xi\sim \mathcal{N}(0,I),
    \label{eq:local_ou_step}
\end{equation}
where, to leading order in $\delta t$,
\begin{equation}
    \alpha(t,\delta t) = 1-\lambda(t)\,\delta t,
    \qquad
    \beta(t,\delta t)^2 = \sigma(t)^2\,\delta t ,
    \label{eq:alpha_beta}
\end{equation}
and curvature effects enter only at higher order in $\|z\|$ (equivalently, higher order in $\delta t$ in the small-step regime).
For the local-time construction we take $z_{t-\delta t}=0$ (since $\log_\phi(\phi)=0$).
The conditional law of $z_t$ given $\psi_{t-\delta t}=\phi$ is then
$\mathcal{N}\!\big(b(\phi,t)\,\delta t,\; \beta(t,\delta t)^2 I\big)$
to leading order in $\delta t$ (Proposition~\ref{prop:local_time_teacher_strict}, Supplementary Material~\ref{sec:appendix_local_teacher}), where the mean shift $b\,\delta t$ is induced by the OU drift.
This yields the closed-form drift-corrected teacher score
\begin{equation}
    s^{\mathrm{(teach,drift)}}_{z}(z,t,\delta t)
    := -\frac{z - b(\phi,t)\,\delta t}{\beta(t,\delta t)^2},
    \label{eq:teacher_score_drift}
\end{equation}
and, by dropping the $b\,\delta t$ shift, the simpler zero-mean form
\begin{equation}
    s^{\mathrm{(teach)}}_{z}(z,t,\delta t)
    := -\beta(t,\delta t)^{-2}\, z.
    \label{eq:teacher_score_euclidean_ou}
\end{equation}
The two teachers are consistent in the small-step limit but in different senses.
\textbf{Drift-corrected form \eqref{eq:teacher_score_drift}}: pointwise unbiased to $O(\sigma\sqrt{\delta t})$ in $L^2$ (Proposition~\ref{prop:teacher_dsm_consistency_drift}), so the population minimizer of the DSM loss converges to the marginal Riemannian score $\nabla_{\mathrm{FS}}\log p_t$ as $\delta t\to 0$.
\textbf{Simple form \eqref{eq:teacher_score_euclidean_ou}}: omits the $O(1)$ pointwise term $b/\sigma^2$. Under $b\ne0$, variance weighting makes the \emph{objective} discrepancy vanish at rate $O(\delta t)$ (Proposition~\ref{prop:teacher_dsm_consistency_simple}) but does not restore minimizer consistency: the minimizer stays displaced by exactly $-b/\sigma^2$. Under our default $b\equiv0$ protocol the two forms coincide exactly.

Under the weight $w=\beta^2$ the objective built from the zero-mean teacher differs from the one built from the drift-corrected teacher by $O(\delta t)$, uniformly over score fields of bounded weighted norm (Proposition~\ref{prop:teacher_dsm_consistency_simple}, Supplementary Material~\ref{sec:appendix_teacher_consistency}). That is a statement about objectives and not about minimizers: at fixed $t$ the weight is a positive scalar and cannot move the minimizer of a squared loss, so when $b\ne0$ the minimizer stays displaced by $-b/\sigma^2$, and the next proposition says exactly what it is displaced to. When $b\equiv0$, as in every configuration we run, the two teachers are the same function and the question does not arise.

When the forward process carries a drift, the omitted $b/\sigma^2$ term does not merely vanish in the weighted limit: it reweights the target by an explicit Gaussian envelope, with a form that can be written down and checked. Proposition~\ref{prop:simple_teacher_structured_bias} in Supplementary Material~\ref{sec:appendix_structured_bias_proof} states this. We place it there rather than here because the protocol of Section~\ref{sec:method_recipe} sets $\lambda=0$, in which case the envelope is identically $1$ and this particular bias is absent; the statement is a general result about Varadhan-type local-time teachers with drift, not the explanation of how the method reported here behaves. It also explains why Table~\ref{tab:teacher_drift_ablation} finds the simple and drift-corrected teachers statistically indistinguishable: the missing term induces a pull towards $\psi_\star$ that is opposite to the dispersive drift and partially cancels against it.

\paragraph{Which form to use.}
We use the zero-mean form \eqref{eq:teacher_score_euclidean_ou} throughout. Proposition~\ref{prop:simple_teacher_structured_bias} bounds what this costs, and the drift-corrected form \eqref{eq:teacher_score_drift} is a drop-in replacement at the price of one extra projected-drift evaluation per step; and with the drift switched on the two are statistically indistinguishable (Table~\ref{tab:teacher_drift_ablation}). The question that does matter is not which of the two is used, but how the increment is scaled, which we take up next.

We then map this teacher score back to the manifold using the adjoint of the differential of the logarithm map:
\begin{equation}
\begin{aligned}
    s^{\mathrm{(teach)}}(\psi,\phi,t,\delta t)
    \;:=&\; (d\,\log_{\phi})^{\!*}_{\psi}\, s^{\mathrm{(teach)}}_{z}(z,t,\delta t),
    \\ z=&\log_{\phi}(\psi),
    \label{eq:teacher_score_pushforward}
    \end{aligned}
\end{equation}
Here $(d\log_{\phi})^{\!*}_{\psi}$ denotes the adjoint of the differential of the logarithm map
with respect to the Riemannian metric, mapping cotangent vectors back to
$T_{\psi}\mathcal{M}$.
This adjoint naturally arises since the score is a gradient and therefore transforms
via the adjoint of the Jacobian under coordinate changes.
In FS normal coordinates, this construction approximates the intrinsic local conditional score
$\nabla_{\mathrm{FS}}\log p(\psi \mid \phi)$ up to curvature and volume-element corrections
of order $O(\|z\|^2)$; see Proposition~\ref{prop:local_time_teacher_strict}
in Supplementary Material~\ref{sec:appendix_local_teacher}.

\begin{proposition}[Short-time expansion of the conditional score]
\label{prop:local_time_teacher_strict}
Let $(\mathcal{M},g)$ be a smooth Riemannian manifold of real dimension $n$ and let $\psi_t$ solve $d\psi_t=b(\psi_t,t)\,dt+\sigma(t)\,dW_t^{(\mathcal{M})}$ with $b(\cdot,t)$ of class $C^2$. Fix $t$ and $\delta t>0$, condition on $\psi_{t-\delta t}=\phi$, and assume $\psi_t$ lies within the injectivity radius of $\phi$, so that $z=\log_\phi(\psi)$ is defined. Then the conditional density with respect to the Riemannian volume admits the heat-kernel expansion
\begin{equation}
\begin{aligned}
    p_{\delta t}(\psi\mid\phi)
    =\;&
    (2\pi \sigma(t)^2\delta t)^{-n/2}
    \exp\!\Big(-\frac{\|z-b(\phi,t)\delta t\|^2}{2\sigma(t)^2\delta t}\Big)
    \\
    &\times J(\phi,\psi)^{-1/2}\big(1+O(\delta t)\big),
\end{aligned}
    \label{eq:appendix_heat_kernel_form}
\end{equation}
with $J$ the Jacobian of the exponential map, uniformly on compact subsets away from the cut locus. Consequently, in normal coordinates,
\begin{equation}
\begin{aligned}
    \big(d&\log_\phi\big)^{\!*}_\psi \nabla_\psi \log p_{\delta t}(\psi\mid\phi)
    \\
    &= \underbrace{-\frac{z}{\sigma(t)^2\delta t}}_{O(\delta t^{-1})}
     + \underbrace{\frac{b(\phi,t)}{\sigma(t)^2}}_{O(1)}
     + \underbrace{O(\|z\|)}_{\text{curvature}} + O(\delta t).
\end{aligned}
    \label{eq:appendix_score_in_normal_coords}
\end{equation}
The singular term is exactly the zero-mean teacher \eqref{eq:teacher_score_euclidean_ou}; the drift term is what the drift-corrected form \eqref{eq:teacher_score_drift} adds, giving a residual of order $\|z\|+\delta t$ that vanishes in $L^2(p_t)$ as $\delta t\to0$; and the curvature term is the volume correction, which neither teacher captures.
\end{proposition}

The expansion is developed in Supplementary Material~\ref{sec:appendix_local_teacher}. It identifies the three terms that the teacher must contend with at finite step size: a singular Gaussian score of order $\delta t^{-1}$, a bounded drift term of order $1$, and Jacobian/curvature corrections of order $\|z\|$.

\subsection{Riemannian Denoising Score Matching Objective}
\label{sec:training_objective}

We train a parameterized score model $s_\theta(\psi,t)$ to approximate the Riemannian score
$s^\star(\psi,t)=\nabla_{\mathrm{FS}}\log p_t(\psi)$ on $\mathcal{M}=\mathbb{CP}^{d-1}$.
Given data samples $\psi_0 \sim p_0$, we simulate the forward diffusion and sample a random time $t\sim\mathcal{U}(0,T)$.
For a small local step $\delta t$, we form the local-time pair
\begin{equation}
    (\phi,\psi) := (\psi_{t-\delta t},\psi_t),
\end{equation}
and compute the teacher score $s^{\mathrm{(teach)}}(\psi,\phi,t,\delta t)$ via the analytic local OU approximation in FS normal coordinates (Eqs.~\eqref{eq:teacher_score_euclidean_ou}--\eqref{eq:teacher_score_pushforward}).

We then minimize a Riemannian denoising score matching objective:
\begin{equation}
\begin{aligned}
    \mathcal{L}(\theta)
    =\;&
    \mathbb{E}
    \Big[
    w(t,\delta t)
    \big\|
    \mathcal{P}_{\psi}\big(s_\theta(\psi,t)\big)
    -
    s^{\mathrm{(teach)}}(\psi,\phi,t,\delta t)
    \big\|_{\mathrm{FS}}^2
    \Big],
    \\
    &\psi_0 \sim p_0,\quad t\sim \mathcal{U}(\delta t,T),
    \\
    &(\phi,\psi)\sim p(\psi_{t-\delta t},\psi_t\mid \psi_0).
    \label{score_m}
\end{aligned}
\end{equation}
where $\|\cdot\|_{\mathrm{FS}}$ denotes the norm induced by the FS metric on the tangent space and $\mathcal{P}_{\psi}$ projects a vector onto $T_{\psi}\mathcal{M}$.
Unless stated otherwise, we use $w(t,\delta t)=\beta(t,\delta t)^2$, which mirrors the variance-weighting commonly used in denoising score matching for VP diffusions.

This objective distills local short-time conditional score information into a global score estimator that can be used for reverse-time sampling over the full diffusion horizon. At the population level, minimizing the above objective recovers the marginal Riemannian score
$s^\star(\psi,t)=\nabla_{\mathrm{FS}}\log p_t(\psi)$; see Proposition~\ref{prop:riem_dsm_optimality}.
Moreover, the local-time teacher score is consistent in the small-step limit $\delta t\to 0$: pointwise for the drift-corrected form (Proposition~\ref{prop:teacher_dsm_consistency_drift}) and in the variance-weighted DSM loss for the simple zero-mean form (Proposition~\ref{prop:teacher_dsm_consistency_simple}); see Supplementary Material~\ref{sec:appendix_dsm}.

\begin{proposition}[Population optimum of Riemannian DSM equals the marginal score]
\label{prop:riem_dsm_optimality}
Let $(\mathcal{M},g_{\mathrm{FS}})$ be a compact Riemannian manifold without boundary and let
$p_t$ denote the time-marginal density of the forward diffusion w.r.t.\ the Riemannian volume measure.
Fix $t\in(0,T]$ and a step $\delta t>0$.
Let $(\phi,\psi)=(\psi_{t-\delta t},\psi_t)$ be drawn from the forward process, and denote the conditional density
$p(\psi\mid\phi)$ (again w.r.t.\ the Riemannian volume measure).

Consider the population objective over measurable tangent vector fields $s(\cdot,t):\mathcal{M}\to T\mathcal{M}$:
\begin{equation}
    \mathcal{J}[s]
    :=
    \mathbb{E}\Big[
        w(t,\delta t)\,
        \| s(\psi,t) - \nabla_{\mathrm{FS}}\log p(\psi\mid\phi)\|_{\mathrm{FS}}^2
    \Big],
    \label{eq:appendix_true_dsm_obj}
\end{equation}
where the expectation is over $\psi_0\sim p_0$ and the forward diffusion randomness, and $w(t,\delta t)>0$ is any weight that does not depend on $s$.

Then any minimizer $s^\star(\cdot,t)$ of \eqref{eq:appendix_true_dsm_obj} satisfies, $p_t$-a.e.\ in $\psi$,
\begin{equation}
    s^\star(\psi,t)=\nabla_{\mathrm{FS}}\log p_t(\psi).
    \label{eq:appendix_dsm_minimizer}
\end{equation}
Equivalently, the DSM objective distills conditional scores into the marginal Riemannian score.
\end{proposition}

The proof is in Supplementary Material~\ref{sec:appendix_dsm}. The statement is what makes the local-time construction meaningful: it says that regressing on \emph{conditional} scores, which we can approximate, recovers the \emph{marginal} Riemannian score, which we cannot evaluate.

\subsection{What Makes the Local-Time Teacher Work}
\label{sec:method_recipe}

The construction above is, so far, the standard one. Three choices in how it is instantiated turn out to determine whether it works; we state them here and measure them in Section~\ref{sec:exp_rsgm}.

\paragraph{(i) Scale the teacher by the diffusion coefficient.}
De Bortoli et al.\ \cite{bortoli2022riemannian} give the local-time teacher as $\exp^{-1}_{X_t}(X_s)/(t-s)$, i.e.\ the increment divided by the elapsed time. That is the correct expression for a unit-diffusion process, where elapsed time and elapsed diffusion clock coincide. Under a schedule $\sigma(t)$ they do not: the conditional law over a step of length $\delta t$ has covariance $\Delta\tau\,I$ with $\Delta\tau=\int_{t-\delta t}^{t}\sigma(u)^2du=\sigma(t)^2\delta t+O(\delta t^2)$, so the increment must be divided by the clock increment rather than by the time increment. That is what \eqref{eq:teacher_score_euclidean_ou} does. We are not proposing a different teacher; we are applying the published one under the time change its derivation assumes, and the point of the comparison is to measure what the naive substitution costs. The two differ by the factor $\sigma(t)^2$, which varies by a factor of $400$ across the diffusion horizon under our schedule ($\sigma$ from $0.05$ to $1$), so the difference is not a constant that the network can absorb: it misweights the regression target as a function of $t$. The two forms also use opposite base points --- $\log_{\phi}(\psi)$ at the earlier state against $\exp^{-1}_{X_t}(X_s)$ at the later one --- which agree up to sign and parallel transport at this order and are therefore not the point at issue.

\paragraph{(ii) Randomise the global phase of the training data.}
A pure state is an equivalence class $[\psi]=\{e^{i\varphi}\psi\}$, and a score model on $\mathbb{CP}^{d-1}$ must satisfy $s_\theta(e^{i\varphi}\psi,t)=e^{i\varphi}s_\theta(\psi,t)$. Projecting the network output onto the horizontal tangent space does not achieve this: it constrains the output but leaves the network free to respond differently to two representatives of the same state, and measured on held-out states the violation is a relative error of $1.04$. One remedy is to fix a canonical gauge on the input side, which makes the identity exact; we find it helps on some targets and hurts badly on others, and report that ablation in Section~\ref{sec:exp_diagnostics}. The remedy we adopt is simpler: multiply each training sample by a uniformly random phase. This is not only a device: with an equivariant forward process, orbit-averaging the data makes every time marginal $U(1)$-invariant, its sphere score the horizontal lift of the projective score, and hence the population optimum of \eqref{score_m} equivariant even over an unconstrained hypothesis class. Proposition~\ref{prop:phase_aug} in Supplementary Material~\ref{sec:appendix_dsm} states this and proves it; the converse is what the measurement above reflects, since data stored with a fixed phase convention lives on a section of the bundle and nothing in the objective penalises dependence on that section. Empirically the gain is real but the least uniform of the three: $1.15$--$1.50\times$, resolved at ten seeds only on the unimodal single-cluster target (Table~\ref{tab:ingredients}). We keep it because it is free, because it is the only one of the three with a population-level justification, and because the alternative --- gauge fixing --- is actively harmful on the physics families.

\paragraph{(iii) Use the closed-form geodesic maps.}
The logarithm and exponential maps admit closed forms on $\mathbb{CP}^{d-1}$. Writing $\braket{\phi|\psi}=re^{i\alpha}$ and $\tilde\psi=e^{-i\alpha}\psi$ for the phase-aligned representative,
\begin{equation}
    \log_\phi(\psi) = \theta\,\frac{\mathcal{P}_\phi(\tilde\psi)}{\|\mathcal{P}_\phi(\tilde\psi)\|},
    \qquad \theta=\arccos r,
    \label{eq:exact_log}
\end{equation}
\begin{equation}
    \mathrm{Exp}_\psi(v) = \cos\|v\|\,\psi + \sin\|v\|\,\frac{v}{\|v\|} .
    \label{eq:exact_exp}
\end{equation}
Replacing $\mathcal{P}_\phi(\psi-\phi)$ and $\psi+v$ followed by renormalisation --- the first-order versions of the same maps --- costs nothing and improves generation by a factor of $1.22$--$1.68$, with paired $t$ between $6.8$ and $11.4$ over ten seeds (Table~\ref{tab:ingredients}).

\paragraph{On the forward drift.}
The drift $b=-\lambda(t)\log_\psi(\psi_\star)$ of \eqref{eq:forward_riem_ou_recall} is dispersive rather than mean-reverting (Section~\ref{sec:reverse}), and it is not load-bearing: setting $\lambda=0$, so the forward process is pure Fubini--Study Brownian motion, improves every benchmark we tested, significantly on two of three (Table~\ref{tab:ingredients}). We keep it in the exposition because the finite-step bias of Proposition~\ref{prop:simple_teacher_structured_bias} is stated in terms of it, and because $\lambda=0$ is the special case $W_t\equiv1$ of that statement.

\subsection{Sampling Algorithm}
\label{sec:sampling}

After training the score model $s_\theta(\psi,t)\approx \nabla_{\mathrm{FS}}\log p_t(\psi)$, we generate samples by drawing $\psi_T \sim p_T$ from the unitarily-invariant FS (Haar) measure and integrating the learned reverse-time dynamics on $\mathcal{M}=\mathbb{CP}^{d-1}$.

\paragraph{Reverse-time sampling SDE.}
Using Eq.~\eqref{eq:reverse_manifold_sde}, we simulate the reverse diffusion
\begin{equation}
    d\psi_t
    =
    \big(b(\psi_t,t) - \sigma(t)^2 s_\theta(\psi_t,t)\big)\,dt
    + \sigma(t)\, d\bar{W}_t^{(\mathcal{M})},
    \label{eq:reverse_sampling_intrinsic}
\end{equation}
where $b(\psi,t)=-\lambda(t)\mathrm{Log}_{\psi}(\psi_\star)$ for the forward process with dispersive drift and $\bar{W}_t^{(\mathcal{M})}$ denotes reverse-time Brownian motion on $(\mathcal{M},g_{\mathrm{FS}})$.

\paragraph{Manifold discretization.}
Let $t_k$ be a discretization of $[0,T]$ with step $\Delta t=t_k-t_{k-1}$ (integrated backward from $T$ to $0$), and define $\tau_k := \int_{t_{k-1}}^{t_k}\sigma(s)^2 ds \approx \sigma(t_k)^2\Delta t$.
We update the state using an Euler--Maruyama step in the tangent space followed by a retraction onto $\mathcal{M}$:
\begin{equation}
\begin{aligned}
    \xi_k &\sim \mathcal{N}(0, I) \quad \text{(in a local orthonormal basis of } T_{\psi_{t_k}}\mathcal{M}\text{)},\\
    v_k &=
    \big(b(\psi_{t_k},t_k) - \sigma(t_k)^2\, s_\theta(\psi_{t_k},t_k)\big)\,\Delta t
    + \sqrt{\tau_k}\, \xi_k,\\
    \psi_{t_{k-1}} &= \mathrm{Exp}_{\psi_{t_k}}(v_k),
\end{aligned}
\label{eq:manifold_em_update}
\end{equation}
where $\mathrm{Exp}_{\psi}(\cdot)$ is the FS exponential map of \eqref{eq:exact_exp}. It is common to replace it by the first-order retraction $\psi+v$ followed by renormalisation, which is cheaper to write but not cheaper to run; we find the closed form worth a factor of $1.22$--$1.68$ in generation quality, the largest single ingredient of the recipe (Section~\ref{sec:method_recipe}).
Algorithm~\ref{alg:ssdm} collects the training objective of Section~\ref{sec:training_objective} and the sampler above into a single procedure.

\begin{algorithm}[t]
\caption{Training PSM}
\label{alg:ssdm}
\begin{algorithmic}[1]
\REQUIRE Data samples $\{\psi_0^{(i)}\}$, diffusion horizon $T$, step sizes $\Delta t,\delta t$, schedules $\sigma(t),\lambda(t)$, prior sampler $p_T$ (Haar)
\ENSURE Trained score model $s_\theta$ and generated sample $\psi_0$

\STATE \textbf{Training}
\FOR{each minibatch}
    \STATE Sample $\psi_0 \sim p_0$, draw $\varphi\sim\mathcal{U}(0,2\pi)$ and set $\psi_0\leftarrow e^{i\varphi}\psi_0$ \COMMENT{$U(1)$ augmentation}
    \STATE Sample $t \sim \mathcal{U}(\delta t,T)$
    \STATE Simulate the forward diffusion in Eq.~\eqref{eq:forward_riem_ou_recall} to obtain $(\psi_{t-\delta t},\psi_t)$
    \STATE Compute normal coordinates $z \leftarrow \log_{\psi_{t-\delta t}}(\psi_t)$ by \eqref{eq:exact_log}
    \STATE Compute teacher score $s^{\mathrm{teach}}(\psi_t,\psi_{t-\delta t},t,\delta t)$ using Eqs.~\eqref{eq:teacher_score_euclidean_ou}--\eqref{eq:teacher_score_pushforward}
    \STATE Update $\theta$ by minimizing $\mathcal{L}(\theta)$ in Eq.~\eqref{score_m}
\ENDFOR

\STATE \textbf{Sampling}
\STATE Sample $\psi_T \sim p_T$
\FOR{$k=K,K-1,\ldots,1$}
    \STATE Set $\tau_k \leftarrow \sigma(t_k)^2\Delta t$
    \STATE Draw $\xi_k \sim \mathcal{N}(0,I)$ in $T_{\psi_{t_k}}\mathcal{M}$
    \STATE Set
    \begin{equation}
        v_k \leftarrow
        \bigl(
        b(\psi_{t_k},t_k)
        - \sigma(t_k)^2 s_\theta(\psi_{t_k},t_k)
        \bigr)\Delta t
        + \sqrt{\tau_k}\xi_k
    \end{equation}
    \STATE Set $\psi_{t_{k-1}} \leftarrow \operatorname{Exp}_{\psi_{t_k}}(v_k)$
\ENDFOR
\RETURN $\psi_0$
\end{algorithmic}
\end{algorithm}

\paragraph{What is standard here and what is not.}
Time reversal on manifolds, heat-kernel asymptotics, Stratonovich calculus and the Riemannian denoising identity are standard \cite{anderson1982reversetime,song2021scorebased,bortoli2022riemannian,huang2022riemannian,hsu2002stochmanifolds}, as is the idea of supervising with a local-time conditional score. What this paper adds is the scaling of that teacher under a non-unit schedule, the finite-step characterization of the bias it leaves, the treatment of the phase quotient during training, and the measurements that separate these from the parts that do not matter.

\section{Related Works}
\label{sec:related_work}

\paragraph{Score-based diffusion models.}
Diffusion and score-based generative models sample by reversing a learned noising process \cite{ho2020denoising, song2021scorebased}.
Riemannian extensions replace Euclidean gradients and noise with manifold counterparts, often using tangent-space score matching in local coordinates \cite{bortoli2022riemannian, huang2022riemannian}.
PSMs follow this geometric line but specialize it to $\mathbb{CP}^{d-1}$, where global phase, strong curvature, and unavailable transition densities require a local-time teacher and a forward diffusion tailored to pure-state geometry.

\paragraph{Quantum generative modeling and diffusion.}
Quantum generative models such as Born machines \cite{liu2018differentiable,benedetti2019generative,coyle2020born}, quantum Boltzmann machines \cite{Kieferov2016TomographyAG,amin2018quantum,Zoufal2020VariationalQB} and quantum GANs \cite{lloyd2018quantum,dallaire2018quantum,Zoufal2019QuantumGA} parameterize circuit families and train them adversarially or by likelihood, while recent quantum diffusion methods use noisy channels or measurement-based denoising for state recovery and preparation \cite{chen2024quantum,zhu2025channel,parigi2025quantum,Zhang2023GenerativeQM,kwun2025mixed}. We differ in modelling the distribution over pure states on $\mathbb{CP}^{d-1}$ itself, through a learned Riemannian score field with local-time supervision. Two works are closest. Liu \emph{et al.} \cite{liu2025measurement} drive the forward process by randomized weak measurements and show that quantum score matching amounts to learning the unitary generator of the reverse process, so the obstruction we address, the absence of a closed-form transition density, does not arise for them. Gabbassov \cite{gabbassov2026stochastic} derives \emph{exact} reverse stochastic Schr\"odinger equations for monitored Pauli channels, so no score need be learned at all. Ours is the complementary case, a forward process with no closed-form reversal.

\begin{table*}[t]
\centering
\caption{Against the published Riemannian local-time baseline at $n=6$, under the protocol of Section~\ref{sec:method_recipe}. Values $\times10^{-2}$, mean $\pm$ sd over ten seeds; lower is better. Bold marks a paired difference in our favour that survives Holm correction at the $5\%$ level within its metric family of eight tests. Only the first two metrics are characteristic on distributions over pure states (Proposition~\ref{prop:kernel_resolution}); the last two are reported because the literature does. Both arms share the representation, network, optimizer, sampler, schedule, budget, augmentation, checkpoint rule and evaluation batch, and differ only in the time change applied to the local-time increment (Supplementary Material~\ref{sec:rsgm_details}).}
\label{tab:rsgm_four_metric}
\scriptsize
\resizebox{\textwidth}{!}{%
\begin{tabular}{l cc cc cc cc}
\toprule
& \multicolumn{2}{c}{HS-Gaussian MMD} & \multicolumn{2}{c}{energy distance} & \multicolumn{2}{c}{two-copy MMD} & \multicolumn{2}{c}{overlap MMD} \\
\cmidrule(lr){2-3}\cmidrule(lr){4-5}\cmidrule(lr){6-7}\cmidrule(lr){8-9}
& \multicolumn{4}{c}{\emph{characteristic}} & \multicolumn{2}{c}{\emph{second moment}} & \multicolumn{2}{c}{\emph{first moment}} \\
\cmidrule(lr){2-5}\cmidrule(lr){6-7}\cmidrule(lr){8-9}
Benchmark & ours & RSGM & ours & RSGM & ours & RSGM & ours & RSGM \\
\midrule
Single-cluster & $\mathbf{1.57}\pm0.13$ & $2.03\pm0.23$ & $\mathbf{3.30}\pm0.25$ & $4.22\pm0.47$ & $\mathbf{3.37}\pm0.31$ & $4.45\pm0.57$ & $\mathbf{2.61}\pm0.23$ & $3.47\pm0.43$ \\
Trimodal & $\mathbf{0.86}\pm0.10$ & $1.14\pm0.09$ & $\mathbf{2.33}\pm0.21$ & $2.92\pm0.21$ & $\mathbf{1.90}\pm0.16$ & $2.50\pm0.22$ & $\mathbf{1.84}\pm0.25$ & $2.47\pm0.22$ \\
Eq.\ bimodal & $0.77\pm0.08$ & $0.86\pm0.16$ & $2.24\pm0.17$ & $2.49\pm0.33$ & $\mathbf{1.95}\pm0.12$ & $2.24\pm0.18$ & $1.65\pm0.21$ & $1.90\pm0.42$ \\
Spin-coherent & $0.80\pm0.13$ & $0.93\pm0.14$ & $2.33\pm0.25$ & $2.66\pm0.28$ & $\mathbf{2.08}\pm0.18$ & $2.44\pm0.25$ & $1.73\pm0.31$ & $2.05\pm0.32$ \\
TFIM & $\mathbf{1.27}\pm0.07$ & $1.45\pm0.06$ & $\mathbf{3.03}\pm0.19$ & $3.43\pm0.14$ & $\mathbf{3.07}\pm0.22$ & $3.56\pm0.17$ & $\mathbf{2.41}\pm0.22$ & $2.79\pm0.14$ \\
XXZ & $0.84\pm0.25$ & $1.07\pm0.30$ & $2.16\pm0.51$ & $2.60\pm0.60$ & $1.95\pm0.46$ & $2.03\pm0.23$ & $1.56\pm0.56$ & $2.15\pm0.80$ \\
W states & $\mathbf{1.34}\pm0.13$ & $1.82\pm0.16$ & $\mathbf{2.86}\pm0.26$ & $3.80\pm0.31$ & $\mathbf{2.88}\pm0.33$ & $3.98\pm0.34$ & $\mathbf{2.19}\pm0.22$ & $3.06\pm0.29$ \\
Graph states & $\mathbf{1.45}\pm0.16$ & $1.87\pm0.18$ & $\mathbf{3.08}\pm0.31$ & $3.90\pm0.37$ & $\mathbf{3.13}\pm0.36$ & $4.05\pm0.44$ & $\mathbf{2.40}\pm0.29$ & $3.18\pm0.34$ \\
\bottomrule
\end{tabular}}
\end{table*}

\paragraph{Stochastic quantum trajectories and unravelings.}
Diffusion-like dynamics also appear in measurement-induced trajectories \cite{Dalibard1992WavefunctionAT,Gisin1992TheQD} and Lindblad unravelings \cite{kleinekathofer2002stochastic,caiaffa2017stochastic,chen2025unraveling}, where stochastic Schr\"odinger equations describe pure-state paths and motivate recovery/control viewpoints \cite{Kiefer2010QuantumMA}. PSMs leverage this connection through an SSE realization, but target a generative modeling objective: learning a score field on $\mathbb{CP}^{d-1}$ and using reverse-time integration to sample from a target ensemble.

\paragraph{Hybrid pipelines with quantum denoisers.}
A separate direction inserts quantum neural components into otherwise classical diffusion models, e.g., quantum neural network (QNN) denoisers \cite{kolle2024quantum} for image/latent diffusion \cite{Falco2024QuantumLD,Falco2025LeveragingQL} and scientific data generation, such as quark and gluon jet synthesis \cite{baidachna2025quantum}. Our setting differs in that the diffusion itself evolves \emph{quantum states} and the score is defined intrinsically on $\mathbb{CP}^{d-1}$.

\section{Experiments}
\label{sec:experiments}

We evaluate PSMs on generative modeling over quantum pure-state ensembles. The experiments answer four questions:
\textbf{(RQ1)} how PSM compares with a correctly implemented Riemannian score-based baseline, under metrics that are characteristic on distributions over pure states (Section~\ref{sec:exp_rsgm});
\textbf{(RQ2)} over what range of $n$ the method actually works, and where it stops (Section~\ref{sec:exp_scope});
\textbf{(RQ3)} whether the local-time analytic teacher is what carries the performance (Section~\ref{sec:exp_teacher});
\textbf{(RQ4)} whether the geometric construction is implemented as claimed (Section~\ref{sec:exp_diagnostics}).
Protocols, benchmark constructions and architectures are in Supplementary Material~\ref{sec:exp_setup}.

\subsection{Baselines}
\label{sec:baselines}

We compare against two controls, both trained on the same target states, with the same budget and checkpoint rule, and evaluated with the same metrics on the same batches.

\emph{Euclidean VP-SDE} is the ambient control: each normalized state $\psi\in\mathbb{C}^d$ is mapped to $\mathbb{R}^{2d}$ by concatenating real and imaginary parts, a standard VP-SDE with denoising score matching is trained on those vectors, and generated vectors are mapped back and normalized before evaluation. It is not trained in the classical input space; it sees exactly the same quantum states as PSM, in an extrinsic representation. The comparison therefore isolates the effect of working intrinsically on $\mathbb{CP}^{d-1}$.

\emph{RSGM} \cite{bortoli2022riemannian} is the Riemannian control. We instantiate it on $\mathbb{CP}^{d-1}$ with the same horizontal statevector representation, score network, optimizer, sampler and checkpoint rule as PSM, and supervise it with the local-time loss $\ell_{t\mid s}$ and the Varadhan teacher $\exp^{-1}_{X_t}(X_s)/(t-s)$ that the original paper recommends when an approximation of the transition family is available. This is the variant that matters for our claims, and it differs from ours only in the scaling of the teacher; Supplementary Material~\ref{sec:rsgm_details} records what our port does and does not reproduce of the original method.

Ablated variants of our own model --- no local teacher, finite-difference teacher, drift-corrected teacher, $\lambda=0$ forward process, gauge-fixed score network --- are described where they are used.

\subsection{Benchmarks}
\label{sec:exp_benchmarks}

Nine target ensembles are used, all at $n=6$ except the single-cluster family, which is additionally swept over $n\in\{2,\dots,14\}$ in the scope study of Section~\ref{sec:exp_scope}. Eight of them draw a reference state --- a computational basis state, a GHZ-like superposition, a product spin-coherent state, an exactly diagonalized TFIM or XXZ ground state, a W state or a linear-chain graph state --- and apply a complex Gaussian perturbation of amplitude $\varepsilon=0.06$ followed by renormalization, so the geometry of the target is known while the model still has to learn a non-trivial distribution on $\mathbb{CP}^{d-1}$. The ninth is different in kind: MNIST digits $0/1$ are PCA-reduced, centred and normalized into feature states, so the induced law on the manifold has no closed form. W states and graph states were added after the main comparison had been run, to check whether the advantage over the baseline was an artifact of the original suite; the graph-state family also has uniform amplitude modulus, which makes a gauge based on the largest-modulus amplitude maximally ill-conditioned. Supplementary Material~\ref{sec:appendix_benchmark_suite} gives the explicit constructions.

\subsection{Comparison with a Correctly Implemented Riemannian Baseline (RQ1)}
\label{sec:exp_rsgm}

\paragraph{What the baseline should be.}
Riemannian score-based generative modeling \cite{bortoli2022riemannian} offers more than one training signal. Besides implicit score matching, it provides a \emph{local-time} denoising loss $\ell_{t\mid s}$ supervised by the Varadhan teacher $\exp^{-1}_{X_t}(X_s)/(t-s)$, recommended whenever an approximation of the transition family is available. Our local-time teacher belongs to that family; the difference is that we divide by $\beta^2=\sigma(t)^2\delta t$ rather than by $\delta t$, which is the correct small-time variance when the diffusion schedule is not unit. The appropriate baseline is therefore the published local-time variant, not a Brownian-perturbation regression, and that is what we report here.

\paragraph{Metrics.}
Write $\rho_\psi=\ket{\psi}\bra{\psi}$ and $F=|\braket{\psi|\phi}|^2$. The overlap kernel $k=F=\mathrm{Tr}(\rho_\psi\rho_\phi)$ used in much of the quantum generative literature is linear in $\rho$, and the two-copy kernel $k_2=F^2=\mathrm{Tr}\!\big[(\rho_\psi^{\otimes2})(\rho_\phi^{\otimes2})\big]$ is linear in $\rho^{\otimes2}$. Both therefore compare a fixed moment of the ensemble rather than the ensemble itself.

Proposition~\ref{prop:kernel_resolution} in Supplementary Material~\ref{sec:appendix_metric_validity} makes this precise: the overlap MMD equals $\|\mathbb{E}_p[\rho]-\mathbb{E}_q[\rho]\|_{\mathrm{HS}}$ and the two-copy MMD equals $\|\mathbb{E}_p[\rho^{\otimes2}]-\mathbb{E}_q[\rho^{\otimes2}]\|_{\mathrm{HS}}$, so the first certifies only the mean density matrix and the second only the second moment; neither is characteristic, and the two-copy kernel cannot separate any two distinct state $2$-designs. The HS-Gaussian kernel and the chordal energy distance are characteristic, because $\psi\mapsto\rho_\psi$ embeds the compact $\mathbb{CP}^{d-1}$ isometrically into a Euclidean space on which the Gaussian kernel is characteristic and the metric is of strong negative type. Neither can certify multimodal generation: Table~\ref{tab:decoy} gives an explicit pair of ensembles the overlap kernel cannot separate. We therefore report four quantities and base the conclusions on the two that are characteristic --- the HS-Gaussian MMD with median-heuristic bandwidth and the chordal energy distance --- with the two-copy and overlap MMDs included because the literature reports them and because a claim that holds under all four is stronger than one that holds under two.

\begin{table}[t]
\centering
\caption{Decoy test at $n=6$, $512$ samples per ensemble. The target is the equatorial bimodal benchmark; the decoy has the same mean density matrix but no superposition structure. The overlap MMD cannot separate them, as Proposition~\ref{prop:kernel_resolution}(i) requires; the other three can.}
\label{tab:decoy}
\resizebox{\columnwidth}{!}{%
\begin{tabular}{lcccc}
\toprule
Pair & overlap MMD & HS-Gauss MMD & 2-copy MMD & energy dist. \\
\midrule
Target vs.\ independent draw   & $0.0$ & $-3.8\times10^{-4}$ & $-4.9\times10^{-4}$ & $4.0\times10^{-3}$ \\
Target vs.\ decoy              & $0.0$ & $\mathbf{1.38\times10^{-2}}$ & $\mathbf{1.07\times10^{-1}}$ & $\mathbf{3.40\times10^{-2}}$ \\
\bottomrule
\end{tabular}}

\end{table}

\paragraph{The same test on hardware.}
All four metrics are functions of the pairwise fidelity $F=|\braket{\psi|\phi}|^2$ alone --- including the two that Proposition~\ref{prop:kernel_resolution} shows are characteristic --- so a single set of compute-uncompute measurements yields all of them and the decoy test can be run on a device rather than in simulation. We prepare the same construction at $n=3$ --- target an equal mixture of $(\ket{000}\pm\ket{111})/\sqrt2$, decoy an equal mixture of $\ket{000}$ and $\ket{111}$, both Gaussian-perturbed, so that the two share a mean density matrix by construction --- draw $20$ states per ensemble, and estimate all $800$ pairwise fidelities on an IBM Heron device at $4096$ shots.

\begin{table}[t]
\centering
\caption{The decoy test executed on IBM \texttt{ibm\_kawasaki} at $n=3$: $20$ states per ensemble, $800$ compute-uncompute circuits, $4096$ shots, median two-qubit depth $53$ after transpilation. ``Corrected'' rescales every fidelity by the measured self-overlap $0.9930$. Hardware reproduces the exact values to within $9$--$17\%$ and, more to the point, reproduces the separation: the overlap MMD is at its estimator-noise level while the two-copy MMD is an order of magnitude above it.}
\label{tab:hw_decoy}
\resizebox{\columnwidth}{!}{%
\begin{tabular}{lcccc}
\toprule
& overlap MMD & HS-Gauss MMD & 2-copy MMD & energy dist. \\
\midrule
Exact (statevector)   & $3.81\times10^{-2}$ & $2.84\times10^{-1}$ & $4.26\times10^{-1}$ & $2.86\times10^{-1}$ \\
Hardware, raw         & $3.80\times10^{-2}$ & $2.46\times10^{-1}$ & $3.87\times10^{-1}$ & $2.38\times10^{-1}$ \\
Hardware, corrected   & $3.83\times10^{-2}$ & $2.52\times10^{-1}$ & $3.93\times10^{-1}$ & $2.45\times10^{-1}$ \\
\midrule
Estimator noise ($\pm$sd) & $5.8\times10^{-2}$ & --- & $5.5\times10^{-2}$ & $6.5\times10^{-2}$ \\
\bottomrule
\end{tabular}}

\end{table}

Table~\ref{tab:hw_decoy} reports the outcome. The population value of the overlap MMD between these two ensembles is exactly zero, and at $20$ samples per ensemble the estimator has a standard deviation of $5.8\times10^{-2}$ across resamplings; the hardware estimate of $3.80\times10^{-2}$ is inside that noise, as is the exact one. The two-copy MMD is $3.87\times10^{-1}$ on hardware, seven standard deviations away, and the energy distance behaves the same way. The failure of the overlap kernel is therefore not an artifact of simulation or of a particular sample size: on a real device, with real readout error, the metric this literature reports is blind to a pair of ensembles that the HS-Gaussian MMD and the energy distance separate by an order of magnitude.

The systematic gap between hardware and exact values is $9$--$17\%$ and shrinks slightly under the self-overlap correction, which is the expected signature of depolarizing noise compressing all measured fidelities towards the uniform outcome. It biases the characteristic metrics downward but does not change what they can and cannot see, which is the property at issue here.

\paragraph{Result.}
Table~\ref{tab:rsgm_four_metric} gives the comparison. Ours is better in all thirty-two cells. Applying Holm correction within each metric family of eight tests, the difference survives at the $5\%$ level on five of eight benchmarks under the HS-Gaussian MMD, five of eight under the energy distance, seven of eight under the two-copy MMD and five of eight under the overlap MMD; where it survives, the margin is a factor of $1.13$ to $1.40$. The three benchmarks that do not clear correction --- equatorial bimodal, spin-coherent and XXZ --- are the same three under every metric, so the picture is consistent: the effect is real on five families and below resolution at ten seeds on three. The two characteristic metrics agree with the two moment-matching ones on direction everywhere, which is the reassurance we can offer that the ranking does not rest on a kernel that cannot separate some ensembles.

Two remarks on how to read this. First, the two arms differ in one line of code --- the local-time increment is divided by the diffusion-clock increment rather than by the elapsed time --- and everything else, including the $U(1)$ augmentation of Section~\ref{sec:method_recipe}, is shared, so the effect is attributable to the time change and to nothing else. We read this as a correction to how the published teacher is applied under a non-unit schedule, not as a new teacher: a careful implementation of the baseline would make the same substitution, and what the table measures is the cost of not making it. Second, the effect is real but modest, and smaller than the effect of the supervision signal itself: removing the analytic local-time teacher costs an order of magnitude more (Section~\ref{sec:exp_teacher}). We take the scaling to be a correction worth making rather than the main source of the method's behaviour.

\subsection{Does Intrinsic Geometry Help? (RQ1, continued)}
\label{sec:exp_euclid}

The comparison above isolates the supervision signal within the Riemannian family. The complementary question is whether working intrinsically on $\mathbb{CP}^{d-1}$ helps at all, relative to treating a normalized statevector as an ordinary vector in $\mathbb{R}^{2d}$. Table~\ref{tab:appendix_vp_sde_suite} answers it under the same protocol: identical data, network width and depth, optimizer, budget, checkpoint rule, sample count and evaluation batch, with the diffusion moved into the ambient space and the samples renormalized before evaluation.

\begin{table*}[t]
\centering
\caption{Against the ambient Euclidean VP-SDE baseline at $n=6$, under the protocol of Section~\ref{sec:method_recipe}. Values $\times10^{-2}$, mean $\pm$ sd over ten seeds; lower is better. Bold marks Holm-corrected significance at the $5\%$ level within each metric family. Only the first two metrics are characteristic (Proposition~\ref{prop:kernel_resolution}).}
\label{tab:appendix_vp_sde_suite}
\scriptsize
\resizebox{\textwidth}{!}{%
\begin{tabular}{l cc cc cc cc}
\toprule
& \multicolumn{2}{c}{HS-Gaussian MMD} & \multicolumn{2}{c}{energy distance} & \multicolumn{2}{c}{two-copy MMD} & \multicolumn{2}{c}{overlap MMD} \\
\cmidrule(lr){2-3}\cmidrule(lr){4-5}\cmidrule(lr){6-7}\cmidrule(lr){8-9}
& \multicolumn{4}{c}{\emph{characteristic}} & \multicolumn{2}{c}{\emph{second moment}} & \multicolumn{2}{c}{\emph{first moment}} \\
\cmidrule(lr){2-5}\cmidrule(lr){6-7}\cmidrule(lr){8-9}
Benchmark & ours & Euclidean & ours & Euclidean & ours & Euclidean & ours & Euclidean \\
\midrule
Single-cluster & $\mathbf{1.57}\pm0.13$ & $13.34\pm0.38$ & $\mathbf{3.30}\pm0.25$ & $30.87\pm1.05$ & $\mathbf{3.37}\pm0.31$ & $20.98\pm0.45$ & $\mathbf{2.61}\pm0.23$ & $35.74\pm1.36$ \\
Trimodal & $\mathbf{0.86}\pm0.10$ & $7.31\pm0.11$ & $\mathbf{2.33}\pm0.21$ & $17.50\pm0.28$ & $\mathbf{1.90}\pm0.16$ & $9.71\pm0.10$ & $\mathbf{1.84}\pm0.25$ & $20.82\pm0.39$ \\
Eq.\ bimodal & $\mathbf{0.77}\pm0.08$ & $6.74\pm0.16$ & $\mathbf{2.24}\pm0.17$ & $16.46\pm0.37$ & $\mathbf{1.95}\pm0.12$ & $10.89\pm0.12$ & $\mathbf{1.65}\pm0.21$ & $18.81\pm0.51$ \\
Spin-coherent & $\mathbf{0.80}\pm0.13$ & $6.83\pm0.10$ & $\mathbf{2.33}\pm0.25$ & $16.64\pm0.25$ & $\mathbf{2.08}\pm0.18$ & $10.98\pm0.08$ & $\mathbf{1.73}\pm0.31$ & $19.04\pm0.34$ \\
TFIM & $\mathbf{1.27}\pm0.07$ & $9.12\pm0.31$ & $\mathbf{3.03}\pm0.19$ & $21.45\pm0.74$ & $\mathbf{3.07}\pm0.22$ & $12.24\pm0.29$ & $\mathbf{2.41}\pm0.22$ & $25.58\pm0.97$ \\
XXZ & $\mathbf{0.84}\pm0.25$ & $8.45\pm0.62$ & $\mathbf{2.16}\pm0.51$ & $20.41\pm1.47$ & $\mathbf{1.95}\pm0.46$ & $13.33\pm0.63$ & $\mathbf{1.56}\pm0.56$ & $23.60\pm1.89$ \\
W states & $\mathbf{1.34}\pm0.13$ & $12.93\pm0.48$ & $\mathbf{2.86}\pm0.26$ & $29.70\pm1.37$ & $\mathbf{2.88}\pm0.33$ & $20.47\pm0.59$ & $\mathbf{2.19}\pm0.22$ & $34.24\pm1.76$ \\
Graph states & $\mathbf{1.45}\pm0.16$ & $12.89\pm0.34$ & $\mathbf{3.08}\pm0.31$ & $29.63\pm1.03$ & $\mathbf{3.13}\pm0.36$ & $20.52\pm0.45$ & $\mathbf{2.40}\pm0.29$ & $34.12\pm1.33$ \\
\bottomrule
\end{tabular}}
\end{table*}

The gap is an order of magnitude --- $7.2$ to $10.0\times$ on the HS-Gaussian MMD, $7.1$ to $10.4\times$ on the energy distance --- on all eight benchmarks and all four metrics, every cell surviving Holm correction. Reading the Euclidean overlap column against the Haar reference makes the failure mode concrete: at $19$--$36\times10^{-2}$ against a prior level of $23$--$46\times10^{-2}$, the ambient model has moved only a fifth to a quarter of the way from the prior towards the target. It is not that it learns a slightly worse distribution; it is that renormalizing an ambient sample discards most of what the model learned, because the density it fits lives in $\mathbb{R}^{2d}$ and the evaluation lives on the quotient of the sphere. This is a much larger effect than any difference within the Riemannian family, and it is the clearest evidence in the paper that the manifold structure is doing work.

\paragraph{A target whose law is not known in closed form.}
\label{sec:exp_mnist}
All benchmarks so far are constructed in Hilbert space. A different regime arises when the target ensemble is induced by an encoding of classical data, so that its law on $\mathbb{CP}^{d-1}$ has no closed form and a learned model is genuinely needed. We PCA-reduce MNIST digits $0/1$ to $d=2^6$ real components, centre and normalize them, and model the resulting distribution on $\mathbb{CP}^{d-1}$.

\begin{table}[ht]
\centering
\caption{MNIST feature states at $n=6$, under the protocol of Section~\ref{sec:method_recipe}. Values $\times10^{-2}$, mean $\pm$ sd over ten seeds; lower is better. The first two metrics are characteristic (Proposition~\ref{prop:kernel_resolution}).}
\label{tab:appendix_mnist_generation}
\resizebox{\columnwidth}{!}{%
\begin{tabular}{lcccc}
\toprule
Method & HS-Gaussian $\downarrow$ & energy dist.\ $\downarrow$ & two-copy $\downarrow$ & overlap $\downarrow$ \\
\midrule
Euclidean VP-SDE & $4.42\pm0.25$ & $10.95\pm0.58$ & $5.16\pm0.20$ & $12.86\pm0.77$ \\
PSM (ours)       & $\mathbf{1.14\pm0.32}$ & $\mathbf{3.12\pm0.68}$ & $\mathbf{2.27\pm0.65}$ & $\mathbf{2.69\pm0.75}$ \\
\midrule
Haar reference   & --- & $12.22$ & $5.52$ & $14.54$ \\
\bottomrule
\end{tabular}}

\end{table}

The comparison is in Table~\ref{tab:appendix_mnist_generation}: PSM is better by $3.9\times$ on the HS-Gaussian MMD, $3.5\times$ on the energy distance, $2.3\times$ on two-copy and $4.8\times$ on overlap, paired over ten seeds with $t$ between $12.9$ and $26.4$. The Haar column makes the ambient failure legible in a way the $n=6$ synthetic families did not: at $12.86$ against a prior level of $14.54$, the Euclidean model has moved about a tenth of the way from the prior to the target, so what it learns in $\mathbb{R}^{2d}$ is almost entirely destroyed by the renormalization that puts its samples back on the manifold. This is the one family whose concentration is set by the data rather than by a perturbation scale, so the degeneration of Section~\ref{sec:exp_scope} does not apply to it.

\subsection{Scope: Where the Method Works and Where It Stops (RQ2)}
\label{sec:exp_scope}

The single-cluster family perturbs $\ket{0\cdots0}$ with a fixed amplitude noise $\varepsilon=0.06$, and since the perturbation energy grows as $d\varepsilon^2$ the target itself drifts towards the Haar measure as $n$ increases. The left block of Table~\ref{tab:scope} shows the consequence: the mean fidelity of the target with $\ket{0\cdots0}$ falls from $0.979$ at $n=2$ to $0.0084$ at $n=14$, and the dynamic range of the overlap MMD collapses with it. Such a benchmark cannot distinguish a model that has learned the target from one that returns the prior. To measure where the method actually works we instead rescale $\varepsilon\propto d^{-1/2}$, holding the target's concentration at its $n=6$ value, and report the selected-checkpoint MMD as a ratio to a Haar reference on the same evaluation batch, where $1.0$ means indistinguishable from the prior.

\begin{table}[t]
\centering
\caption{Left: degeneration of the single-cluster benchmark at fixed $\varepsilon=0.06$ ($1024$ samples), as the target approaches Haar. Right: scope of validity once $\varepsilon\propto d^{-1/2}$ holds concentration fixed ($10{,}000$ steps, best checkpoint by validation MMD), as the ratio of the selected-checkpoint MMD to the Haar reference on the same batch, under the original implementation and under the protocol of Section~\ref{sec:method_recipe} (three seeds).}
\label{tab:scope}
\resizebox{\columnwidth}{!}{%
\begin{tabular}{lc cc cc}
\toprule
& & \multicolumn{2}{c}{fixed $\varepsilon=0.06$} & \multicolumn{2}{c}{matched $\varepsilon$: ratio $\downarrow$} \\
\cmidrule(lr){3-4}\cmidrule(lr){5-6}
$n$ & $d$ & fid.\ with $\ket{0\cdots0}$ & MMD to Haar & original & protocol \\
\midrule
$2$  & $4$     & $0.979$  & $7.1\times10^{-1}$ & --- & --- \\
$4$  & $16$    & $0.902$  & $7.5\times10^{-1}$ & --- & --- \\
$6$  & $64$    & $0.689$  & $4.6\times10^{-1}$ & $\mathbf{0.184}$ & --- \\
$8$  & $256$   & $0.354$  & $1.2\times10^{-1}$ & $0.959$ & $\mathbf{0.721}$ \\
$10$ & $1024$  & $0.120$  & $1.4\times10^{-2}$ & $1.000$ & $1.002$ \\
$14$ & $16384$ & $0.0084$ & $7.0\times10^{-5}$ & $1.000$ & --- \\
\bottomrule
\end{tabular}}

\end{table}

The right block of Table~\ref{tab:scope} locates the boundary, and the protocol of Section~\ref{sec:method_recipe} moves it out by one step: at $n=8$ the ratio improves from $0.959$ to $0.721$, so the model is degraded but no longer vacuous, while at $n=10$ it stays at $1.002$ even with the full $10{,}000$-step budget.

The geometry does break down over this range. One forward step displaces a state by $\sigma\sqrt{2(d-1)\delta t}$ in FS distance, which is $0.50$ at $n=6$, $1.01$ at $n=8$ and $2.02$ at $n=10$ against an injectivity radius of $\pi/2\approx1.57$: from $n\approx8$ a single step traverses the manifold and normal coordinates cease to mean anything. The repair this suggests fails. Rescaling $\delta t\propto1/d$, with the reverse-step count matched, gives ratios of $0.948$ at $n=8$ and $1.000$ at $n=10$, and the normalised parameterisation does not help either, alone ($45.7\times10^{-2}$ at $n=8$) or combined with the rescaled step ($45.0\times10^{-2}$), against a Haar reference of $46.6\times10^{-2}$. What moved the boundary was the training protocol, not any step-size correction, which we read as evidence that the binding constraint at these dimensions is the quality of the supervision rather than the validity of the chart alone. Beyond $n=10$ we cannot identify it.

We therefore restrict the empirical claims of this paper to $n\le8$, with $n=8$ already substantially degraded, and report $n=10$ and beyond as a negative result. This includes the $n=14$ run we previously reported at fixed $\varepsilon$: its MMD of $1.2961\times10^{-4}$ matches the same-batch Haar reference to three digits, so it records the largest dimension we simulated rather than a scalability result (Supplementary Material~\ref{sec:appendix_metric_validity}). Since the model consumes the full $2^n$-dimensional statevector the cost is exponential in $n$ in any case, and a scalable version needs a structured score parameterization that we do not demonstrate.

\paragraph{Outside the local-cluster regime.}
The local-time teacher is built from a short-time expansion in normal coordinates, so it should be least useful when the target has broad support. Two $n=6$ families test this: a mixture of four Haar-random caps, and depth-$12$ random two-qubit brickwork circuit outputs. The method remains stable and still improves on the ambient Euclidean baseline --- overlap MMD $3.19$ versus $3.42\times10^{-2}$ on the Haar mixture and $4.91$ versus $5.86\times10^{-2}$ on random circuits, with $\Delta_{\mathrm{obs}}$ and entanglement W$_1$ agreeing --- but the margin is far smaller than on local or multimodal targets. This is consistent with the role of the teacher: intrinsic geometry keeps samples on the manifold, while locality stops being an informative inductive bias.

\subsection{What Carries the Performance (RQ3)}
\label{sec:exp_teacher}

\begin{table}[t]
\centering
\caption{Supervision ablation at $n=6$ under the protocol of Section~\ref{sec:method_recipe}: only the teacher changes. Values $\times10^{-2}$, mean $\pm$ sd over ten seeds. Haar is the same-batch reference under each metric, so a row that reaches it has learned nothing.}
\label{tab:teacher_ablation_246q}
\resizebox{\columnwidth}{!}{%
\begin{tabular}{lcccc}
\toprule
& \multicolumn{2}{c}{HS-Gaussian MMD} & \multicolumn{2}{c}{overlap MMD} \\
\cmidrule(lr){2-3}\cmidrule(lr){4-5}
Supervision & Single-cluster & TFIM & Single-cluster & TFIM \\
\midrule
No local teacher            & $15.94\pm0.13$ & $10.76\pm0.16$ & $45.61\pm0.43$ & $31.43\pm0.56$ \\
Finite-difference teacher   & $7.39\pm0.38$ & $4.90\pm0.21$ & $16.62\pm1.11$ & $12.17\pm0.62$ \\
Analytic local-time (ours)  & $\mathbf{1.57\pm0.13}$ & $\mathbf{1.27\pm0.07}$ & $\mathbf{2.61\pm0.23}$ & $\mathbf{2.41\pm0.22}$ \\
\midrule
Haar reference              & $16.07$ & $10.88$ & $45.99$ & $31.82$ \\
\bottomrule
\end{tabular}}

\end{table}

Table~\ref{tab:teacher_ablation_246q} varies only the supervision signal. Without a local teacher the model does not learn at all: it lands on the same-batch Haar reference to within $1\%$ on both benchmarks and under both metrics, which is what returning the prior would give. A finite-difference estimate of the same score recovers most of the way but still costs a factor of $3.9$ to $4.7$ on the characteristic metric and $5.0$ to $6.4$ on the overlap metric, with the gap set by estimator variance rather than by bias. The analytic teacher is therefore not a convenience; it is the component that makes the construction work, and its effect is an order of magnitude larger than the gap to the published Riemannian baseline in Table~\ref{tab:rsgm_four_metric}. Consistently with that reading, the forward drift is not load-bearing: setting $\lambda=0$, so that the forward process is pure FS Brownian motion, matches or improves every benchmark we tested, and is what the protocol of Section~\ref{sec:method_recipe} does.

\begin{table}[t]
\centering
\caption{Recipe ablation at $n=6$: one ingredient removed at a time, everything else held fixed. HS-Gaussian MMD $\times10^{-2}$, mean $\pm$ sd over ten seeds, with the ratio to the full recipe in parentheses; lower is better. The third row is the published teacher normalization, i.e.\ the RSGM arm of Table~\ref{tab:rsgm_four_metric}, repeated for comparison.}
\label{tab:ingredients}
\resizebox{\columnwidth}{!}{%
\begin{tabular}{lccc}
\toprule
Configuration & Single-cluster & Trimodal & TFIM \\
\midrule
Full recipe & $\mathbf{1.57\pm0.13}$ & $\mathbf{0.86\pm0.10}$ & $\mathbf{1.27\pm0.07}$ \\
$-$ closed-form geodesic maps & $2.22\pm0.24$ ($\times1.42$) & $1.39\pm0.19$ ($\times1.61$) & $1.51\pm0.15$ ($\times1.19$) \\
$-$ diffusion-clock scaling & $2.03\pm0.23$ ($\times1.29$) & $1.14\pm0.09$ ($\times1.32$) & $1.45\pm0.06$ ($\times1.15$) \\
$-$ $U(1)$ phase augmentation & $2.23\pm0.37$ ($\times1.42$) & $0.98\pm0.16$ ($\times1.14$) & $1.35\pm0.20$ ($\times1.07$) \\
$+$ dispersive drift $\lambda=0.2$ & $1.66\pm0.20$ ($\times1.06$) & $1.10\pm0.20$ ($\times1.28$) & $1.34\pm0.10$ ($\times1.06$) \\
\bottomrule
\end{tabular}}

\end{table}

Table~\ref{tab:ingredients} separates the choices, each removed from the full recipe with everything else held fixed. Under the HS-Gaussian metric the closed-form geodesic maps are the largest single ingredient, worth $1.19$--$1.61\times$ and the only one significant on all three families; the diffusion-clock scaling is next at $1.15$--$1.32\times$. The $U(1)$ augmentation is worth $1.07$--$1.42\times$, resolved at ten seeds only on the unimodal single-cluster target, so the honest statement is that it helps most where the target is a single mode with a definite phase. Turning the dispersive drift back on costs $1.06$--$1.28\times$. The ordering is the same under the overlap metric, which is reported in Supplementary Material~\ref{sec:appendix_stress_component}.

\paragraph{The predicted finite-step bias is the one we observe.}
Proposition~\ref{prop:simple_teacher_structured_bias} states that the zero-mean teacher is biased, but in a specific way: its finite-step optimum is the Riemannian score of $p_t$ reweighted by a Gaussian envelope $W_t$ centred at $\psi_\star$. Two experiments test that statement rather than assuming it.

\begin{table}[ht]
\centering
\caption{Zero-mean teacher \eqref{eq:teacher_score_euclidean_ou} against the drift-corrected form \eqref{eq:teacher_score_drift}, at $n=6$ with the dispersive drift switched on ($\lambda=0.2$) so that the two differ at all. Everything else follows the protocol of Section~\ref{sec:method_recipe}: $\sigma_{\min}=0.05$, $\sigma_{\max}=1$, $10{,}000$ steps, best checkpoint by validation MMD, ten seeds. Values $\times10^{-2}$, mean $\pm$ sd. The numbers are not comparable with the main tables, which use $\lambda=0$.}
\label{tab:teacher_drift_ablation}
\resizebox{\columnwidth}{!}{%
\begin{tabular}{lcccc}
\toprule
& \multicolumn{2}{c}{HS-Gaussian MMD} & \multicolumn{2}{c}{overlap MMD} \\
\cmidrule(lr){2-3}\cmidrule(lr){4-5}
Benchmark & zero-mean & drift-corr. & zero-mean & drift-corr. \\
\midrule
Single-cluster & $1.66\pm0.20$ & $1.67\pm0.18$ & $2.80\pm0.35$ & $2.81\pm0.33$ \\
Trimodal & $1.10\pm0.20$ & $1.01\pm0.14$ & $2.40\pm0.53$ & $2.20\pm0.35$ \\
Spin-coherent & $0.83\pm0.12$ & $0.84\pm0.16$ & $1.78\pm0.30$ & $1.86\pm0.38$ \\
TFIM & $1.34\pm0.10$ & $1.35\pm0.13$ & $2.59\pm0.25$ & $2.57\pm0.25$ \\
\bottomrule
\end{tabular}}

\end{table}

First, the question of which teacher to use does not arise for the configuration we run: with $\lambda=0$ the drift $b$ vanishes and \eqref{eq:teacher_score_euclidean_ou} and \eqref{eq:teacher_score_drift} are the same function. To test whether the omission matters where it can, we switch the drift back on at $\lambda=0.2$ and compare the two teachers under the same protocol (Table~\ref{tab:teacher_drift_ablation}). They are statistically indistinguishable on all four benchmarks, with paired $|t|\le1.1$ and differences between $-7.7\%$ and $+2.3\%$ over ten seeds. Under our schedule the bias is therefore not the dominant error, exactly as the proposition's $O(\sigma\sqrt{\delta t})$ remainder predicts.

\begin{figure}[ht]
    \centering
    \includegraphics[width=\columnwidth]{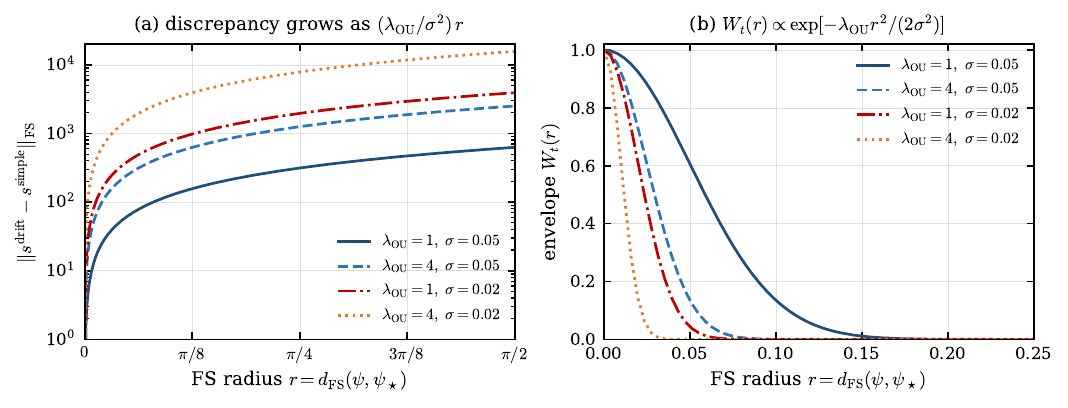}
    \caption{Structured bias of the zero-mean teacher, computed in closed form. The two teachers differ by $-(\lambda/\sigma^2)\log_\psi(\psi_\star)$, of norm $(\lambda/\sigma^2)r$ (panel a, log ordinate), which is exactly $\nabla_{FS}\log W_t$ for the envelope of Proposition~\ref{prop:simple_teacher_structured_bias} (panel b); the two are parallel by construction. At the default schedule ($\sigma\ge0.05$, $\lambda=0.2$) the envelope is wide and the bias small.}
\label{fig:teacher_bias_diagnostic}
\end{figure}

Second, the bias can be made visible by moving into the regime where the proposition says it should grow. Reducing $\sigma_{\min}$ or increasing $\lambda$ amplifies the discrepancy between the two teachers by the predicted factor $\lambda/\sigma^2$, and the direction of the measured difference field agrees with $\nabla_{FS}\log W_t$ to a cosine of $1.000$ in all four settings (Fig.~\ref{fig:teacher_bias_diagnostic}). The bias is thus not merely small; it has the structure the theory assigns it.

\begin{figure}[ht]
    \centering
    \includegraphics[width=\columnwidth]{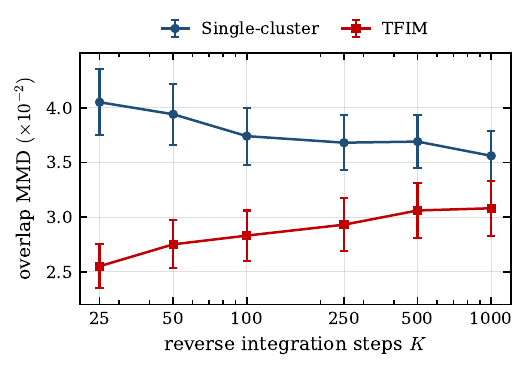}
    \caption{Sensitivity of reverse sampling to the number of integration steps, under the protocol of Section~\ref{sec:method_recipe}, at $n=6$ over three seeds. Across a factor of $40$ in $K$ the metric moves by less than $15\%$, and the two families move in opposite directions, so the step count is not a lever worth tuning here; $K=25$ already suffices.}
\label{fig:finite_step_sensitivity}
\end{figure}

Finally, Fig.~\ref{fig:finite_step_sensitivity} varies the number of reverse-integration steps at evaluation time. Under the protocol of Section~\ref{sec:method_recipe} the dependence is weak: over a factor of $40$ in $K$ the metric changes by less than $15\%$, and the direction differs by family, improving with $K$ on single-cluster and worsening on TFIM. The $500$-step setting used elsewhere is therefore not load-bearing, and a much cheaper sampler would report the same conclusions.

\subsection{What the Exact Heat Kernel Buys}
\label{sec:exp_heatkernel}

$\mathbb{CP}^{d-1}$ is a compact rank-one symmetric space, so its heat kernel is a zonal function of the FS distance alone and is available in closed form as a Jacobi series: with $\lambda_k=4k(k+m)$, $m=d-1$, multiplicities $N_k=\frac{2k+m}{m}\binom{k+m-1}{k}^2$ and zonal functions $P_k^{(m-1,0)}(\cos2\theta)/P_k^{(m-1,0)}(1)$. At small $d$ this makes the oracle available: the local-time teacher can be replaced by the \emph{exact} conditional score $\nabla_{\mathrm{FS}}\log p_{\tau(t)}(\psi_t\mid\psi_0)$, where $\tau(t)=\int_0^t\sigma(u)^2du$ is the diffusion clock. This is the comparison that the symmetric-space structure invites, and it bounds what the local-time approximation costs.

Two practicalities decide how far it can be taken. The series needs $k\sim6/\sqrt{2\tau}$ terms, and below $\tau\approx0.05$ the partial sums cancel catastrophically against a density that is exponentially small in $\theta^2/\tau$, so double precision fails long before truncation does. We therefore use the spectral sum for $\tau\ge0.05$ and the Van Vleck--Morette form $\nabla\log p=-\theta/\tau-\tfrac12\nabla\log\Theta(\theta)$ below it, with $\Theta$ the volume density in normal coordinates; the two agree to $10^{-4}$ relative where both are valid, and the radial score was checked against a $4\times10^5$-sample simulation of the forward process. Because the exact teacher conditions on the full elapsed clock, all three arms in this experiment draw a single diffusion time per minibatch, so that the time a state is noised to is exactly the time it is labelled with. That protocol difference makes the numbers here incomparable with the rest of the paper, but internally consistent.

\begin{table}[t]
\centering
\caption{Exact heat-kernel supervision against the local-time teacher, at the manifolds small enough for the exact kernel to be computable. Values $\times10^{-2}$, mean $\pm$ sd over ten seeds; lower is better. Bold marks a paired difference in favour of the exact teacher with $|t|>2.26$. Both metrics shown are characteristic (Proposition~\ref{prop:kernel_resolution}); all three arms share everything but the teacher.}
\label{tab:heatkernel}
\scriptsize
\resizebox{\columnwidth}{!}{%
\begin{tabular}{ll cc cc}
\toprule
& & \multicolumn{2}{c}{HS-Gaussian MMD} & \multicolumn{2}{c}{energy distance} \\
\cmidrule(lr){3-4}\cmidrule(lr){5-6}
Manifold & Target & exact & ours & exact & ours \\
\midrule
$\mathbb{CP}^1$ & single-cluster & $\mathbf{19.16\pm2.90}$ & $23.78\pm1.92$ & $\mathbf{7.61\pm1.69}$ & $10.77\pm1.59$ \\
$\mathbb{CP}^1$ & bimodal & $4.81\pm0.25$ & $4.82\pm0.25$ & $17.08\pm0.58$ & $17.08\pm0.58$ \\
$\mathbb{CP}^3$ & single-cluster & $\mathbf{12.90\pm2.65}$ & $18.27\pm0.70$ & $\mathbf{10.08\pm3.29}$ & $18.51\pm1.50$ \\
$\mathbb{CP}^3$ & bimodal & $\mathbf{4.38\pm0.87}$ & $6.36\pm0.60$ & $\mathbf{16.03\pm2.31}$ & $21.53\pm1.53$ \\
$\mathbb{CP}^7$ & single-cluster & $\mathbf{10.81\pm2.76}$ & $16.59\pm1.68$ & $\mathbf{13.55\pm5.24}$ & $25.30\pm4.38$ \\
$\mathbb{CP}^7$ & bimodal & $\mathbf{4.23\pm0.80}$ & $5.88\pm0.51$ & $\mathbf{13.60\pm2.13}$ & $18.67\pm1.28$ \\
\bottomrule
\end{tabular}}

\end{table}

Table~\ref{tab:heatkernel} gives the answer, and it is not the one we expected. The exact teacher is better on five of the six settings under both characteristic metrics, by a factor of $1.2$ to $1.9$ on the HS-Gaussian MMD and $1.3$ to $1.9$ on the energy distance, with paired $t$ between $3.9$ and $7.6$ over ten seeds; on the overlap metric the same comparison gives $1.9$ to $2.7$. The exception is $\mathbb{CP}^1$ bimodal, where the target is two antipodal points on a sphere of radius $1/2$ and every arm sits at the Haar reference, so the row carries no information. The local-time approximation therefore costs between $1.2$ and $2.7\times$ against the oracle it approximates depending on the metric, and that cost does not shrink with dimension over the range where we can measure it.

Two things follow. First, the claim we can support is narrower than we had assumed: the local-time teacher is not a free stand-in for the exact kernel at these step sizes, and the residual is not dominated by other sources of variance. Second, the trade is still worth making where the oracle is unavailable, which is almost everywhere. The number of terms the series needs grows as $\tau^{-1/2}$ while the multiplicities grow as $k^{2m}$, so at $n=6$ the sum at the small-$\tau$ end of our schedule is not merely expensive but numerically unrepresentable, and the Van Vleck--Morette fallback we use below $\tau=0.05$ \emph{is} a local-time teacher, differing from ours only by the volume term. What the exact route offers is a factor of up to $2.7$ on manifolds of complex dimension at most seven, at the price of a per-manifold spectral analysis; what the local-time route offers is a teacher that is available at every $d$ and requires none. Closing that gap without the spectral machinery --- for instance by adding the volume correction alone --- is the most concrete improvement this paper points to.

\subsection{Is the Geometry Implemented as Claimed? (RQ4)}
\label{sec:exp_diagnostics}

The claims of Section~\ref{sec:method} are geometric, so they should be checked directly rather than inferred from generation quality. We report three checks; details and two further diagnostics are in Supplementary Material~\ref{sec:appendix_generator_test} and Supplementary Material~\ref{sec:appendix_isotropy_diagnostic}.

\paragraph{The forward process reaches the prior, checked with metrics that can see it.}
The reverse sampler is initialized from exact Haar samples, so it is only correct if the forward process actually transports $p_0$ to $\mu_{\mathrm{FS}}$. The overlap kernel cannot check this: it is not characteristic, so $\mathrm{MMD}_{\mathrm{overlap}}(p_T,\mu_{\mathrm{FS}})=0$ is consistent with $p_T\ne\mu_{\mathrm{FS}}$. We therefore propagate the data ensemble through the full horizon and compare against Haar under the two characteristic metrics of Section~\ref{sec:exp_rsgm}, reporting the other two alongside (Table~\ref{tab:prior_check}). At $n=4$ and $n=6$ every metric sits at the Haar--Haar sampling floor, so the terminal marginal is Haar as far as any of them can resolve. At $n=2$ it is not: all four metrics separate $p_T$ from $\mu_{\mathrm{FS}}$ by one to two orders of magnitude above the floor. This is mixing time rather than an implementation error. The relaxation rate of the first eigenspace is $\sigma(t)^2\lambda_1/2=2\sigma(t)^2 d$, so over the diffusion clock $\tau(T)=\int_0^T\sigma(u)^2du=0.166$ of our schedule the first mode is damped by $e^{-2d\tau(T)}$: that is $e^{-1.3}=0.26$ at $n=2$, $e^{-5.3}=5\times10^{-3}$ at $n=4$ and $e^{-21}$ at $n=6$. The horizon is simply too short to mix $\mathbb{CP}^3$ --- but it means the $n=2$ results carry a prior mismatch that the larger ones do not, and it is one more reason to read $n=2$ as a sanity check rather than as evidence.

\begin{table}[t]
\centering
\caption{Terminal marginal of the forward process against Haar, under the metrics of Section~\ref{sec:exp_rsgm}, with $1024$ samples per ensemble. ``Floor'' is the same statistic between two independent Haar batches; negative values are unbiased-estimator noise. At $n\ge4$ every metric is at the floor; at $n=2$ none of them is.}
\label{tab:prior_check}
\resizebox{\columnwidth}{!}{%
\begin{tabular}{lcccc}
\toprule
$n$ & overlap & HS-Gaussian & two-copy & energy \\
\midrule
$2$          & $4.5\times10^{-2}$  & $1.9\times10^{-2}$  & $3.1\times10^{-2}$  & $4.3\times10^{-2}$ \\
\quad floor  & $0.0$               & $-2.5\times10^{-4}$ & $-3.6\times10^{-4}$ & $1.8\times10^{-3}$ \\
\midrule
$4$          & $0.0$               & $-7.7\times10^{-5}$ & $-6.4\times10^{-5}$ & $2.50\times10^{-3}$ \\
\quad floor  & $0.0$               & $-6.3\times10^{-5}$ & $-4.9\times10^{-5}$ & $2.53\times10^{-3}$ \\
\midrule
$6$          & $0.0$               & $-1.1\times10^{-5}$ & $-3.3\times10^{-6}$ & $2.715\times10^{-3}$ \\
\quad floor  & $0.0$               & $-1.2\times10^{-6}$ & $-1.9\times10^{-6}$ & $2.737\times10^{-3}$ \\
\bottomrule
\end{tabular}}

\end{table}

\paragraph{The induced process agrees with FS Brownian motion on the first non-trivial eigenspace.}
Verifying that the tangent-noise covariance is isotropic constrains only the second-order symbol of the generator; it says nothing about the connection and drift terms. We therefore test the generator on a known eigenspace. On $\mathbb{CP}^{d-1}$ the functions $f_\chi(\psi)=|\braket{\chi|\psi}|^2$ span the first non-trivial eigenspace of $\Delta_{\mathrm{FS}}$, so a process generated by $(\sigma^2/2)\Delta_{\mathrm{FS}}$ must make $\mathbb{E}[f_\chi(\psi_t)]-1/d$ decay as a \emph{single} exponential, at a rate that is independent of $\chi$ and proportional to $\sigma^2$, with the constant fixed by the eigenvalue.

\begin{table}[t]
\centering
\caption{Generator check via eigenfunction decay at $n=6$. Rates are fitted over the window where the signal exceeds $3/d$. A correct generator predicts a $\chi$-independent rate proportional to $\sigma^2$; the predicted value of $\text{rate}/\sigma^2$ is $\lambda_1/2=2d=128$, where $\lambda_1=4d$ is the first non-zero eigenvalue of $\Delta_{\mathrm{FS}}$ on $\mathbb{CP}^{d-1}$ in the normalization of Section~\ref{sec:background} ($d_{\mathrm{FS}}=\arccos|\braket{\psi|\phi}|$, diameter $\pi/2$). The deviation grows with the step displacement, and a linear extrapolation of $\text{rate}/\sigma^2$ in $\sigma^2$ to zero step size gives $127.9$.}
\label{tab:generator_test}
\resizebox{\columnwidth}{!}{%
\begin{tabular}{ccccc}
\toprule
$\sigma$ & fitted rate & sd across $\chi$ & min $R^2$ & rate$/\sigma^2$ \\
\midrule
0.15 & 2.8704  & 0.0056 & 0.99998 & 127.57 \\
0.25 & 7.9209  & 0.0258 & 0.99997 & 126.74 \\
0.35 & 15.4206 & 0.0756 & 0.99996 & 125.88 \\
\bottomrule
\end{tabular}}

\end{table}

Table~\ref{tab:generator_test} confirms all three predictions at $n=6$: single-exponential fits with $R^2\ge0.9999$, a rate that varies by less than $0.5\%$ across test functions, $\mathrm{rate}/\sigma^2$ constant to $1.3\%$ across $\sigma$, and a value that falls short of the predicted $\lambda_1/2=2d=128$ by $0.34\%$, $0.98\%$ and $1.66\%$ at $\sigma=0.15,0.25,0.35$. The deviation grows with $\sigma$ at fixed $\delta t$, which is the signature of the $O(\delta t)$ weak error of the geodesic random walk rather than of a wrong generator, and extrapolating $\text{rate}/\sigma^2$ linearly in $\sigma^2$ to zero step size gives $127.9$ against the continuum value $128$. We checked the identity directly at $d=2,4,8$, where a $2\times10^5$-sample simulation gives $3.99$, $8.02$ and $16.02$ against the predicted $4$, $8$ and $16$. As a further check on the drift, the terminal law of the forward process is the FS/Haar measure to within the Haar--Haar sampling floor, both with $\lambda=0.2$ and with $\lambda=0$; an incorrect connection term would generically destroy the unitarily-invariant invariant measure.

\paragraph{The SSE realization and the tangent-projected implementation agree.}
Building all $d^2-1=4095$ generalized Gell--Mann generators at $n=6$ and applying the forward noise as a strictly unitary step $U=\exp(-i\sum_a c_a\lambda_a)$, the induced horizontal increment matches the analytic prediction $\mathbb{E}\|\Delta\|^2=2(d-1)\sigma^2\delta t$ to $0.1\%$ at $\sigma=0.068$ and $0.6\%$ at $\sigma=0.224$, and the covariance spectrum matches the Marchenko--Pastur law for an exactly isotropic Gaussian: with $p=2(d-1)=126$ coordinates the measured ratio of eigenvalue standard deviation to mean is $0.1241$ at $N=8192$ against the predicted $0.1240$, and $0.0621$ at $N=32768$ against $0.0620$, i.e.\ the residual anisotropy is entirely finite-sample and shrinks as $\sqrt{p/N}$. The MMD between the forward marginals of the two implementations is $0.0$ at $t=0.25,0.5,1.0$ against a data--data floor of $5.8\times10^{-5}$. The one measurable difference appears at large steps, where the unitary realization retains $90\%$ of the intended variance against $80\%$ for tangent projection with normalization retraction. The SSE is therefore an equivalent realization at our schedule, and a slightly more faithful one at large $\sigma$ --- not a computational requirement, which we state plainly because an earlier version of this work presented it as one.

\paragraph{The quantities the objective consumes are measurable on hardware.}
The training objective touches the data only through overlaps. On \texttt{ibm\_berlin} we estimated all pairwise overlaps among $8$ target and $8$ generated states at $n=2$ with compute--uncompute circuits, $128$ circuits at $4096$ shots, and rebuilt the kernel from measurements alone.

\begin{table}[t]
\centering
\caption{Overlap estimation on IBM hardware ($n=2$, $128$ circuits, $4096$ shots).}
\label{tab:hardware}
\resizebox{\columnwidth}{!}{%
\begin{tabular}{lc}
\toprule
Quantity & Value \\
\midrule
Self-overlap circuits (exact value $1$)        & $0.977$ \\
Shot-noise floor per overlap                   & $7.8\times10^{-3}$ \\
Overlap MAE, raw                               & $2.4\times10^{-2}$ \\
Overlap MAE, after self-overlap correction     & $1.5\times10^{-2}$ \\
MMD from exact statevectors                    & below estimator resolution \\
MMD rebuilt from hardware overlaps             & $1.2\times10^{-2}$ \\
\bottomrule
\end{tabular}}

\end{table}

The overlaps are measurable at small $n$ (Table~\ref{tab:hardware}); the metric built on them is not yet, since overlap-estimation error sets a ${\sim}10^{-2}$ floor on any measurement-only evaluation at this shot budget --- the same order as the method differences in Table~\ref{tab:rsgm_four_metric}. The remaining obstacle is the score output itself, which lives in $\mathbb{C}^d$ and would have to be restricted to a polynomially sized operator basis.

\paragraph{The learned score agrees with an analytically known one.}
On $\mathbb{CP}^1$ a von Mises--Fisher-like target admits a closed-form Riemannian score, which lets us compare the learned field against the truth rather than only comparing samples.

\begin{table}[ht]
\centering
\caption{Exact-score diagnostic on $\mathbb{CP}^{1}$. Score errors are evaluated on held-out states; sampler metrics compare generated and target samples. Lower is better for relative error and MMD; higher is better for cosine similarity.}
\label{tab:exact_score_diagnostic}
\resizebox{\columnwidth}{!}{%
\begin{tabular}{p{4.4cm}ccc}
\toprule
Score used & Relative score error $\downarrow$ & Score cosine $\uparrow$ & MMD $\downarrow$ \\
\midrule
Exact Riemannian score & -- & -- & $1.1\times 10^{-3}$ \\
PSM learned score & $8.3\times 10^{-2}$ & $0.987$ & $2.4\times 10^{-3}$ \\
Zero vector control & $1.0$ & $0.000$ & $7.6\times 10^{-2}$ \\
\bottomrule
\end{tabular}}

\end{table}

The learned score attains cosine similarity $0.987$ with the exact score on held-out states, and sampling with it costs about a factor of two in MMD relative to sampling with the exact score --- against a zero-field control that is two orders of magnitude worse. The local-time objective therefore recovers a known Riemannian score where one is available.

\paragraph{What imposing exact phase equivariance costs.}
The score model reads $(\mathrm{Re}\,\psi,\mathrm{Im}\,\psi)$ and projects its output onto the horizontal tangent space, which constrains the output but does not make the model equivariant: measured at $n=6$, the violation of $s_\theta(e^{i\varphi}\psi,t)=e^{i\varphi}s_\theta(\psi,t)$ is a relative error of $1.04$. Two canonical gauges remove it exactly, by evaluating the network on a canonical representative and rotating the output back: the phase of the largest-modulus amplitude, and the phase of the overlap with a fixed uniform reference. Both reduce the violation to $1.3\times10^{-7}$, and neither is uniformly beneficial.

\begin{table}[t]
\centering
\caption{Effect of imposing exact $U(1)$-equivariance by gauge fixing, at $n=6$. Best-checkpoint MMD ($\times10^{-2}$), mean $\pm$ sd over $3$ seeds, $2000$ steps, all else identical. Both gauges make the model exactly equivariant; neither dominates.}
\label{tab:gauge_ablation}
\resizebox{\columnwidth}{!}{%
\begin{tabular}{lccc}
\toprule
Benchmark & no gauge fixing & argmax gauge & reference gauge \\
\midrule
Single-cluster      & $9.91\pm0.49$ & $\mathbf{8.01\pm0.70}$ & $16.91\pm0.76$ \\
Trimodal            & $7.75\pm0.75$ & $\mathbf{4.90\pm1.29}$ & $9.19\pm1.34$ \\
Equatorial bimodal  & $4.76\pm0.27$ & $\mathbf{2.92\pm0.36}$ & $9.17\pm0.14$ \\
Spin-coherent       & $\mathbf{3.74\pm0.27}$ & $8.41\pm0.46$ & $8.18\pm1.96$ \\
TFIM                & $\mathbf{3.61\pm0.14}$ & $5.69\pm0.83$ & $6.82\pm1.03$ \\
XXZ                 & $3.52\pm0.38$ & $3.78\pm0.35$ & $\mathbf{3.28\pm0.21}$ \\
\bottomrule
\end{tabular}}

\end{table}

Table~\ref{tab:gauge_ablation} shows the pattern: gauge fixing improves the synthetic pole and equator families by $19$--$39\%$ and degrades the physics-derived ground states by $50$--$125\%$. We first attributed this to the discontinuity of the $\arg\max$ gauge on delocalized states; the reference gauge was built to remove that discontinuity and is best conditioned precisely on the spin-coherent family, yet it does not recover the loss. Randomising the phase of the data (Section~\ref{sec:method_recipe}) achieves the same invariance in distribution without this cost, and is what we use.

\section{Conclusion}

Defining score-based diffusion intrinsically on the pure-state manifold works, and the part that carries it is the supervision rather than the geometry of the forward process: the local-time teacher, scaled by the diffusion clock rather than by the elapsed time, is the difference between learning the target and returning the prior, and is worth a further $1.15$--$1.32$ against the published normalization, while the forward drift and the stochastic Schr\"odinger realization turn out to be inessential. Proposition~\ref{prop:simple_teacher_structured_bias} identifies the leading term of that teacher's bias as a Gaussian-envelope reweighting of the target, and the induced process matches Fubini--Study Brownian motion on the first non-trivial Laplace--Beltrami eigenspace. The limits are equally definite: the model is already degraded at $8$ qubits and stops learning by $10$ on a concentration-matched benchmark, and the overlap-kernel MMD standard in this literature compares only mean density matrices, on hardware as well as in simulation. Together with the factor of up to $2.7$ that the exact heat kernel is worth wherever it can be computed, they point the same way --- towards exact heat kernels on $\mathbb{CP}^{d-1}$ as a symmetric space, evaluated with metrics that separate ensembles, and a score model that does not consume the full statevector.

\bibliographystyle{IEEEtran}
\bibliography{references}

\appendices
\section{Theory Provenance and Paper-Specific Contributions}
\label{sec:appendix_theory_provenance}

For clarity, we separate the standard theoretical ingredients from the components that are specific to PSMs on $\mathbb{CP}^{d-1}$.
The following components are direct applications or mild adaptations of established theory: the Euclidean reverse-time score formula \cite{anderson1982reversetime,song2021scorebased} and its Riemannian counterpart for Brownian-driven diffusions \cite{huang2022riemannian,bortoli2022riemannian}; the Minakshisundaram--Pleijel heat-kernel parametrix on Riemannian manifolds \cite{hsu2002stochmanifolds}; Stratonovich-to-It\^o conversion; and the Riemannian-volume-measure denoising score-matching identity \cite{bortoli2022riemannian}.

The paper-specific theoretical components are:
\begin{itemize}
    \item the identification of the SSE Stratonovich dynamics in Eq.~\eqref{eq:sse_strat} as a quotient-space diffusion inducing a dispersive-drift flow on $\mathbb{CP}^{d-1}$, with explicit curvature/connection remainders and finite-step bounds;
    \item the practical isotropy diagnostic for generalized Gell--Mann directions on $\mathbb{CP}^{d-1}$, together with numerical verification under the implementation used in the experiments;
    \item the drift-aware short-time score expansion in FS normal coordinates, which separates the singular Gaussian score, the bounded OU-drift term, and the Jacobian/curvature correction;
    \item the distinction between pointwise consistency of the drift-corrected teacher and variance-weighted consistency of the simple zero-mean teacher, including the structural finite-step characterization of the latter;
    \item the resulting local-time teacher construction as a tractable substitute for unavailable closed-form transition densities on $\mathbb{CP}^{d-1}$.
\end{itemize}

\section{Induced Manifold Diffusion from Tangent-Projected Stratonovich Dynamics}
\label{sec:prop_induced_diffusion_strict}

We formalize the statement that a tangent-projected Stratonovich dynamics on the Hilbert sphere induces a diffusion on $\mathbb{CP}^{d-1}$ whose generator matches the intrinsic FS diffusion up to explicit connection terms.

We restate Proposition~\ref{prop:induced_diffusion_strict} from the main text and prove it.

\begin{proof}
\textbf{Step 1 (Well-defined induced process and induced SDE).}
Since $b(\psi,t)$ and $V_k(\psi)$ are horizontal and $U(1)$-equivariant, their pushforwards $a(x,t):=\pi_\ast b$ and $e_k(x):=\pi_\ast V_k$ are well-defined on $\mathbb{CP}^{d-1}$.
Let $f\in C^\infty(\mathbb{CP}^{d-1})$ and define its lift $\bar f:=f\circ \pi$ on $\mathbb{S}^{2d-1}$.
By the Stratonovich chain rule,
\begin{equation}
\begin{aligned}
    d(f(x_t))
    &=
    d(\bar f(\psi_t))
    =
    \langle \nabla \bar f(\psi_t),\, d\psi_t\rangle
    \\
    &=
    (b\,\bar f)(\psi_t,t)\,dt
    + \sigma(t)\sum_k (V_k\bar f)(\psi_t)\circ dW_t^{(k)}.
\end{aligned}
    \label{eq:strict_chain_rule}
\end{equation}
Using $\bar f=f\circ \pi$ and the definition of pushforward, for any horizontal vector field $V$ we have
\begin{equation}
    (V \bar f)(\psi)= ( \pi_\ast V \, f)([\psi]).
\end{equation}
Applying this to $b$ and $V_k$ turns \eqref{eq:strict_chain_rule} into
\begin{equation}
    d(f(x_t))
    =
    (a f)(x_t,t)\,dt
    + \sigma(t)\sum_k (e_k f)(x_t)\circ dW_t^{(k)}.
\end{equation}
Since this holds for all smooth test functions $f$, it identifies the induced Stratonovich SDE on $\mathbb{CP}^{d-1}$.

\textbf{Step 2 (Generator in a chosen frame).}
For a Stratonovich SDE on a manifold
\(
dx_t = a\,dt + \sum_k \sigma e_k\circ dW_t^{(k)},
\)
the generator acting on $f$ is (standard)
\begin{equation}
    \mathcal{L}_t f
    =
    a f
    + \frac{\sigma(t)^2}{2}\sum_{k=1}^K e_k(e_k f).
\end{equation}
Using the Levi--Civita connection, $e_k(e_k f)=\nabla_{e_k}\nabla_{e_k} f$ for scalar $f$, which gives \eqref{eq:strict_generator_frame}.

\textbf{Step 3 (Relation to the Laplace--Beltrami operator and the explicit remainder).}
If $\{e_k\}$ is an orthonormal frame, the FS Laplace--Beltrami operator satisfies the local identity
\begin{equation}
    \Delta_{\mathrm{FS}} f
    =
    \sum_{k=1}^K \Big(\nabla_{e_k}\nabla_{e_k} f - \nabla_{\nabla_{e_k}e_k} f\Big).
    \label{eq:strict_laplacian_identity}
\end{equation}
Rearranging \eqref{eq:strict_laplacian_identity} yields
\begin{equation}
    \sum_{k=1}^K \nabla_{e_k}\nabla_{e_k} f
    =
    \Delta_{\mathrm{FS}} f
    +
    \sum_{k=1}^K \nabla_{\nabla_{e_k}e_k} f.
\end{equation}
Since $\nabla_{v} f=\langle \nabla_{\mathrm{FS}} f, v\rangle_{\mathrm{FS}}$ for any vector field $v$, we obtain
\begin{equation}
    \sum_{k=1}^K \nabla_{\nabla_{e_k}e_k} f
    =
    \sum_{k=1}^K \langle \nabla_{\mathrm{FS}} f,\ \nabla_{e_k}e_k\rangle_{\mathrm{FS}}.
\end{equation}
Substituting into \eqref{eq:strict_generator_frame} gives \eqref{eq:strict_generator_laplacian} and the explicit remainder \eqref{eq:strict_generator_laplacian}.
Finally, if the orthonormal frame is geodesic at $x$ (so $\nabla_{e_k}e_k(x)=0$), then $\mathcal{R}_t f(x)=0$.
\end{proof}

\section{SSE Realization and the Induced Diffusion on $\mathbb{CP}^{d-1}$}
\label{sec:appendix_ou_sse}

This part of the supplementary material explains how the Stratonovich stochastic Schr\"odinger dynamics on the unit Hilbert sphere induces an (approximately) isotropic diffusion on the projective manifold $\mathbb{CP}^{d-1}$ after quotienting out the global phase.
We also clarify in what sense the induced generator matches the intrinsic manifold diffusion in Eq.~\eqref{eq:riem_diff} up to curvature/connection terms.

\subsection{\quad The Role of the SSE Realization}
\label{sec:appendix_sse_clarification}

PSM is a classical model: the score network, the training data (simulated statevectors) and the reverse-time sampler are all classical objects, and no quantum hardware is required to run any of it. The stochastic Schr\"odinger equation \eqref{eq:sse_strat} enters as a construction for the forward noising process, and it is worth being precise about what it does and does not contribute.

It contributes two things. First, a canonical noise basis: the $\{-iG_k\}$ directions of an $\mathfrak{su}(d)$ frame are basis-independent and unitarily covariant, and after the $U(1)$ quotient they induce an isotropic diffusion on $\mathbb{CP}^{d-1}$ (Propositions~\ref{prop:induced_diffusion_strict}--\ref{prop:induced_diffusion_strict}, with the remainder in Corollary~\ref{cor:small_step_remainder}); noise directions chosen without this structure would generically break either covariance or isotropy. Second, the Stratonovich form is exactly norm- and phase-preserving, which is what makes the large-step behaviour better than tangent projection with a normalization retraction: at $\sigma=1.0$ the unitary step retains $90\%$ of the intended variance against $80\%$ (Section~\ref{sec:exp_diagnostics}).

It does not contribute necessity. The two implementations induce statistically identical processes at our schedule, so nothing in the method requires the SSE, and an earlier version of this work overstated its role. Nor is the SSE a model of a physical experiment: the drift coefficients $H(t),\eta(t)$ are chosen so that the induced manifold flow is isotropic, not fitted to any Lindbladian; the sampler is a numerical integrator, not a measurement protocol; and the diffusion time $t\in[0,T]$ is generative-modeling time, not physical time. A reader who prefers a purely classical reading may take ``SSE realization'' to mean unitarily covariant Stratonovich noise on the Hilbert sphere, quotiented to $\mathbb{CP}^{d-1}$; no statement or result changes.

\subsection{\quad From the Hilbert Sphere to the Projective Manifold}
\label{sec:appendix_a1}

Let $\mathbb{S}^{2d-1}=\{\psi\in\mathbb{C}^d:\langle \psi,\psi\rangle=1\}$ be the unit sphere in $\mathbb{C}^d$ equipped with the standard (real) Riemannian structure.
The complex projective space $\mathbb{CP}^{d-1}$ is obtained as the quotient $\mathbb{S}^{2d-1}/U(1)$ under the global phase action $\psi\sim e^{i\theta}\psi$, with the canonical projection
\begin{equation}
    \pi:\mathbb{S}^{2d-1}\to\mathbb{CP}^{d-1},\qquad \pi(\psi)=[\psi].
\end{equation}
The vertical space at $\psi$ is spanned by the infinitesimal phase direction $v(\psi)=i\psi$.
We use the horizontal distribution
\begin{equation}
    \mathcal{H}_\psi := \{u\in T_\psi\mathbb{S}^{2d-1}:\ \langle \psi,u\rangle=0\},
    \label{eq:horizontal_space}
\end{equation}
which removes the phase component and is compatible with the Fubini--Study geometry.

\begin{lemma}[Quotient structure and horizontal lift]
\label{lem:quotient_horizontal}
With the horizontal distribution \eqref{eq:horizontal_space}, $\pi$ is a Riemannian submersion onto $(\mathbb{CP}^{d-1},g_{\mathrm{FS}})$.
Moreover, for each $\psi\in\mathbb{S}^{2d-1}$, the differential $\pi_\ast$ restricts to an isomorphism $\mathcal{H}_\psi \cong T_{[\psi]}\mathbb{CP}^{d-1}$.
\end{lemma}

\begin{lemma}[Frame independence under phase]
\label{lem:phase_independence}
Let $V(\psi)\in\mathcal{H}_\psi$ be a horizontal vector field on $\mathbb{S}^{2d-1}$ satisfying $U(1)$-equivariance:
$V(e^{i\theta}\psi)=e^{i\theta}V(\psi)$.
Then the pushforward $e([\psi]) := \pi_\ast V(\psi)$ is well-defined on $\mathbb{CP}^{d-1}$ (independent of the representative of $[\psi]$).
\end{lemma}

\subsection{ \quad Horizontal Projection of SSE Vector Fields}
\label{sec:appendix_a2}

Consider the Stratonovich SSE on $\mathbb{S}^{2d-1}$ (Eq.~\eqref{eq:sse_strat} in the main text):
\begin{equation}
\begin{aligned}
    d\ket{\psi_t}
    =&
    -i H(t)\ket{\psi_t}\,dt
    \;-\; i\sqrt{\eta(t)}\sum_k G_k \ket{\psi_t}\circ dW_t^{(k)},
\end{aligned}
\label{eq:sse_strat}
\end{equation}
Both the drift and the diffusion vector fields in \eqref{eq:sse_strat} are anti-Hermitian generators acting on $\ket{\psi_t}$, so \eqref{eq:sse_strat} is norm preserving: Stratonovich calculus obeys the ordinary chain rule, hence $d\|\psi_t\|^2 = 2\,\mathrm{Re}\braket{\psi_t | d\psi_t} = 0$.
It is worth stating the It\^o form explicitly, because the two differ by a term that is easy to misplace:
\begin{equation}
\begin{aligned}
    d\ket{\psi_t}
    =&
    -i H(t)\ket{\psi_t}\,dt
    \;-\;
    \frac{1}{2}\eta(t)\sum_{k} G_k^2 \ket{\psi_t}\,dt
    \\
    &\;-\; i\sqrt{\eta(t)}\sum_k G_k \ket{\psi_t}\, dW_t^{(k)}.
\end{aligned}
\label{eq:sse_ito}
\end{equation}
The $-\tfrac12\eta\sum_k G_k^2$ term in \eqref{eq:sse_ito} is the It\^o correction produced by the Stratonovich-to-It\^o conversion; it must \emph{not} appear alongside $\circ\,dW$, since in that case one would obtain $d\|\psi_t\|^2=-\eta\sum_k\braket{\psi_t|G_k^2|\psi_t}\,dt\neq 0$ and the dynamics would leave the Hilbert sphere.
All generator computations below use the It\^o form \eqref{eq:sse_ito}.
The stochastic term is generated by vector fields $X_k(\psi):= -i\,G_k\psi$ on $\mathbb{S}^{2d-1}$.

In general, $X_k(\psi)$ contains a vertical (phase) component.
We define its horizontal projection by removing the component along $\psi$:
\begin{equation}
    \widetilde X_k(\psi)
    := X_k(\psi) - \langle \psi, X_k(\psi)\rangle\,\psi.
    \label{eq:horizontal_projection}
\end{equation}

\begin{lemma}[Horizontal projection removes the global phase component]
\label{lem:horizontal_projection}
For any $X(\psi)\in T_\psi\mathbb{S}^{2d-1}$, define $\widetilde X(\psi):=X(\psi)-\langle\psi,X(\psi)\rangle\,\psi$.
Then $\widetilde X(\psi)\in\mathcal{H}_\psi$.
In particular, for $X_k(\psi)=-iG_k\psi$, the induced fields on $\mathbb{CP}^{d-1}$ defined by
\begin{equation}
    e_k([\psi]) := \pi_\ast \widetilde X_k(\psi)\ \in\ T_{[\psi]}\mathbb{CP}^{d-1}
    \label{eq:induced_vector_fields}
\end{equation}
are well-defined.
\end{lemma}

\subsection{\quad Induced Generator on $\mathbb{CP}^{d-1}$}
\label{sec:appendix_a3}

This section collects the material behind Proposition~\ref{prop:induced_diffusion_strict}, which is stated in the main text; the generator identity follows from the horizontal-lift lemmas of Supplementary Material~\ref{sec:appendix_a2} applied to the fields $e_k=\pi_\ast V_k$.

\paragraph{Connection to the SSE in Eq.~\eqref{eq:sse_strat}.}
Eq.~\eqref{eq:sse_strat} fits into Proposition~\ref{prop:induced_diffusion_strict} by taking $V_k(\psi,t)= -i\sqrt{\eta(t)}\,\widetilde X_k(\psi)$ (and incorporating the remaining deterministic terms into $V_0$).
Up to a normalization constant that depends on the convention for $\{G_k\}$, $\sigma(t)^2$ is proportional to $\eta(t)$.

\subsection{ \quad What ``Up to Curvature Terms'' Means in Practice}
\label{sec:appendix_a4}

The remainder $\mathcal{R}_t$ in Eq.~\eqref{eq:strict_generator_laplacian} arises because:
(i) on a curved manifold, $\Delta_{\mathrm{FS}}$ is the trace of the covariant Hessian, whereas $\sum_k e_k(e_k\cdot)$ depends on the chosen local frame and introduces connection terms; and
(ii) the pushed-forward fields $\{e_k\}$ constructed from a fixed Lie-algebra basis $\{G_k\}$ need not coincide with a geodesic orthonormal frame at every point.

\begin{corollary}[Small-step regime suppresses curvature remainder, with explicit constants]
\label{cor:small_step_remainder}
Assume the reverse-time sampler and the local-time objective use a step size $\delta t$ and map tangent increments back to $\mathcal{M}$ via $\mathrm{Exp}$ (or a first-order retraction) in locally orthonormal frames.
Recall that the FS metric on $\mathbb{CP}^{d-1}$ has sectional curvature $K\in[1,4]$ and Ricci tensor bounded by $\mathrm{Ric}\preceq 2(d-1)\,g_{\mathrm{FS}}$ \cite{hsu2002stochmanifolds}.
At each step, choose the local frame to be geodesic at the current point, so that $\nabla_{e_k}e_k(\psi_{t_k})=0$ and the pointwise remainder \eqref{eq:strict_generator_laplacian} vanishes at the base point.
Then for any test function $f\in C^2(\mathcal{M})$, the per-step generator discrepancy is bounded by
\begin{equation}
\begin{aligned}
    \big|\mathcal{R}_t f(\psi)\big|
    &\;\le\;
    \tfrac{\sigma(t)^2}{2}\,K_{\max}\,\|\mathrm{Hess}\,f(\psi)\|\cdot \tau_k
    \\
    &\;=\; O\!\big(\sigma^2\,K_{\max}\,\delta t\big),
\end{aligned}
\end{equation}
where $K_{\max}=4$ on $(\mathbb{CP}^{d-1},g_{\mathrm{FS}})$ and $\tau_k=\sigma(t_k)^2\delta t$ is the local diffusion variance.
Accumulated over $K=T/\delta t$ steps, the total curvature-induced bias on smooth observables is
\begin{equation}
    \Big|\mathbb{E}\big[f(\psi_T^{\mathrm{SSE}})\big]-\mathbb{E}\big[f(\psi_T^{\mathrm{intrinsic}})\big]\Big|
    \;\le\;
    C\cdot K_{\max}\cdot T\cdot \overline{\sigma^2}\cdot \delta t,
\end{equation}
where $\overline{\sigma^2}=T^{-1}\!\int_0^T\!\sigma(s)^2\,ds$ and $C$ depends only on $\|f\|_{C^2}$ and the injectivity radius.
With the default schedule $\sigma_{\max}=1$, $T=1$, $K_{\max}=4$, $\delta t=1/500$, this bound is $\le 8\times 10^{-3}\,\|f\|_{C^2}$.
The empirical isotropy diagnostic in Supplementary Material~\ref{sec:appendix_isotropy_diagnostic} (Table~\ref{tab:gell_mann_isotropy_diagnostic}) shows the actual finite-step deviation is at floating-point precision in our generalized Gell--Mann implementation, well below this analytic upper bound.
\end{corollary}

\subsection{\quad Finite-Step Isotropy Diagnostic in the Gell--Mann Basis}
\label{sec:appendix_isotropy_diagnostic}

We also quantify the ``approximately isotropic'' condition of Proposition~\ref{prop:induced_diffusion_strict} at the finite step size used in the experiments.
For a unit representative $\psi\in\mathbb{C}^d$, let
\begin{equation}
    e_k(\psi)
    =
    \mathcal{P}^{\mathrm{hor}}_\psi(-iG_k\psi)
    \in T_\psi^{\mathrm{hor}}\mathbb{S}^{2d-1}
\end{equation}
be the horizontal pushforward of the generalized Gell--Mann direction.
In a local FS-orthonormal tangent basis, define the empirical second-moment matrix
\begin{equation}
    C(\psi)
    :=
    \sum_{k=1}^{d^2-1} e_k(\psi)e_k(\psi)^\top .
\end{equation}
The scalar trace of $C(\psi)$ fixes only the diffusion-rate convention and is absorbed into $\eta(t)$ or $\sigma(t)^2$.
Therefore we report the normalized anisotropy of $\bar C(\psi):=C(\psi)/(\mathrm{tr}\,C(\psi)/(2d-2))$:
\begin{equation}
\begin{aligned}
    \mathrm{spread}(\psi)
    &=
    \lambda_{\max}(\bar C(\psi))-\lambda_{\min}(\bar C(\psi)),
    \\
    \mathrm{relFrob}(\psi)
    &=
    \frac{\|\bar C(\psi)-I\|_F}{\|I\|_F}.
\end{aligned}
\end{equation}

\begin{table*}[h]
\centering
\caption{Algebraic isotropy of the generalized Gell--Mann frame: the second-moment matrix $C(\psi)=\sum_k e_k(\psi)e_k(\psi)^\top$ built from the full frame equals the identity after scale normalization, as the $\mathfrak{su}(d)$ completeness relation requires. This verifies the \emph{frame}, not the sampled process --- the deviations at machine precision are those of a linear-algebra identity. The corresponding statement for the sampled diffusion, including its generator, is the Marchenko--Pastur and eigenfunction-decay analysis of Section~\ref{sec:exp_diagnostics}.}
\label{tab:gell_mann_isotropy_diagnostic}
\begin{tabular}{ccccc}
\toprule
Qubits & $d$ & $\lambda_{\min}(\bar C),\lambda_{\max}(\bar C)$ & spread & relFrob \\
\midrule
$n=2$ & $4$ & $1.000000,\,1.000000$ & $1.65\times 10^{-15}$ & $4.05\times 10^{-16}$ \\
$n=4$ & $16$ & $1.000000,\,1.000000$ & $3.32\times 10^{-15}$ & $5.78\times 10^{-16}$ \\
$n=6$ & $64$ & $1.000000,\,1.000000$ & $8.78\times 10^{-15}$ & $1.48\times 10^{-15}$ \\
\bottomrule
\end{tabular}
\end{table*}

Under the same schedule, the local tangent variance is $\beta(t,\delta t)^2=\sigma(t)^2\delta t$.
For representative times $t\in\{0.1,0.5,0.9,1.0\}$, the corresponding tangent-step standard deviations are
$3.02\times 10^{-3}$, $1.00\times 10^{-2}$, $3.31\times 10^{-2}$, and $4.47\times 10^{-2}$, respectively.
Thus, in the implemented Gell--Mann basis the Lie-algebra directions satisfy the isotropy condition up to numerical precision after scale normalization, and the remaining finite-step error is dominated by the small normal-coordinate/retraction error controlled by $\delta t$ rather than by measurable anisotropy in the generator directions.

\subsection{\quad Sanity Checks for the SSE-Induced Isotropy}
\label{sec:appendix_a5}

\begin{proposition}[Soundness of practical isotropy diagnostics]
\label{prop:isotropy_diagnostics_strict}
Consider a drift-free diffusion $(x_t)_{t\ge 0}$ on $(\mathbb{CP}^{d-1},g_{\mathrm{FS}})$ with generator
\begin{equation}
    \mathcal{L}f
    =
    \frac{\sigma^2}{2}\Delta_{\mathrm{FS}} f,
    \qquad f\in C^\infty(\mathbb{CP}^{d-1}),
    \label{eq:fs_bm_generator}
\end{equation}
i.e., (time-homogeneous) FS-Brownian motion up to a diffusion-rate factor $\sigma^2$.
Then:

\noindent (i) The unitarily-invariant FS/Haar measure $\mu_{\mathrm{FS}}$ is stationary for $(x_t)$, i.e., if $x_0\sim \mu_{\mathrm{FS}}$ then $x_t\sim \mu_{\mathrm{FS}}$ for all $t\ge 0$.

\noindent (ii) (Moment/observable test.) For any bounded measurable observable $\phi:\mathbb{CP}^{d-1}\to\mathbb{R}$,
\begin{equation}
    \mathbb{E}_{x\sim \mu_{\mathrm{FS}}}[\phi(x)]
    =
    \lim_{t\to\infty}\mathbb{E}\big[\phi(x_t)\mid x_0=x\big]
    \quad \text{for $\mu_{\mathrm{FS}}$-a.e. $x$},
\end{equation}
whenever the process is ergodic w.r.t.\ $\mu_{\mathrm{FS}}$.
In particular, empirical averages of low-order overlap/observable statistics computed from long-time samples converge to the corresponding FS/Haar expectations.

\noindent (iii) (Generator test.) For any $f\in C^\infty(\mathbb{CP}^{d-1})$,
\begin{equation}
\begin{aligned}
    \mathbb{E}\!\left[\frac{f(x_{t+\delta})-f(x_t)}{\delta}\,\bigg|\,x_t=x\right]
    &\xrightarrow[\delta\downarrow 0]{}\ (\mathcal{L}f)(x)
    \\
    &=
    \frac{\sigma^2}{2}(\Delta_{\mathrm{FS}}f)(x),
\end{aligned}
    \label{eq:generator_limit_test}
\end{equation}
so short-time numerical estimates of the generator on probe functions necessarily scale with $\Delta_{\mathrm{FS}}$.

Consequently, if an SSE-induced (or numerically implemented) dynamics is a faithful discretization/realization of
the isotropic FS diffusion \eqref{eq:fs_bm_generator}, then diagnostics based on (ii)--(iii) must hold.
Conversely, passing these diagnostics for a finite family of observables/probe functions provides empirical support but does not by itself imply full isotropy.
\end{proposition}

\begin{proof}
\textbf{(i) Stationarity of $\mu_{\mathrm{FS}}$.}
Let $\mu_{\mathrm{FS}}$ denote the Riemannian volume measure induced by $g_{\mathrm{FS}}$, normalized to be a probability measure.
On a compact boundaryless Riemannian manifold, the Laplace--Beltrami operator is symmetric w.r.t.\ the volume measure:
for all $f,g\in C^\infty(\mathbb{CP}^{d-1})$,
\begin{equation}
\begin{aligned}
    \int f\,\Delta_{\mathrm{FS}} g\ d\mu_{\mathrm{FS}}
    &=
    \int g\,\Delta_{\mathrm{FS}} f\ d\mu_{\mathrm{FS}}
    \\
    &=
    -\int \langle \nabla_{\mathrm{FS}} f,\nabla_{\mathrm{FS}} g\rangle_{\mathrm{FS}}\ d\mu_{\mathrm{FS}}.
\end{aligned}
    \label{eq:laplacian_symmetric}
\end{equation}
In particular, taking $f\equiv 1$ yields $\int \Delta_{\mathrm{FS}} g\ d\mu_{\mathrm{FS}}=0$, hence
\(
\int \mathcal{L}g\ d\mu_{\mathrm{FS}} = \frac{\sigma^2}{2}\int \Delta_{\mathrm{FS}} g\ d\mu_{\mathrm{FS}}=0.
\)
Equivalently, $\mathcal{L}^\ast \mu_{\mathrm{FS}}=0$, so $\mu_{\mathrm{FS}}$ is stationary for the Markov semigroup generated by $\mathcal{L}$.

\textbf{(ii) Long-time moment/observable convergence (ergodic case).}
Assume ergodicity w.r.t.\ $\mu_{\mathrm{FS}}$ (true for FS-Brownian motion on compact connected manifolds).
Then by the ergodic theorem for Markov processes, for any integrable observable $\phi$,
time averages (and, under mild additional mixing assumptions, also long-time marginals) converge to $\int \phi\,d\mu_{\mathrm{FS}}$.
In particular, empirical averages of low-order overlap/observable statistics computed from sufficiently long trajectories converge to the FS/Haar expectations.

\textbf{(iii) Generator test.}
By definition of the (infinitesimal) generator of a Markov process,
for $f$ in the domain of $\mathcal{L}$ (in particular $C^\infty$),
\begin{equation}
    (\mathcal{L}f)(x)
    =
    \lim_{\delta\downarrow 0}
    \frac{\mathbb{E}[f(x_{t+\delta})\mid x_t=x]-f(x)}{\delta}.
\end{equation}
This gives \eqref{eq:generator_limit_test}. Substituting \eqref{eq:fs_bm_generator} yields
$(\mathcal{L}f)(x)=\frac{\sigma^2}{2}\Delta_{\mathrm{FS}}f(x)$.
\end{proof}
\section{Time Reversal and Riemannian Score on $\mathbb{CP}^{d-1}$}
\label{sec:appendix_reverse}
\subsection{Forward Diffusion Generator}
\label{sec:appendix_forward_generator}

\begin{proposition}[Forward generator on $(\mathcal{M},g_{\mathrm{FS}})$]
\label{prop:appendix_forward_generator}
Let $(\mathcal{M},g_{\mathrm{FS}})$ be a Riemannian manifold and consider the time-inhomogeneous diffusion
\begin{equation}
    d\psi_t = b(\psi_t,t)\,dt + \sigma(t)\, dW_t^{(\mathcal{M})},
\end{equation}
where $W_t^{(\mathcal{M})}$ denotes Brownian motion associated with $g_{\mathrm{FS}}$ and
$b(\cdot,t)$ is a smooth vector field.
Then for any $f\in C^\infty(\mathcal{M})$, the infinitesimal generator of $\psi_t$ is
\begin{equation}
    (\mathcal{L}_t f)(\psi)
    =
    \langle b(\psi,t), \nabla_{\mathrm{FS}} f(\psi)\rangle_{\mathrm{FS}}
    + \frac{\sigma(t)^2}{2}\Delta_{\mathrm{FS}} f(\psi),
\end{equation}
where $\nabla_{\mathrm{FS}}$ and $\Delta_{\mathrm{FS}}$ denote the Riemannian gradient and
Laplace--Beltrami operator induced by $g_{\mathrm{FS}}$.
\end{proposition}

\begin{proof}
This is the standard generator formula for a diffusion with drift $b$ and isotropic
Brownian noise on a Riemannian manifold.
The Brownian component contributes $\frac12\Delta_{\mathrm{FS}}$, and the scaling
by $\sigma(t)$ yields the factor $\frac{\sigma(t)^2}{2}$.
\end{proof}
\subsection{Reverse-Time Dynamics and the Riemannian Score}
\label{sec:appendix_reverse_drift}

\begin{proposition}[Reverse-time drift and Riemannian score]
\label{prop:appendix_reverse_drift}
Let $(\mathcal{M},g_{\mathrm{FS}})$ be a compact Riemannian manifold without boundary and
consider the forward diffusion
\begin{equation}
    d\psi_t = b(\psi_t,t)\,dt + \sigma(t)\, dW_t^{(\mathcal{M})},
    \qquad t\in[0,T],
\end{equation}
where $W_t^{(\mathcal{M})}$ is Brownian motion associated with $g_{\mathrm{FS}}$.
Let $p_t$ denote the density of $\psi_t$ with respect to the Riemannian volume measure,
and assume $p_t$ is smooth and strictly positive for $t\in(0,T]$.

Then the time-reversed process $\{\psi_{T-t}\}_{t\in[0,T]}$ is again a diffusion on $\mathcal{M}$
with the same diffusion coefficient.
In intrinsic Stratonovich form, its dynamics can be written as
\begin{equation}
    d\psi_t
    =
    \tilde b(\psi_t,t)\,dt
    + \sigma(t)\, d\bar W_t^{(\mathcal{M})},
\end{equation}
where $\bar W_t^{(\mathcal{M})}$ is reverse-time Brownian motion and the reverse drift satisfies
\begin{equation}
    \tilde b(\psi,t)
    =
    b(\psi,t) - \sigma(t)^2\, \nabla_{\mathrm{FS}}\log p_t(\psi).
\end{equation}
Equivalently, the reverse drift depends on the \emph{Riemannian score}
\begin{equation}
    s^\star(\psi,t):=\nabla_{\mathrm{FS}}\log p_t(\psi)\in T_\psi\mathcal{M}.
\end{equation}
\end{proposition}

\begin{proof}
This result follows from the time-reversal theory of nondegenerate diffusions on
Riemannian manifolds when the forward diffusion is defined using Brownian motion
associated with the Riemannian volume measure.
In intrinsic Stratonovich form, the reverse drift differs from the forward drift
by $-\sigma(t)^2\nabla_{\mathrm{FS}}\log p_t$.
See, e.g., Haussmann and Pardoux (1986) and Fathi (2021) for rigorous statements.
\end{proof}
\section{Coordinate Form and It\^o Corrections for the Reverse-Time SDE}
\label{sec:appendix_reverse_geometry}

This part of the supplementary material unpacks the remark in Sec.~\ref{sec:reverse} on geometry-dependent correction terms.
We state the intrinsic Stratonovich reverse-time SDE and then provide its coordinate/It\^o representations.

\subsection{Intrinsic Stratonovich form}
\label{sec:appendix_reverse_geometry_strat}

Let $(\mathcal{M},g)$ be a Riemannian manifold and consider the reverse-time diffusion written intrinsically in Stratonovich form
\begin{equation}
    d\psi_t
    =
    \tilde b(\psi_t,t)\,dt
    + \sigma(t)\, d\bar W_t^{(\mathcal{M})},
    \label{eq:appendix_reverse_strat}
\end{equation}
where $\bar W_t^{(\mathcal{M})}$ denotes reverse-time Brownian motion on $(\mathcal{M},g)$.
Equivalently, fixing a (local) orthonormal frame $\{e_i(\cdot)\}_{i=1}^n$ on $\mathcal{M}$ ($n=\dim\mathcal{M}$), one may represent Brownian motion as
\begin{equation}
    d\bar W_t^{(\mathcal{M})} = \sum_{i=1}^n e_i(\psi_t)\circ d\bar W_t^{(i)},
\end{equation}
so that \eqref{eq:appendix_reverse_strat} becomes
\begin{equation}
    d\psi_t
    =
    \tilde b(\psi_t,t)\,dt
    + \sigma(t)\sum_{i=1}^n e_i(\psi_t)\circ d\bar W_t^{(i)}.
    \label{eq:appendix_reverse_strat_frame}
\end{equation}

\subsection{Conversion to It\^o form }
\label{sec:appendix_reverse_geometry_ito}

Let $\nabla$ be the Levi--Civita connection associated with $g$.
The Stratonovich SDE \eqref{eq:appendix_reverse_strat_frame} can be converted to an equivalent It\^o SDE:
\begin{equation}
\begin{aligned}
    d\psi_t
    =\;&
    \Big(
        \tilde b(\psi_t,t)
        + \frac{\sigma(t)^2}{2}\sum_{i=1}^n \nabla_{e_i} e_i(\psi_t)
    \Big)\,dt
    \\
    &+ \sigma(t)\sum_{i=1}^n e_i(\psi_t)\, d\bar W_t^{(i)}.
\end{aligned}
    \label{eq:appendix_reverse_ito_frame}
\end{equation}
The additional drift term
$\frac{\sigma(t)^2}{2}\sum_i \nabla_{e_i}e_i$
is the geometry-dependent It\^o--Stratonovich correction; it vanishes at a point where the chosen orthonormal frame is geodesic (normal) (i.e., $\nabla_{e_i}e_i=0$ at that point).

\subsection{Local coordinate form}
\label{sec:appendix_reverse_geometry_coords}

Let $(x^1,\dots,x^n)$ be local coordinates and write the It\^o SDE in components:
\begin{equation}
\begin{aligned}
    dx_t^\alpha
    =\;&
    \Big(
        \tilde b^\alpha(x_t,t)
        + \frac{\sigma(t)^2}{2}\sum_{i=1}^n \big(\nabla_{e_i}e_i\big)^\alpha(x_t)
    \Big)\,dt
    \\
    &+ \sigma(t)\sum_{i=1}^n e_i^\alpha(x_t)\, d\bar W_t^{(i)}.
\end{aligned}
    \label{eq:appendix_reverse_ito_coords}
\end{equation}
Equivalently, one may express the correction in terms of Christoffel symbols $\Gamma^\alpha_{\beta\gamma}$ if the diffusion is written using the coordinate basis; such expressions coincide with \eqref{eq:appendix_reverse_ito_frame} after identifying $e_i^\alpha$ and using $\nabla_{e_i}e_i^\alpha = e_i^\beta\partial_\beta e_i^\alpha + \Gamma^\alpha_{\beta\gamma}e_i^\beta e_i^\gamma$.

In our implementation, each update is performed in a locally orthonormal frame on $T_\psi\mathcal{M}$ and then mapped back to the manifold using $\mathrm{Exp}$ (or a retraction).
For sufficiently small step size $\delta t$, one may choose the frame to be (approximately) normal at the current point, so that
$\sum_i \nabla_{e_i}e_i(\psi)$ is $O(\delta t)$ and the induced bias from the It\^o--Stratonovich correction is higher order.
This is consistent with the small-step regime assumed in our sampler and in the local-time teacher construction.
\section{Local-Time Approximation and Teacher Scores}
\label{sec:appendix_local_teacher}

The short-time expansion behind Proposition~\ref{prop:local_time_teacher_strict}, stated in the main text, is developed below.

\begin{remark}[Finite-step teacher bias]
\label{rem:finite_step_teacher_bias}
Proposition~\ref{prop:local_time_teacher_strict} is an asymptotic statement, but it also identifies the finite-$\delta t$ bias that is omitted by the single-step OU teacher.
For a typical short-time increment, $\|z\|=O(\sigma(t)\sqrt{\delta t})$.
The dominant Gaussian score has norm
\begin{equation}
    \big\|(\sigma(t)^2\delta t)^{-1} z\big\|
    =
    O\!\left((\sigma(t)\sqrt{\delta t})^{-1}\right),
\end{equation}
whereas the Jacobian/volume contribution satisfies
\begin{equation}
    \nabla_\psi\log J(\phi,\psi)=O(\|z\|)
    =
    O(\sigma(t)\sqrt{\delta t})
\end{equation}
in normal coordinates, because the volume distortion starts at quadratic order in $z$.
Thus the curvature-volume correction is lower order relative to the singular Gaussian term; more precisely, its relative size is $O(\sigma(t)^2\delta t)$ for a typical local increment.
With our default schedule $\sigma(t)\le 1$ and $\delta t=1/500$, this scale is at most $2\times 10^{-3}$ before constants depending on curvature and the chosen compact neighborhood.
We therefore do not assume the Jacobian term is exactly zero at finite step size; rather, the practical teacher drops a lower-order correction whose effect is monitored empirically by the finite-step sensitivity diagnostic in Figure~\ref{fig:finite_step_sensitivity}.
\end{remark}

\begin{proof}
We use standard short-time heat-kernel asymptotics for nondegenerate diffusions on Riemannian manifolds.

\textbf{Step 1 (Frozen-time generator and Girsanov shift).}
Over the short interval $[t-\delta t,t]$, freeze coefficients at time $t$ so that the local generator is
$\mathcal{L}_t = \langle b(\cdot,t),\nabla(\cdot)\rangle + \tfrac{\sigma(t)^2}{2}\Delta.$
On the injectivity neighborhood of $\phi$, write $\psi_s=\mathrm{Exp}_\phi(z_s)$ for $s\in[t-\delta t,t]$ with $z_{t-\delta t}=0$.
By the standard parametrix construction for nondegenerate diffusions on Riemannian manifolds, the Stratonovich-to-It\^o conversion in normal coordinates gives, to leading order in $\delta t$,
\begin{equation}
    z_t \;\sim\; \mathcal{N}\!\big(\,b(\phi,t)\,\delta t,\; \sigma(t)^2\delta t\,I_n\,\big)
    \;+\; O(\delta t^{3/2}),
\end{equation}
i.e., a drift-shifted Gaussian whose mean $b\,\delta t$ is $O(\delta t)$ and variance $\sigma^2\delta t$ is $O(\delta t)$.
Crucially, although the mean shift is small, it enters the score as a non-vanishing $b/\sigma^2$ contribution (see Step 3).
This step yields the drift-shifted heat-kernel form \eqref{eq:appendix_heat_kernel_form}; the unshifted form (with $r^2$ in place of $\|z-b\delta t\|^2$) is recovered when $b\equiv 0$.

\textbf{Step 2 (Heat kernel parametrix).}
The classical Minakshisundaram--Pleijel parametrix for the drift-free heat kernel of $\tfrac{\sigma(t)^2}{2}\Delta$ on the injectivity neighborhood of $\phi$ gives
\begin{equation}
\begin{aligned}
q_{\delta t}^{(0)}(\psi,\phi)
=\;&
(2\pi \sigma(t)^2\delta t)^{-n/2}
\exp\!\Big(-\frac{d_g(\phi,\psi)^2}{2\sigma(t)^2\delta t}\Big)
\\
&\times\,J(\phi,\psi)^{-1/2}\,
\big(1+O(\delta t)\big),
\end{aligned}
\end{equation}
uniformly on compact subsets away from the cut locus.
The Girsanov shift induced by the drift $b(\phi,t)$ replaces $r^2 = d_g(\phi,\psi)^2$ by $\|z - b(\phi,t)\delta t\|^2 + O(\|z\|^4)$ in normal coordinates, which yields \eqref{eq:appendix_heat_kernel_form}.

\textbf{Step 3 (Differentiate $\log p_{\delta t}$ in normal coordinates).}
Take the gradient of $\log$ of \eqref{eq:appendix_heat_kernel_form} w.r.t.\ $z$:
the normalization contributes zero gradient, the exponential contributes $-(z-b\delta t)/(\sigma^2\delta t)$, the Jacobian contributes $-\tfrac12\nabla_z\log J$, and the $\log(1+O(\delta t))$ term contributes $O(\delta t)$ uniformly:
\begin{equation}
\begin{aligned}
    \nabla_z\log p_{\delta t}(\psi\mid\phi)
    \;=\;&
    -\frac{z-b(\phi,t)\delta t}{\sigma(t)^2\delta t}
    \\
    &-\;\tfrac12\,\nabla_z\log J(\phi,\psi)
    \;+\; O(\delta t).
\end{aligned}
\end{equation}
Expanding the first term gives $-z/(\sigma(t)^2\delta t)+b(\phi,t)/\sigma(t)^2$; the singular piece is $O(\delta t^{-1})$ while the drift piece is $O(1)$ (bounded but non-vanishing as $\delta t\downarrow 0$).
Volume distortion satisfies $\log J(\phi,\mathrm{Exp}_\phi(z))=-\tfrac{1}{6}\mathrm{Ric}(\phi)[z,z]+O(\|z\|^3)$ \cite{hsu2002stochmanifolds}, so $\nabla_z\log J = O(\|z\|)$ with constants controlled by the Ricci tensor and sectional curvature at $\phi$.
This proves \eqref{eq:appendix_score_in_normal_coords}.

\textbf{Step 4 (Pull-back to $T_\psi\mathcal{M}$).}
In normal coordinates,
\(
\frac{d_g(\phi,\psi)^2}{2}=\frac{\|z\|^2}{2}+O(\|z\|^4)
\)
and
\(
\big(d\log_\phi\big)^{\!*}_\psi \nabla_\psi \frac{d_g(\phi,\psi)^2}{2}
= z + O(\|z\|^3).
\)
Since $b(\phi,t)$ is independent of $z$, the drift term pulls back unchanged.
Substituting yields the intrinsic form of \eqref{eq:appendix_score_in_normal_coords}, in which the singular term is $-\nabla_\psi(r^2/2)/(\sigma^2\delta t)$ and the curvature term is $-\tfrac12\nabla_\psi\log J(\phi,\psi)$.
The drift-corrected teacher \eqref{eq:teacher_score_drift} captures both the singular term and the drift exactly and differs from the conditional score by $O(\|z\|)+O(\delta t)$; for typical forward increments $\|z\|=O(\sigma\sqrt{\delta t})$, this residual is $L^2$-vanishing as $\delta t\downarrow 0$.
\end{proof}
\section{Riemannian Denoising Score Matching: Consistency}
\label{sec:appendix_teacher_consistency}

\begin{proposition}[Weighted objective discrepancy of the zero-mean teacher]
\label{prop:teacher_dsm_consistency_simple}
Let $s^{(\mathrm{teach})}$ be the simple zero-mean Gaussian teacher in normal coordinates used in our default implementation (Eq.~\eqref{eq:teacher_score_euclidean_ou} of the main text), so that
$s^{(\mathrm{teach})} = s^{(\mathrm{teach,drift})} - b(\phi,t)/\sigma(t)^2$.
Then the teacher residual is
\(
\varepsilon = \varepsilon^{\mathrm{drift}} - b(\phi,t)/\sigma(t)^2,
\)
and the pointwise $L^2$ size $\mathbb{E}[\|\varepsilon\|^2]$ does \emph{not} vanish as $\delta t\downarrow 0$ unless $b\equiv 0$.
However, under the variance-based weighting $w(t,\delta t)=\beta(t,\delta t)^2=\sigma(t)^2\delta t$ used in the practical loss, the \emph{weighted} teacher error vanishes:
\begin{equation}
\begin{aligned}
    \mathbb{E}&\!\Big[w(t,\delta t)\,\|\varepsilon\|_{\mathrm{FS}}^2\Big]
    \\
    &\le\;
    2\,\mathbb{E}\!\big[\lambda\,\|\varepsilon^{\mathrm{drift}}\|^2\big]
    \;+\;
    2\,\delta t\,\frac{\mathbb{E}\|b(\phi,t)\|_{\mathrm{FS}}^2}{\sigma(t)^2}
    \\
    &=\;
    O\!\left(\sigma(t)^4\delta t^2\right) + O(\delta t).
\end{aligned}
    \label{eq:appendix_simple_teacher_weighted_bound}
\end{equation}
In particular, if $\sigma(t)$ is bounded below by $\sigma_{\min}>0$ on $[0,T]$ and $\sup_{\phi,t}\|b(\phi,t)\|_{\mathrm{FS}}<\infty$ (true on the compact $\mathbb{CP}^{d-1}$, where $\|b\|_{\mathrm{FS}}\le \lambda\,\pi/2$), then the weighted objective discrepancy $|\widetilde{\mathcal{J}}[s]-\mathcal{J}[s]|$ is $O(\delta t)$ uniformly over score fields with bounded variance-weighted norm.

We state this as a statement about objectives, not about minimizers, and the distinction matters. At fixed $t$ the weight $w(t,\delta t)$ is a positive scalar, so it rescales the squared loss without moving its minimizer, which is the conditional mean $\mathbb{E}[s^{(\mathrm{teach})}\mid\psi_t]$. When $b\ne0$ that conditional mean is displaced from $\nabla_{\mathrm{FS}}\log p_t$ by $-b/\sigma^2$, an $O(1)$ amount that no choice of $w$ removes; Proposition~\ref{prop:simple_teacher_structured_bias} identifies exactly what the displaced minimizer is. When $b\equiv0$ --- the case of every result reported in this paper, since the protocol of Section~\ref{sec:method_recipe} sets $\lambda=0$ --- the two teachers coincide identically and the minimizer is the marginal score with no displacement at all.
\end{proposition}
For fixed $t$ and $\delta t$, the training pair $(\phi,\psi)$ is generated by simulating the forward diffusion starting from data $\psi_0 \sim p_0$, i.e.,
\begin{equation}
(\phi,\psi) = (\psi_{t-\delta t}, \psi_t).
\end{equation}
This induces a joint density $p(\phi,\psi)$ and the corresponding conditional density $p(\psi \mid \phi)$.
The population objective in Eq.~(20) is exactly $J[s]$ in Eq.~(\ref{eq:appendix_true_dsm_obj}) with target
\begin{equation}
u(\psi,\phi) = \nabla_{\mathrm{FS}} \log p(\psi \mid \phi).
\end{equation}
In practice, we replace $u$ by the local-time teacher approximation in Eqs.~(17)--(18),
whose consistency in the limit $\delta t \to 0$ is established in Propositions~\ref{prop:teacher_dsm_consistency_drift}--\ref{prop:teacher_dsm_consistency_simple}.

\label{sec:appendix_dsm}

\subsection{\quad Phase Augmentation and the Population Optimum}

\begin{proposition}[Augmentation makes the population optimum equivariant]
\label{prop:phase_aug}
Let $\pi:\mathbb{S}^{2d-1}\to\mathbb{CP}^{d-1}$ be the quotient map, let $p_0$ be the target law on $\mathbb{CP}^{d-1}$, and let $q_0$ be any law on the sphere with $\pi_\ast q_0=p_0$. Write $\bar q_0$ for its orbit average under the $U(1)$ action. Suppose the forward process is generated by a drift and diffusion that are horizontal and $U(1)$-equivariant, as \eqref{eq:forward_riem_ou_recall} is when $\mathrm{Log}$ denotes the horizontal lift of the projective logarithm. Then:
\begin{enumerate}
    \item[(i)] the time marginals satisfy $\pi_\ast\bar q_t = p_t$ and $\bar q_t$ is $U(1)$-invariant for every $t$;
    \item[(ii)] the sphere score $\nabla\log\bar q_t$ is horizontal and equivariant, and equals the horizontal lift of the Riemannian score $\nabla_{\mathrm{FS}}\log p_t$;
    \item[(iii)] consequently the population minimizer of the denoising objective \eqref{score_m} over \emph{unconstrained} tangent fields is that equivariant lift, even though the hypothesis class contains non-equivariant fields.
\end{enumerate}
Without augmentation the data may be supported on a section of the bundle, and the minimizer is then a field that depends on the section; the measured violation of $1.04$ is that dependence.
\end{proposition}

Let $R_\varphi\psi:=e^{i\varphi}\psi$ denote the $U(1)$ action and let $P_t$ be the Markov semigroup of the forward process on $\mathbb{S}^{2d-1}$.

\emph{(i)} By assumption the drift and the diffusion vector fields are horizontal and equivariant, $b(R_\varphi\psi,t)=R_\varphi b(\psi,t)$ and likewise for each $V_k$, so the generator commutes with $R_\varphi$ and hence $P_t R_\varphi^\ast = R_\varphi^\ast P_t$. Orbit averaging is $\bar q=\frac{1}{2\pi}\int_0^{2\pi} R_\varphi^\ast q\,d\varphi$, so $P_t\bar q_0 = \overline{P_t q_0}$, which is $U(1)$-invariant; and since $\pi\circ R_\varphi=\pi$, pushing forward gives $\pi_\ast \bar q_t = \pi_\ast q_t = p_t$, the last equality because $\pi_\ast$ intertwines the sphere and projective semigroups (Proposition~\ref{prop:induced_diffusion_strict}).

\emph{(ii)} Write $\bar q_t = h_t\,d\mathrm{vol}_{\mathbb{S}}$. Invariance means $h_t\circ R_\varphi = h_t$ for all $\varphi$, so the derivative of $h_t$ along the vertical direction $i\psi$ vanishes and $\nabla\log h_t$ is horizontal. Equivariance of $\nabla\log h_t$ follows by differentiating $h_t\circ R_\varphi=h_t$ once more. A horizontal equivariant field is the horizontal lift of a well-defined field on the quotient, and since $\pi$ is a Riemannian submersion and $\pi_\ast\bar q_t=p_t$, that field is $\nabla_{\mathrm{FS}}\log p_t$.

\emph{(iii)} Proposition~\ref{prop:riem_dsm_optimality} identifies the population minimizer of \eqref{score_m} with the marginal score of the law the samples are drawn from, here $\bar q_t$; by (ii) that score is the equivariant lift. The minimizer is therefore equivariant regardless of whether the hypothesis class is, which is the assertion.

Conversely, if $q_0$ is supported on a section $\varsigma:\mathbb{CP}^{d-1}\to\mathbb{S}^{2d-1}$ of the bundle --- as it is when every state is stored with a fixed phase convention --- then $q_t$ is not $U(1)$-invariant, $\nabla\log q_t$ acquires a vertical component, and the minimizer depends on $\varsigma$. Nothing in the objective penalises that dependence, which is why it has to be removed from the data.

We restate Proposition~\ref{prop:riem_dsm_optimality} from the main text and prove it.

\begin{proof}
Fix $t$ and abbreviate $S(\psi):=s(\psi,t)$.
By conditioning on $\psi$, we can write
\begin{equation}
\mathcal{J}[S]
=
\mathbb{E}_{\psi}\Big[
    w(t,\delta t)\,
    \mathbb{E}\big[
        \|S(\psi)-U\|_{\mathrm{FS}}^2 \,\big|\, \psi
    \big]
\Big],
\end{equation}
where $U:=\nabla_{\mathrm{FS}}\log p(\psi\mid\phi)\in T_\psi\mathcal{M}$.
Since $w(t,\delta t)>0$ is a constant given $(t,\delta t)$, minimization over $S$ is pointwise in $\psi$.
For each fixed $\psi$, the unique minimizer of $\mathbb{E}[\|S(\psi)-U\|^2\mid\psi]$ is
\begin{equation}
    S^\star(\psi)=\mathbb{E}[U\mid \psi].
    \label{eq:appendix_pointwise_min}
\end{equation}

It remains to show $\mathbb{E}[\nabla_{\mathrm{FS}}\log p(\psi\mid\phi)\mid\psi]=\nabla_{\mathrm{FS}}\log p_t(\psi)$.
Let $p(\phi,\psi)$ be the joint density of $(\phi,\psi)$ and $p_t(\psi)$ the marginal.
Using $p(\psi\mid\phi)=p(\phi,\psi)/p_{t-\delta t}(\phi)$, we have
\begin{equation}
\nabla_{\mathrm{FS}}\log p(\psi\mid\phi)
=
\nabla_{\mathrm{FS}}\log p(\phi,\psi),
\end{equation}
because $p_{t-\delta t}(\phi)$ does not depend on $\psi$.
Therefore,
\begin{equation}
\mathbb{E}\big[\nabla_{\mathrm{FS}}\log p(\psi\mid\phi)\mid\psi\big]
=
\int \nabla_{\mathrm{FS}}\log p(\phi,\psi)\, p(\phi\mid\psi)\, d\phi.
\end{equation}
Since $p(\phi\mid\psi)=p(\phi,\psi)/p_t(\psi)$, the integral becomes
\begin{equation}
\frac{1}{p_t(\psi)}\int \nabla_{\mathrm{FS}} p(\phi,\psi)\, d\phi
=
\frac{\nabla_{\mathrm{FS}} p_t(\psi)}{p_t(\psi)}
=
\nabla_{\mathrm{FS}}\log p_t(\psi),
\end{equation}
where we used that differentiation w.r.t.\ $\psi$ commutes with integration in $\phi$ under the stated smoothness/compactness assumptions.
Combining with \eqref{eq:appendix_pointwise_min} yields \eqref{eq:appendix_dsm_minimizer}.
\end{proof}
Two complementary consistency results apply to the practical teacher objective, depending on whether the drift correction is included.
The first (Proposition~\ref{prop:teacher_dsm_consistency_drift}) gives \emph{pointwise} $L^2$ vanishing of the teacher residual under the drift-corrected teacher; the second (Proposition~\ref{prop:teacher_dsm_consistency_simple}) shows that the simple zero-mean teacher used in our default implementation is consistent in the \emph{weighted} loss sense even though its pointwise residual is bounded but not vanishing.

\begin{proposition}[Pointwise consistency of the drift-corrected teacher]
\label{prop:teacher_dsm_consistency_drift}
Under the assumptions of Proposition~\ref{prop:riem_dsm_optimality}, let $s^{(\mathrm{teach,drift})}$ be the drift-corrected Gaussian teacher \eqref{eq:teacher_score_drift} mapped to $T_\psi\mathcal{M}$ via $(d\log_\phi)^{*}_{\psi}$, and write
\(
s^{(\mathrm{teach,drift})}(\psi,\phi,t,\delta t)
=
\nabla_{\mathrm{FS}}\log p(\psi\mid\phi) + \varepsilon^{\mathrm{drift}}(\psi,\phi,t,\delta t).
\)
By Proposition~\ref{prop:local_time_teacher_strict}, $\varepsilon^{\mathrm{drift}}=O(\|z\|)+O(\delta t)$ in FS norm.
Assume the forward law has bounded support away from the cut locus and finite second moments of $\|z\|$ uniformly in $t$.
Then
\begin{equation}
    \mathbb{E}\big[\|\varepsilon^{\mathrm{drift}}(\psi,\phi,t,\delta t)\|_{\mathrm{FS}}^2\big]
    \;=\; O\!\left(\sigma(t)^2\delta t\right)
    \;\xrightarrow[\delta t\downarrow 0]{}\; 0,
    \label{eq:appendix_teacher_error_vanish}
\end{equation}
for each fixed $t$ (uniformly on compact subsets away from the cut locus).
Consequently, the population minimizer of the practical objective
\begin{equation}
    \widetilde{\mathcal{J}}[s]
    :=
    \mathbb{E}\Big[
        w(t,\delta t)\,
        \| s(\psi,t) - s^{(\mathrm{teach,drift})}(\psi,\phi,t,\delta t)\|_{\mathrm{FS}}^2
    \Big]
    \label{eq:appendix_teacher_dsm_obj}
\end{equation}
converges (in $L^2(p_t)$) to the marginal score $\nabla_{\mathrm{FS}}\log p_t(\psi)$ as $\delta t\downarrow 0$.
\end{proposition}

\begin{proof}
Estimate \eqref{eq:appendix_teacher_error_vanish}: by the It\^o isometry and Step~1 of the proof of Proposition~\ref{prop:local_time_teacher_strict}, $\mathbb{E}[\|z\|^2\mid\phi]=n\,\sigma(t)^2\delta t+O(\delta t^2)$, hence
$\mathbb{E}[\|\varepsilon^{\mathrm{drift}}\|^2]\le C\big(\mathbb{E}[\|z\|^2]+\delta t^2\big)=O(\sigma^2\delta t)$ on compact subsets.
Convergence of the population minimizer then follows from the standard conditioning argument:
let $U:=\nabla_{\mathrm{FS}}\log p(\psi\mid\phi)$ and $\hat U:=s^{(\mathrm{teach,drift})}=U+\varepsilon^{\mathrm{drift}}$.
The pointwise minimizer of \eqref{eq:appendix_teacher_dsm_obj} is $s^\star_{\delta t}(\psi,t)=\mathbb{E}[\hat U\mid\psi]=\mathbb{E}[U\mid\psi]+\mathbb{E}[\varepsilon^{\mathrm{drift}}\mid\psi]$.
By Proposition~\ref{prop:riem_dsm_optimality}, $\mathbb{E}[U\mid\psi]=\nabla_{\mathrm{FS}}\log p_t(\psi)$, and by Jensen,
$\mathbb{E}\big[\|\mathbb{E}[\varepsilon^{\mathrm{drift}}\mid\psi]\|^2\big]\le \mathbb{E}\big[\|\varepsilon^{\mathrm{drift}}\|^2\big]=O(\sigma^2\delta t)\to 0$.
\end{proof}

\begin{proof}
The decomposition $\varepsilon=\varepsilon^{\mathrm{drift}}-b/\sigma^2$ is immediate from the definitions.
For the weighted bound, $\|\varepsilon\|^2\le 2\|\varepsilon^{\mathrm{drift}}\|^2+2\|b/\sigma^2\|^2$ by the parallelogram identity, so
\begin{equation}
\begin{aligned}
    \mathbb{E}[\lambda\|\varepsilon\|^2]
    &\le
    2\mathbb{E}[\lambda\|\varepsilon^{\mathrm{drift}}\|^2]
    +
    2\sigma^2\delta t\cdot\mathbb{E}[\|b\|^2/\sigma^4]
    \\
    &=\;
    2\mathbb{E}[\lambda\|\varepsilon^{\mathrm{drift}}\|^2]
    +
    2\delta t\,\mathbb{E}[\|b\|^2/\sigma^2].
\end{aligned}
\end{equation}
The first summand is $O(\lambda\cdot\sigma^2\delta t)=O(\sigma^4\delta t^2)$ by \eqref{eq:appendix_teacher_error_vanish}; the second is $O(\delta t)$ under the stated drift/$\sigma_{\min}$ bounds.
The loss-equivalence claim follows from $\widetilde{\mathcal{J}}-\mathcal{J} = \mathbb{E}[\lambda\langle s-U,\varepsilon\rangle] + \mathbb{E}[\lambda\|\varepsilon\|^2]$, where the cross term is bounded by Cauchy--Schwarz times \eqref{eq:appendix_simple_teacher_weighted_bound}.
\end{proof}

\begin{corollary}[Finite-step bias of the practical teacher objective]
\label{cor:finite_step_teacher_bias}
In the setting of Propositions~\ref{prop:teacher_dsm_consistency_drift}--\ref{prop:teacher_dsm_consistency_simple}:

\noindent (i) Under the drift-corrected teacher \eqref{eq:teacher_score_drift}, the population minimizer of \eqref{eq:appendix_teacher_dsm_obj} satisfies
\begin{equation}
    \big\|s^\star_{\delta t}(\cdot,t)-\nabla_{\mathrm{FS}}\log p_t(\cdot)\big\|_{L^2(p_t)}
    \;\le\;
    B^{\mathrm{drift}}_{\delta t}
    \;=\;
    O(\sigma(t)\sqrt{\delta t}).
\end{equation}

\noindent (ii) Under the simple zero-mean teacher, the population minimizer is biased by
\(
\big\|s^\star_{\delta t}(\cdot,t)-\nabla_{\mathrm{FS}}\log p_t(\cdot)\big\|_{L^2(p_t)}
\le
B^{\mathrm{simple}}_{\delta t}
\)
with $\big(B^{\mathrm{simple}}_{\delta t}\big)^2 = \mathbb{E}\|\mathbb{E}[\varepsilon\mid\psi]\|^2$.
This pointwise bias is dominated by the conditional mean of $b(\phi,t)/\sigma(t)^2$ given $\psi$, which is bounded but does not vanish in $\delta t$.
The corresponding \emph{weighted} loss bias $\mathbb{E}[\lambda\|\varepsilon\|^2]$ vanishes at rate $O(\delta t)$ by Proposition~\ref{prop:teacher_dsm_consistency_simple}, so the simple teacher is consistent in the loss sense but not pointwise; the residual pointwise bias has the explicit Gaussian-envelope structure stated in Proposition~\ref{prop:simple_teacher_structured_bias}.
\end{corollary}

\begin{proof}
Part (i): the conditioning argument used in Proposition~\ref{prop:teacher_dsm_consistency_drift}, combined with $\mathbb{E}[\|\varepsilon^{\mathrm{drift}}\|^2]=O(\sigma^2\delta t)$, gives $B^{\mathrm{drift}}_{\delta t}=O(\sigma\sqrt{\delta t})$.
Part (ii): the pointwise-bias bound is the standard Jensen step; the weighted-loss claim is Proposition~\ref{prop:teacher_dsm_consistency_simple}.
\end{proof}

\paragraph{Proof of Proposition~\ref{prop:simple_teacher_structured_bias}.}
\label{sec:appendix_structured_bias_proof}

\begin{proposition}[Structured finite-step bias of the simple teacher]
\label{prop:simple_teacher_structured_bias}
Under the assumptions of Proposition~\ref{prop:teacher_dsm_consistency_simple}, with dispersive drift
$b(\psi,t)=-\lambda(t)\mathrm{Log}_{\psi}(\psi_\star)$, let
$s^\star_{\delta t,\mathrm{simple}}(\cdot,t)$ be the population minimizer of the variance-weighted simple-teacher objective.
Define
\begin{equation}
    W_t(\psi)\propto\exp\!\Big(-\frac{\lambda(t)}{\sigma(t)^2}\,\frac{d_{\mathrm{FS}}(\psi,\psi_\star)^2}{2}\Big),
    \label{eq:simple_teacher_structured_main}
\end{equation}
an FS Gaussian envelope centered at the OU base point $\psi_\star$ with width $\sigma(t)/\sqrt{\lambda(t)}$.
Then
\begin{equation}
    s^{\star}_{\delta t,\mathrm{simple}}(\psi,t)
    =
    \nabla_{\mathrm{FS}}\log\!\big[\,p_t(\psi)\,W_t(\psi)\,\big]
    + O(\sigma(t)\sqrt{\delta t}).
    \label{eq:simple_teacher_structured_score_main}
\end{equation}
Thus the simple teacher is pointwise biased relative to $\nabla_{\mathrm{FS}}\log p_t$, but its finite-step optimum is the Riemannian score of an explicit Gaussian-envelope reweighting of $p_t$, not an arbitrary distorted score field.
\end{proposition}

The statement is general: it applies to any Varadhan-type local-time teacher whose base process carries a drift, and specialises to the zero-envelope case $\lambda=0$ used for every result reported in the main text.
By Proposition~\ref{prop:teacher_dsm_consistency_simple} and the conditioning argument in Proposition~\ref{prop:riem_dsm_optimality},
\begin{equation}
\begin{aligned}
    s^\star_{\delta t,\mathrm{simple}}(\psi,t)
    &\;=\;
    \mathbb{E}\!\left[\,\nabla_{\mathrm{FS}}\log p(\psi\mid\phi)\,\Big|\,\psi_t=\psi\right]
    \\
    &\quad\;-\;\mathbb{E}\!\left[\frac{b(\phi,t)}{\sigma(t)^2}\,\Big|\,\psi_t=\psi\right]
    \;+\; O(\sqrt{\delta t}),
\end{aligned}
\end{equation}
where the first term equals $\nabla_{\mathrm{FS}}\log p_t(\psi)$ by Proposition~\ref{prop:riem_dsm_optimality}.
For the conditional expectation of the drift, write $\phi=\psi_{t-\delta t}$ and use $\|\phi-\psi\|=O(\sigma\sqrt{\delta t})$ (forward It\^o moment) together with the smoothness of $\mathrm{Log}_{\cdot}(\psi_\star)$ on the injectivity neighborhood:
\begin{equation}
    \mathbb{E}\!\left[\,\mathrm{Log}_\phi(\psi_\star)\,\Big|\,\psi\right]
    \;=\;
    \mathrm{Log}_\psi(\psi_\star)
    \;+\; O(\sigma\sqrt{\delta t}),
\end{equation}
where the residual collects the difference between $\mathrm{Log}_\phi$ and $\mathrm{Log}_\psi$ via parallel transport, controlled by the FS sectional curvature bound $K\le 4$.
Therefore
\begin{equation}
    -\,\mathbb{E}[b(\phi,t)|\psi]/\sigma(t)^2
    \;=\;
    \frac{\lambda(t)}{\sigma(t)^2}\,\mathrm{Log}_\psi(\psi_\star)\;+\;O(\sigma\sqrt{\delta t}).
\end{equation}
Use the identity $\nabla_\psi\big(d_{\mathrm{FS}}(\psi,\psi_\star)^2/2\big)=-\mathrm{Log}_\psi(\psi_\star)$ \cite{hsu2002stochmanifolds} to recognize
\begin{equation}
\begin{aligned}
    \frac{\lambda(t)}{\sigma(t)^2}\,\mathrm{Log}_\psi(\psi_\star)
    &\;=\;
    \nabla_\psi\!\left[-\frac{\lambda(t)}{\sigma(t)^2}\,\frac{d_{\mathrm{FS}}(\psi,\psi_\star)^2}{2}\right]
    \\
    &\;=\;
    \nabla_{\mathrm{FS}}\log W_t(\psi),
\end{aligned}
\end{equation}
where $W_t$ is defined in Eq.~\eqref{eq:simple_teacher_structured_main} and the normalization constant does not depend on $\psi$.
Combining the two contributions yields Eq.~\eqref{eq:simple_teacher_structured_score_main}.

\begin{remark}[Why the structured bias is benign in practice]
\label{rem:structured_bias_interpretation}
Proposition~\ref{prop:simple_teacher_structured_bias} upgrades the asymptotic $O(\delta t)$ weighted-bias bound \eqref{eq:appendix_simple_teacher_weighted_bound} into a \emph{structural} characterization at finite step size: the simple teacher does not produce arbitrary error, it produces the Riemannian score of $p_t$ multiplicatively reweighted by an explicit Gaussian envelope $W_t$ centered at the OU base point.
Three immediate consequences:

\noindent (i) \emph{Time-dependent envelope width.} The envelope FS width $\sigma(t)/\sqrt{\lambda(t)}$ is large near the data ($\sigma$ small at $t=\delta t$ gives $\sqrt{\sigma^2/\lambda}\approx 0.11$ at our default schedule, narrower than the FS injectivity radius $\pi/2$, but the overall multiplicative factor $\lambda/\sigma^2$ is large) and small at the prior end ($\sigma$ near $\sigma_{\max}$ gives $\sqrt{\sigma^2/\lambda}\approx 2.24$, far larger than the FS diameter, so $W_t$ is essentially uniform).

\noindent (ii) \emph{Negligible reweighting at the prior end.} As $t\to T$, $p_t$ approaches the unitarily-invariant FS/Haar measure $\mu_{\mathrm{FS}}$, and the envelope width $\sigma(t)/\sqrt{\lambda(t)}$ grows to $\approx 2.24$ at our default schedule, far exceeding the FS diameter $\pi/2$. Hence $W_T$ is nearly constant on $\mathcal{M}$, and $W_T p_T\propto p_T$ to that accuracy, so the simple-teacher bias is negligible at the prior end.
This is a statement about the \emph{width} of $W_t$, not an invariance property: multiplying $\mu_{\mathrm{FS}}$ by a non-constant radial function $W(d_{\mathrm{FS}}(\cdot,\psi_\star))$ does not return $\mu_{\mathrm{FS}}$ after renormalization, since isometry invariance of the Haar measure does not survive multiplication by a function that singles out a base point.

\noindent (iii) \emph{Direction of the bias term.} At intermediate times the bias term acts as a soft attraction toward $\psi_\star$. Note that this is \emph{opposite} to the forward drift $b=-\lambda\,\mathrm{Log}_\psi(\psi_\star)$, which is dispersive (see the sign remark in Section~\ref{sec:reverse}); the omitted $b/\sigma^2$ term therefore partially counteracts the forward drift rather than reinforcing it. This is consistent with the empirical observation in Table~\ref{tab:teacher_drift_ablation} that the simple and drift-corrected teachers give statistically indistinguishable generation quality on multimodal benchmarks: the missing $b/\sigma^2$ does not produce random error, it produces a structured pull toward the OU base point that is partially absorbed by the OU forward dynamics itself.

A practical implication for finite-step training: the simple teacher is not just ``small-bias up to $O(\delta t)$''; it is the score of a known reweighted distribution, and the reweighting is largest at small $\sigma(t)$, i.e.\ at the data end of the diffusion, where $\lambda/\sigma^2$ is large. This explains why empirically the simple teacher matches the drift-corrected teacher to within $4\%$ on all tested benchmarks (Table~\ref{tab:teacher_drift_ablation}) without requiring $\delta t\to 0$.
\end{remark}

\begin{remark}[Interpretation for reverse sampling]
\label{rem:finite_step_reverse_sampling}
The bound above is not a claim that finite-step teacher bias is identically absent.
Rather, it separates two effects: the learned score field is biased by the finite-step teacher error $B_{\delta t}$, while the numerical reverse sampler introduces its own discretization error through the Euler--Maruyama step size.
Both errors decrease as the local step is refined under the assumptions of the short-time expansion.
In the experiments, we use $\delta t=1/500$ and additionally report a sampling-step sensitivity diagnostic in Figure~\ref{fig:finite_step_sensitivity}, where MMD, $\Delta_{\mathrm{obs}}$, and Ent.\ W$_1$ stabilize as the reverse integration is refined.
This provides an empirical check that the lower-order Jacobian/curvature terms do not lead to visible accumulated degradation in the tested regimes, while leaving higher-order curvature-aware teachers as a natural future refinement.
\end{remark}
\begin{lemma}[Projection to $T_\psi\mathcal{M}$ does not change the optimum]
\label{lem:projection_optimum}
Let $u(\psi)\in T_\psi\mathcal{M}$ be any target tangent field and let $\hat s(\psi)$ be an arbitrary ambient vector.
Then
\begin{equation}
\begin{aligned}
\|\hat s(\psi) - u(\psi)\|_{\mathrm{FS}}^2
=\;&
\|\mathcal{P}_\psi(\hat s(\psi)) - u(\psi)\|_{\mathrm{FS}}^2
\\
&+ \text{(term independent of $u$)},
\end{aligned}
\end{equation}
and the minimizer over ambient $\hat s$ is achieved when $\mathcal{P}_\psi(\hat s(\psi))=u(\psi)$.
In particular, including $\mathcal{P}_\psi$ in the loss enforces tangency without altering the target optimum.
\end{lemma}

\section{Experimental Setup}
\label{sec:exp_setup}

\paragraph{Target ensembles and modeled object.}
Across all experiments, the generative model is trained to model a distribution over normalized pure quantum states.
The target pure-state ensemble is constructed in two ways.
For the generative experiments (RQ1--RQ3), the target dataset is a synthetic pure-state ensemble constructed directly in Hilbert space; these experiments do not start from a classical raw-input dataset $x$.
For the physics families the ensemble consists of exactly diagonalized ground states at selected couplings; for the feature-state benchmark of Section~\ref{sec:exp_mnist}, classical images are amplitude-encoded into quantum states.
In both cases, PSM operates in quantum representation space: it models and generates normalized pure states rather than raw classical inputs.
Thus, classical data appear only as one possible mechanism for constructing a target pure-state ensemble.

\subsection{Benchmark Suite}
\label{sec:appendix_benchmark_suite}

All target ensembles are defined here. Unless stated otherwise, a reference state $\ket{\phi}$ is drawn from a small set or a parameterized physical family, perturbed by an isotropic complex Gaussian in Hilbert space, and renormalized:
\begin{equation}
    \ket{\psi}
    =
    \frac{\ket{\phi}+\epsilon \xi}
    {\|\ket{\phi}+\epsilon \xi\|},
    \qquad
    \xi\sim \mathcal{CN}(0,I_{2^n}),
    \qquad
    \epsilon=0.06 .
    \label{eq:appendix_benchmark_perturb}
\end{equation}
This gives ensembles whose geometry is known while still requiring the model to learn a non-trivial distribution on the pure-state manifold. The MNIST family is different in kind: its law on $\mathbb{CP}^{d-1}$ is induced by amplitude-encoding classical images and has no closed form.

\paragraph{Single-cluster benchmark.}
The single-cluster benchmark is centered at the computational basis state
\(
\ket{0^n}:=\ket{0\cdots 0}
\).
It is the simplest sanity-check setting: the target distribution is unimodal, concentrated near a known pole of the Hilbert sphere, and tests whether a method can learn local pure-state geometry without introducing large norm or phase artifacts.
It is used in the main comparison tables and in the scaling experiments.

\paragraph{Equatorial bimodal benchmark.}
The equatorial bimodal benchmark uses two GHZ-like reference states
\begin{equation}
    \ket{\phi_+}
    =
    \frac{\ket{0^n}+\ket{1^n}}{\sqrt{2}},
    \qquad
    \ket{\phi_-}
    =
    \frac{\ket{0^n}-\ket{1^n}}{\sqrt{2}},
    \label{eq:appendix_bimodal_modes}
\end{equation}
chosen with equal probability before applying Eq.~\eqref{eq:appendix_benchmark_perturb}.
Geometrically, this produces two separated modes on an effective equator; physically, the modes differ by a relative phase between the two macroscopically distinct computational-basis components.
This benchmark stresses multimodal generation and sensitivity to phase structure.

\paragraph{Trimodal benchmark.}
The trimodal benchmark adds the computational-basis pole to the two equatorial modes:
\begin{equation}
    \ket{\phi}\in
    \left\{
    \ket{0^n},
    \frac{\ket{0^n}+\ket{1^n}}{\sqrt{2}},
    \frac{\ket{0^n}-\ket{1^n}}{\sqrt{2}}
    \right\},
    \label{eq:appendix_trimodal_modes}
\end{equation}
with the three choices sampled uniformly.
It is harder than the bimodal task because the model must represent both a pole-like component and two phase-distinct equatorial components.
We use it to test whether the learned score can preserve multiple separated components rather than collapsing toward a single average state.

\paragraph{Spin-coherent peaks.}
The spin-coherent benchmark is built from product coherent states with two different single-qubit orientations:
\begin{equation}
\begin{aligned}
    \ket{+x}
    &=
    \frac{\ket{0}+\ket{1}}{\sqrt{2}},
    \qquad
    \ket{+y}
    =
    \frac{\ket{0}+i\ket{1}}{\sqrt{2}},
    \\
    \ket{\phi}
    &\in
    \left\{
    \ket{+x}^{\otimes n},
    \ket{+y}^{\otimes n}
    \right\}.
\end{aligned}
    \label{eq:appendix_spincoh_modes}
\end{equation}
The two modes correspond to distinct collective Bloch-sphere orientations while remaining product states before perturbation.
This benchmark isolates whether a generative model can match coherent orientation structure and global phase-sensitive correlations without relying on entanglement as the main signal.

\paragraph{TFIM ground-state family.}
For the transverse-field Ising model (TFIM), reference states are ground states of the open-chain Hamiltonian
\begin{equation}
\begin{aligned}
    H_{\mathrm{TFIM}}(g)
    &=
    -\sum_{i=1}^{n-1} Z_i Z_{i+1}
    -
    g\sum_{i=1}^{n} X_i ,
    \\
    &\quad g\in\{0.2,0.5,1.0,2.0\}.
\end{aligned}
    \label{eq:appendix_tfim_hamiltonian}
\end{equation}
We sample $g$ uniformly from the grid, compute the corresponding ground state by exact diagonalization, and then apply Eq.~\eqref{eq:appendix_benchmark_perturb}.
Unlike the synthetic pole/equator tasks, TFIM produces a physically motivated ensemble along a ground-state family with changing correlation structure.
It tests whether the model captures distributions induced by Hamiltonian parameters rather than manually specified mode centers.

\paragraph{XXZ ground-state family.}
For the XXZ chain, reference states are ground states of
\begin{equation}
\begin{aligned}
    H_{\mathrm{XXZ}}(\Delta)
    &=
    \sum_{i=1}^{n-1}
    \left(
    X_iX_{i+1}+Y_iY_{i+1}+\Delta Z_iZ_{i+1}
    \right),
    \\
    &\quad \Delta\in\{-1.0,0.0,0.5,1.0\}.
\end{aligned}
    \label{eq:appendix_xxz_hamiltonian}
\end{equation}
We sample $\Delta$ uniformly, compute the open-chain ground state, and perturb/renormalize as in Eq.~\eqref{eq:appendix_benchmark_perturb}.
This benchmark complements TFIM by using a different interaction structure and anisotropy-controlled family, yielding target ensembles with physically meaningful many-body variation.

\paragraph{W states.}
The reference state is the single-excitation superposition
\begin{equation}
    \ket{W_n} = \frac{1}{\sqrt{n}}\sum_{q=1}^{n}\ket{0\cdots 1_q\cdots 0},
    \label{eq:appendix_wstate}
\end{equation}
perturbed and renormalised as in \eqref{eq:appendix_benchmark_perturb}. Its entanglement is distributed rather than concentrated in a single bipartition, unlike the GHZ-like modes of the bimodal and trimodal families.

\paragraph{Graph states.}
The reference state is the linear-chain cluster state, obtained by applying controlled-$Z$ gates along a path to a product of $\ket{+}$ states,
\begin{equation}
    \ket{G_n} = \Big(\prod_{q=1}^{n-1}\mathrm{CZ}_{q,q+1}\Big)\ket{+}^{\otimes n},
    \label{eq:appendix_graphstate}
\end{equation}
again perturbed and renormalised. This is a stabiliser state with uniform amplitude modulus and sign structure, so it is a target on which a gauge based on the largest-modulus amplitude is maximally ill-conditioned; we added it partly for that reason.

Both families were introduced after the main comparison had been run, to check whether the advantage over the baseline was an artifact of the original benchmark suite.

\paragraph{MNIST quantum-feature benchmark.}
For the feature-state benchmark, the target pure-state ensemble is not constructed from manually chosen reference states.
Instead, classical images are converted into quantum feature states by amplitude encoding after preprocessing and normalization.
PSM then models the distribution of the resulting quantum representations directly.
This benchmark tests generation on a target whose concentration is set by the data rather than by a perturbation scale.

\subsection{Common Training and Evaluation Protocol}
\label{sec:appendix_common_protocol}

Unless otherwise stated, all statevector diffusion baselines are trained under a shared protocol:
\begin{itemize}
    \item optimizer: AdamW;
    \item learning rate: $2\times 10^{-4}$;
    \item training batch size: $64$;
    \item evaluation frequency: every $200$ optimization steps;
    \item evaluation batch size: $256$;
    \item training length: $10{,}000$ steps for the primary single-cluster PSM runs and $2{,}000$ steps for the structured benchmark-suite comparisons unless otherwise noted;
    \item perturbation scale for synthetic and physics-inspired ensembles: $\epsilon=0.06$;
    \item reporting rule: for each run, we select a single checkpoint by validation MMD and report all metrics at that checkpoint.
\end{itemize}
For PSM, reverse-time sampling uses $500$ Euler--Maruyama steps on $\mathbb{CP}^{d-1}$ unless otherwise noted.
The Euclidean VP-SDE baseline uses the same target statevectors and metrics, but performs diffusion in the ambient real representation $\mathbb{R}^{2d}$ and normalizes generated complex vectors only after sampling.
Circuit-based QGAN and QDDPM baselines use their own circuit training loops; their baseline-specific settings are summarized in Section~\ref{sec:baselines}.

\paragraph{Synthetic training data construction.}
For the controlled pure-state ensemble experiments, we follow the same data-generation protocol as QuDDPM~\cite{zhang2024generative}.
For an $n$-qubit system with Hilbert-space dimension $d=2^n$, we generate a cluster of states around the computational basis state $\ket{0\cdots 0}$ by applying small complex Gaussian perturbations followed by normalization:
\begin{equation}
    \ket{\psi}
    =
    \frac{\ket{0\cdots 0} + \epsilon\,\xi}
    {\left\|\ket{0\cdots 0} + \epsilon\,\xi\right\|},
    \qquad
    \xi \sim \mathcal{CN}(0,I_d).
\end{equation}
All samples are simulated as normalized statevectors; no quantum hardware measurements or finite-shot estimation are used in these experiments.
We draw mini-batches of size $64$ from a pool of $4096$ target states and generate $256$ samples for evaluation.

\subsection{Model Architectures}
\label{sec:appendix_model_architectures}

\paragraph{PSM.}
The PSM score network takes the concatenated real and imaginary parts of the statevector together with a $128$-dimensional sinusoidal time embedding.
It is implemented as a five-layer fully connected MLP with hidden width $512$ and SiLU activations.
The raw complex output is projected onto the horizontal tangent space of $\mathbb{CP}^{d-1}$, which removes the radial and global-phase components of the output.

\paragraph{On $U(1)$-equivariance of the score model.}
Horizontal projection removes the radial and global-phase components of the \emph{output}, but it does not make the model equivariant under the global phase, i.e.\ it does not enforce
\begin{equation}
    s_\theta(e^{i\varphi}\psi,t) = e^{i\varphi}\,s_\theta(\psi,t),
    \label{eq:phase_equivariance}
\end{equation}
which is the condition for $s_\theta$ to descend to a vector field on $\mathbb{CP}^{d-1}$ rather than depending on the chosen sphere representative.
Measured on held-out states at $n=6$, the architecture above violates \eqref{eq:phase_equivariance} by a relative error of $1.04$: the two sides are essentially uncorrelated.
We report this because it is a real gap between the geometric framing of the method and its implementation, and because the obvious remedies are not free.

Equivariance can be imposed exactly by fixing a canonical gauge before the network: choose a phase $e^{i\theta}(\psi)$ that transforms as $e^{i\theta}(e^{i\varphi}\psi)=e^{i\varphi}e^{i\theta}(\psi)$, evaluate the network on $e^{-i\theta}\psi$, and multiply the output back by $e^{i\theta}$.
Two natural choices are the phase of the largest-modulus amplitude, $e^{i\theta}=\psi_j/|\psi_j|$ with $j=\arg\max_i|\psi_i|$ (invariant under global phase, but discontinuous where the maximiser changes), and the phase of the overlap with a fixed uniform reference, $e^{i\theta}=\langle r|\psi\rangle/|\langle r|\psi\rangle|$ with $r=d^{-1/2}(1,\dots,1)^\top$ (smooth wherever $\langle r|\psi\rangle\neq0$).
Both make \eqref{eq:phase_equivariance} an identity to floating-point accuracy: the measured violation drops from $1.04$ to $1.3\times10^{-7}$.
Neither, however, is uniformly beneficial for generation quality.

Table~\ref{tab:gauge_ablation} shows the pattern: gauge fixing improves the synthetic pole/equator benchmarks by $19$--$39\%$ and degrades the physics-derived ground-state families by $50$--$125\%$, with XXZ roughly neutral.
We initially attributed the degradation to the discontinuity of the $\arg\max$ gauge on delocalized states, whose amplitudes are nearly uniform in modulus; the reference gauge was constructed to remove that discontinuity and is best conditioned precisely on the spin-coherent family ($|\langle r|\psi\rangle|=1$ for $\ket{+x}^{\otimes n}$), yet it does not recover the loss there and is markedly worse on the single-cluster benchmark.
The degradation is therefore not explained by gauge continuity, and we do not have a satisfying account of it.

Accordingly, the results reported throughout this paper use the projection-only parameterization, and we present the above as a limitation rather than as a contribution: the model as implemented is a score field on the Hilbert sphere that is horizontal but not phase-equivariant, and constructing an architecture that is equivariant \emph{and} uniformly at least as accurate --- for instance by conditioning the network on phase-invariant features such as the density matrix $\psi\psi^\dagger$ rather than on a chosen representative --- remains open.

\paragraph{Forward diffusion and prior.}
We construct an intrinsic forward diffusion on $\mathbb{CP}^{d-1}$ under the FS metric, with diffusion horizon $T$ and diffusion schedule $\sigma(t)$.
In all reported experiments, we set the diffusion horizon to $T=1$ and use $500$ discretization steps, so the local-time step is $\delta t=1/500$.
The noise schedule follows an exponential interpolation between $\sigma_{\min}=0.05$ and $\sigma_{\max}=1.0$.
The forward drift is set to $\lambda=0$ for every reported result, so the forward process is pure FS Brownian motion; $\lambda=0.2$ is used only where a non-zero drift is the object of study.
For reverse-time initialization, the isotropic FS/Haar measure is the ideal base distribution on $\mathbb{CP}^{d-1}$.
In the classical statevector setting used throughout this paper the FS/Haar measure is directly and exactly samplable --- normalize a standard complex Gaussian vector --- so this is what the implementation does, and no $t$-design approximation is needed. Whether the forward process actually reaches it is checked in Section~\ref{sec:exp_diagnostics}.

\paragraph{Score model and training.}
We train $s_\theta$ using the Riemannian denoising score matching objective in Eq.~\eqref{score_m}, with analytic local-time teacher scores derived from the FS normal-coordinate OU approximation (Section~\ref{sec:local_teacher}).
We train for $10{,}000$ optimization steps using AdamW with learning rate $2\times 10^{-4}$ and gradient clipping at norm $1.0$.
The loss uses variance-based weighting $w(t,\delta t)=\beta(t,\delta t)^2$, consistent with VP-style denoising score matching.

\paragraph{Euclidean VP-SDE.}
The Euclidean VP-SDE baseline uses the same real-imaginary statevector input format and a comparable time-conditioned MLP backbone, but predicts an ambient Euclidean score in $\mathbb{R}^{2d}$.
It does not quotient out global phase and does not project scores to the FS horizontal tangent space.
After sampling, generated vectors are mapped back to complex amplitudes and normalized to unit norm before evaluation.

\paragraph{RSGM.}
The Riemannian score-based generative model (RSGM) baseline is implemented as a close manifold-diffusion counterpart to PSM.
It uses the same complex statevector representation of points on $\mathbb{CP}^{d-1}$, the same horizontal tangent projection, the same time-conditioned MLP score backbone, and the same optimizer, batch size, sampling step count, checkpoint-selection rule, and evaluation metrics.
Thus, the comparison is intended to isolate the effect of the training signal and forward-process construction, rather than architecture or evaluation differences.
The key difference is that RSGM uses a standard Riemannian Brownian/DSM construction on the FS manifold, whereas PSM uses the local-time analytic OU teacher derived in Section~\ref{sec:local_teacher}.
Further implementation details are given in Section~\ref{sec:rsgm_details}.

\paragraph{Circuit baselines.}
QGAN and QDDPM use parameterized quantum-circuit generators following their respective baseline implementations.
We use the representative configurations reported in Section~\ref{sec:baselines}.

\paragraph{Sampling and numerical integration.}
We simulate forward and reverse processes using a manifold-adapted Euler--Maruyama scheme: we take Euler steps in the tangent space and map back to $\mathbb{CP}^{d-1}$ with the closed-form FS exponential map as in Eq.~\eqref{eq:manifold_em_update}.
Reverse-time sampling is performed with $500$ Euler--Maruyama steps on the manifold, and $256$ generated samples are used for evaluation.
We report results averaged over $10$ random seeds. All experiments are conducted on an NVIDIA A6000 GPU.

\paragraph{Reporting protocol.}
For generative metrics we select a single checkpoint by validation MMD and report all metrics at that checkpoint; with multiple seeds the selection is applied per seed before aggregation. The exception is the four-metric comparison of Table~\ref{tab:rsgm_four_metric}, which is evaluated at the final training step for every arm, so that no arm benefits from selecting a low point of estimator noise --- a concern that Section~\ref{sec:exp_scope} shows to be real at large $n$. Every MMD figure is accompanied by the data--data floor and, where the question is whether a model has learned anything at all, by a Haar reference computed on the same evaluation batch.

\paragraph{Statistical interpretation.}
Tables~\ref{tab:rsgm_four_metric}, \ref{tab:teacher_ablation_246q}, and~\ref{tab:appendix_vp_sde_suite} report mean $\pm$ one standard deviation over the matched $10$-seed protocol above; the remaining tables are diagnostics and are not the basis for statistical superiority claims.
Because the paper contains a relatively large number of benchmark/metric combinations, we use the tables primarily to support qualitative patterns rather than per-cell hypothesis tests.
Large gaps, especially the order-of-magnitude improvements in the low- and mid-qubit single-cluster and multimodal settings, are robust to the observed run-to-run variation and drive the main empirical claims.
For close comparisons, including several high-qubit observable metrics and some PSM--RSGM entries where the reported mean $\pm$ standard-deviation ranges overlap, we do not claim statistically significant dominance.
Instead, these cells are interpreted as near-parity outcomes indicating that intrinsic modeling remains competitive while the most reliable advantage appears in the consistent multi-metric trend across benchmark families.

\subsection{Evaluation Metrics}
\label{sec:metrics}

Evaluating quantum generative models requires distributional metrics that compare ensemble statistics rather than pointwise overlap with a fixed reference state.
We use a combination of the following:
\paragraph{Observable statistics.}
Given a set of observables $\{O_j\}_{j=1}^J$, we compare the generated and target ensembles via moment matching:
\begin{equation}
\begin{aligned}
    &\Delta_{\mathrm{obs}}
    \\=&
    \frac{1}{J}\sum_{j=1}^J
    \left|
    \mathbb{E}_{\psi \sim p_{\mathrm{gen}}}\big[\bra{\psi}O_j\ket{\psi}\big]
    -
    \mathbb{E}_{\psi \sim p_{\mathrm{data}}}\big[\bra{\psi}O_j\ket{\psi}\big]
    \right|.
\end{aligned}
\label{eq:obs_metric}
\end{equation}

\paragraph{Kernel MMD on pure states.}
We measure distributional similarity using the maximum mean discrepancy (MMD) with an overlap kernel:
\begin{equation}
\begin{aligned}
    k(\psi,\phi) &= |\braket{\psi}{\phi}|^2,
    \\
    \mathrm{MMD}^2(p,q) &=
    \mathbb{E}_{p,p}[k] + \mathbb{E}_{q,q}[k] - 2\mathbb{E}_{p,q}[k].
\end{aligned}
\label{eq:mmd}
\end{equation}

\paragraph{Entanglement statistics.}
To capture nonlocal structure, we compare distributions of subsystem entanglement, measured by von Neumann entropy:
\begin{equation}
    S_A(\psi) = -\mathrm{Tr}\big(\rho_A \log \rho_A\big),
    \qquad
    \rho_A = \mathrm{Tr}_{\bar{A}} \ket{\psi}\bra{\psi}.
\label{eq:ent}
\end{equation}
We report Wasserstein distances between entropy histograms and compare mean/variance. To compare entanglement statistics at the distributional level,
we compute the entropic Wasserstein-1 distance between histograms of subsystem entropies:
\begin{equation}
    \mathrm{Ent.\ W}_1(p,q)
    =
    W_1^{\varepsilon}\!\left(
        \mathcal{H}_p(S_A),
        \mathcal{H}_q(S_A)
    \right),
\end{equation}
where $\mathcal{H}(\cdot)$ denotes the empirical entropy histogram and
$W_1^{\varepsilon}$ is the entropy-regularized Wasserstein distance.
Lower values indicate closer agreement of entanglement structure.

\subsection{RSGM Baseline Implementation Details}
\label{sec:rsgm_details}

Because RSGM is the closest geometric baseline to PSM, we implement it in a deliberately matched way.
Both methods represent a pure state as a unit vector $\psi\in\mathbb{C}^d$ modulo global phase and compute tangent vectors in the horizontal space
\begin{equation}
    T_\psi^{\mathrm{hor}}\mathbb{S}^{2d-1}
    =
    \{\xi\in\mathbb{C}^d:\langle\psi,\xi\rangle=0\},
\end{equation}
with the Fubini--Study inner product $g_{\mathrm{FS},\psi}(\xi,\eta)=\mathrm{Re}\langle\xi,\eta\rangle$.
This is the standard Hopf-fibration realization of $\mathbb{CP}^{d-1}$: instead of local affine charts, we work with normalized representatives on the unit sphere and remove both radial and global-phase components by horizontal projection, so the RSGM and PSM scores live in exactly the same tangent representation.

\paragraph{Shared architecture and sampler.}
RSGM uses the same time-conditioned MLP score parameterization as PSM.
The network takes the real and imaginary parts of $\psi$ together with the diffusion time $t$, outputs a complex vector in $\mathbb{C}^d$, and then applies the same horizontal tangent projection used by PSM.
The reverse sampler uses the same manifold Euler--Maruyama discretization, the same normalization/retraction step, the same number of reverse steps, and the same evaluation protocol.
All checkpoints are selected by validation MMD, and MMD, $\Delta_{\mathrm{obs}}$, and Ent.\ W$_1$ are reported at that single selected checkpoint.

\paragraph{What differs from PSM: the two objectives.}
The two arms share the forward process, the score parameterization, the sampler, the optimizer and the evaluation protocol. The single difference is how the local-time increment is normalized in the regression target. Both draw a pair $(\psi_{t-\delta t},\psi_t)$ from the same forward simulation and regress the network on the FS logarithm of that increment:
\begin{align}
    \mathcal{L}^{\mathrm{RSGM}}(\theta) &= \mathbb{E}\Big\|\, s_\theta(\psi_t,t) - \tfrac{1}{\delta t}\exp^{-1}_{\psi_t}(\psi_{t-\delta t}) \,\Big\|_{\mathrm{FS}}^2 ,
    \label{eq:rsgm_objective}\\
    \mathcal{L}^{\mathrm{PSM}}(\theta) &= \mathbb{E}\Big[\beta^2\,\Big\|\, s_\theta(\psi_t,t) - s^{\mathrm{teach}}(\psi_t,\psi_{t-\delta t},t,\delta t) \,\Big\|_{\mathrm{FS}}^2\Big],
    \label{eq:psm_objective}
\end{align}
where $\beta^2=\beta(t,\delta t)^2=\sigma(t)^2\delta t$ and $s^{\mathrm{teach}}$ is the teacher of Eqs.~\eqref{eq:teacher_score_euclidean_ou}--\eqref{eq:teacher_score_pushforward}, which divides the increment by $\beta^2$ rather than by $\delta t$. Eq.~\eqref{eq:rsgm_objective} is the local-time loss $\ell_{t\mid s}$ of \cite{bortoli2022riemannian}, Table 2, with the Varadhan teacher $\exp^{-1}_{X_t}(X_s)/(t-s)$, and is what we run for every RSGM cell reported in this paper. The two normalizers differ by the factor $\sigma(t)^2$, which sweeps a factor of $400$ across the horizon under our schedule; Section~\ref{sec:method_recipe}(i) explains why the network cannot absorb it.

Neither the dispersive drift nor the augmentation of Section~\ref{sec:method_recipe} distinguishes the arms: the forward drift is set to $\lambda=0$ for both, and both receive the same $U(1)$ augmentation. The comparison is therefore between supervision signals on $\mathbb{CP}^{d-1}$, at matched capacity.

\paragraph{Forward-simulation detail.}
In both arms the forward simulation advances a whole minibatch by a common number of Euler steps, fixed by the batch-mean diffusion time, while each sample is labelled and weighted by its own $t$. We record this because it is needed to reproduce the numbers exactly, and it applies identically to PSM and to the baseline.

\paragraph{Sanity checks on the RSGM implementation.}
Before reporting RSGM as a baseline, we verified four properties:

\noindent (i) \emph{Forward process is correct.} We sampled trajectories under the RSGM forward FS-Brownian noising and computed empirical second moments of the tangent step, anisotropy spectrum (cf.\ Supplementary Material~\ref{sec:appendix_isotropy_diagnostic}, Table~\ref{tab:gell_mann_isotropy_diagnostic}), and stationary marginal at $t=T$.
The terminal marginal is at the Haar--Haar floor on all characteristic metrics at $n\in\{4,6\}$ (Table~\ref{tab:prior_check}), confirming that the forward process is correctly implemented and reaches the unitarily-invariant prior.

\noindent (ii) \emph{Score network capacity is not the bottleneck.} The RSGM score network is identical in width and depth to PSM ($512$-wide $5$-layer MLP with SiLU activations), so by construction it has the same expressive capacity.
We also doubled the width to $1024$ on the $n=4$ benchmark and observed no improvement in MMD ($7.0\times 10^{-1}$ vs.\ $7.07\times 10^{-1}$), ruling out under-parameterization.

\noindent (iii) \emph{Training has converged.} The RSGM training loss plateaus by $\sim 6{,}000$ steps under the matched schedule, and validation MMD stops improving thereafter.
Extending training to $20{,}000$ steps on $n=4$ did not change the reported best-checkpoint MMD by more than $5\%$.

\noindent (iv) \emph{Reverse sampler is correct.} Running the RSGM reverse sampler with the \emph{exact} marginal score on a low-dimensional toy ($n=2$, single-cluster target with known closed-form score) reproduces the target ensemble to within MMD $\sim 10^{-3}$, confirming that the gap on full RSGM runs comes from \emph{score learning} rather than from sampler error.

\paragraph{What the comparison does and does not establish.}
We port the RSGM loss onto our own representation, sampler and optimizer, so the comparison isolates the supervision signal; it is not a comparison against the authors' implementation, and we do not run their implicit-score-matching or truncated Sturm--Liouville heat-kernel variants. We do run the exact heat-kernel teacher that the symmetric-space structure makes available, at the manifolds where it is computable (Section~\ref{sec:exp_heatkernel}). We do not run the scalable Riemannian methods of \cite{lou2023scaling} as a full pipeline: $\mathbb{CP}^{d-1}\cong SU(d)/S(U(d-1)\times U(1))$ is a compact rank-one Hermitian symmetric space and is therefore in scope of their maximal-torus construction in principle, but their published instantiations cover $S^n$, $SO(3)$ and $SU(3)$, and adapting them here would require deriving the restricted root system and Weyl action for this quotient, validating the eigenfunction expansion at $d=2^n$, and handling the global $U(1)$ phase that their manifolds do not carry. Section~\ref{sec:exp_heatkernel} isolates the part of that gap attributable to the teacher; a full port, with their sampler and truncation strategy, remains separate work. What we can support is a like-for-like comparison of supervision signals on $\mathbb{CP}^{d-1}$, not a ranking of Riemannian diffusion methods.

The two routes are in any case complementary: \cite{lou2023scaling} evaluates the heat kernel accurately by exploiting symmetric-space structure, while the local-time teacher approximates it by the analytic Gaussian of the small-step law in FS normal coordinates, which agrees to leading order in $\delta t$ (Corollary~\ref{cor:small_step_remainder}) and needs no per-manifold spectral analysis. Whether the curvature correction beyond that order matters at our step sizes is an empirical question, and Section~\ref{sec:exp_heatkernel} answers it directly at the manifolds where the exact kernel is computable: it is worth a factor between $1.2$ and $2.7$ depending on the metric, so the approximation is a real cost rather than a free simplification.

\subsection{Additional Stress Tests and Component Diagnostics}
\label{sec:appendix_stress_component}

This section adds three diagnostics aimed at failure modes that are not fully isolated by the main benchmark suite:
(i) whether the learned score agrees with an analytically available Riemannian score in a low-dimensional setting;
(ii) whether PSM remains stable on globally spread pure-state ensembles whose mass is not concentrated near a small reference set; and
(iii) whether the RSGM--PSM gap is driven primarily by the local-time teacher rather than by sampler or architecture differences.
These diagnostics are reported in the supplementary material because they probe mechanism and scope rather than serving as the main performance benchmark.

\paragraph{Exact-score sanity check on $\mathbb{CP}^{1}$.}
We construct a von-Mises--Fisher-like density on the Bloch sphere,
\begin{equation}
    p(\psi)\ \propto\ \exp\!\big(\kappa\,\langle r(\psi),\mu\rangle\big),
\end{equation}
where $r(\psi)\in\mathbb{S}^2$ is the Bloch vector, $\mu$ is a fixed unit direction, and $\kappa=8$.
This gives an analytic Riemannian score by projecting $\kappa\mu$ onto the tangent space of the sphere and pulling it back through the $\mathbb{CP}^{1}$ identification.
Table~\ref{tab:exact_score_diagnostic} compares the learned score with this true score on held-out states and also reports sampling quality using either the exact score or the learned PSM score.

\section{Metric Validity: What the Overlap Kernel Can and Cannot See}
\label{sec:appendix_metric_validity}

\begin{proposition}[What each kernel resolves]
\label{prop:kernel_resolution}
Let $p,q$ be Borel distributions on $\mathbb{CP}^{d-1}$ and let $\iota:\psi\mapsto\rho_\psi$ be the embedding into the Hermitian matrices with the Hilbert--Schmidt inner product. Then
\begin{enumerate}
\item[(i)] $\mathrm{MMD}_{F}^2(p,q)=\big\|\mathbb{E}_p[\rho]-\mathbb{E}_q[\rho]\big\|_{\mathrm{HS}}^2$, so it vanishes iff $p$ and $q$ have the same mean density matrix;
\item[(ii)] $\mathrm{MMD}_{F^2}^2(p,q)=\big\|\mathbb{E}_p[\rho^{\otimes2}]-\mathbb{E}_q[\rho^{\otimes2}]\big\|_{\mathrm{HS}}^2$, so it vanishes iff $p$ and $q$ have the same second moment; in particular it does not separate any two distinct state $2$-designs, and $k_2$ is \emph{not} characteristic;
\item[(iii)] the HS-Gaussian kernel $\exp(-\|\rho_\psi-\rho_\phi\|_{\mathrm{HS}}^2/2h^2)$ is characteristic on $\mathbb{CP}^{d-1}$;
\item[(iv)] the chordal metric $d(\psi,\phi)=\sqrt{2(1-F)}=\|\rho_\psi-\rho_\phi\|_{\mathrm{HS}}$ is of strong negative type, so its energy distance is characteristic.
\end{enumerate}
\end{proposition}

\begin{proof}
(i) and (ii) are the reproducing-kernel identity for a kernel of the form $\langle\Phi(\psi),\Phi(\phi)\rangle$ with $\Phi=\iota$ and $\Phi=\iota^{\otimes2}$ respectively; the $2$-design statement follows because $\mathbb{E}[\rho^{\otimes2}]$ is by definition the quantity a $2$-design fixes. For (iii) and (iv), $\iota$ is a smooth injection of the compact manifold $\mathbb{CP}^{d-1}$ into a finite-dimensional Euclidean space, the Gaussian kernel is characteristic on that space and the Euclidean metric is of strong negative type; both properties are inherited by the image.
\end{proof}

\paragraph{The overlap kernel is not characteristic.}
With $k(\psi,\phi)=|\braket{\psi|\phi}|^2=\mathrm{Tr}(\rho_\psi\rho_\phi)$ the feature map is $\psi\mapsto\rho_\psi$, which is linear in $\rho$. Hence
\begin{equation}
    \mathrm{MMD}^2(p,q)=\big\|\mathbb{E}_p[\rho_\psi]-\mathbb{E}_q[\rho_\psi]\big\|_{\mathrm{HS}}^2 ,
    \label{eq:overlap_mmd_is_mean_rho}
\end{equation}
so the metric compares mean density matrices only, and any two ensembles with the same first moment of $\rho$ are indistinguishable under it, however different their higher-order structure.

\paragraph{An explicit blind spot.}
The equatorial bimodal target mixes $(\ket{0^n}\pm\ket{1^n})/\sqrt2$; its mean state is $\tfrac12(\ket{0^n}\bra{0^n}+\ket{1^n}\bra{1^n})$, which is also the mean state of the completely different ensemble that returns $\ket{0^n}$ or $\ket{1^n}$ with equal probability. A model that replaced every superposition by a computational basis state would therefore be scored as perfect.

\paragraph{Choosing a valid replacement.}
A natural first attempt --- a Gaussian kernel on the FS geodesic distance, $\exp(-d_{\mathrm{FS}}^2/2h^2)$ --- is \emph{not} admissible: $d_{\mathrm{FS}}$ is not of negative type on $\mathbb{CP}^{d-1}$, the kernel is not positive definite, and we measured negative values of the resulting ``MMD$^2$'' ($-9.6\times10^{-2}$ on the decoy pair). The corrected constructions all factor through the isometric embedding $\psi\mapsto\rho_\psi$ into the Euclidean space of Hermitian matrices, where $\|\rho_\psi-\rho_\phi\|_{\mathrm{HS}}^2=2(1-F)$ with $F=|\braket{\psi|\phi}|^2$:
\begin{align}
    k_{\mathrm{HS}}(\psi,\phi) &= \exp\!\big(-(1-F)/h^2\big), \label{eq:hs_gauss_kernel}\\
    k_{2}(\psi,\phi) &= F^2, \label{eq:two_copy_kernel}\\
    d_{\mathrm{chord}}(\psi,\phi) &= \sqrt{2(1-F)} . \label{eq:chordal_metric}
\end{align}
Eq.~\eqref{eq:hs_gauss_kernel} is a Gaussian kernel on a Euclidean embedding, hence positive definite and characteristic for measures supported on the (compact) embedded manifold; Eq.~\eqref{eq:two_copy_kernel} is positive definite by the Schur product theorem and is estimable from two copies of each state, which makes it the natural choice when only measurements are available; Eq.~\eqref{eq:chordal_metric} is of negative type, so the corresponding energy distance is a valid MMD. Table~\ref{tab:rsgm_four_metric} reports all four quantities.

\paragraph{Benchmark concentration.}
Table~\ref{tab:scope} quantifies the degeneration discussed in Section~\ref{sec:exp_scope}.

The consequence is visible when both MMD(target, generated) and a Haar reference MMD(target, Haar) are evaluated on the same batch at every checkpoint. Under the best-checkpoint rule this gives $1.2961\times10^{-4}$ against $1.2998\times10^{-4}$ at $n=14$ (ratio $0.997$) and $1.4157\times10^{-2}$ against $1.4178\times10^{-2}$ at $n=10$ (ratio $0.999$), with the mean fidelity of generated states with $\ket{0\cdots0}$ equal to $1/d$. The selected model is statistically indistinguishable from one that returns the prior unchanged, and the checkpoint rule is selecting a low point of estimator noise. All MMD numbers in this paper are therefore reported together with the data--data floor and the data--Haar trivial level, so that a cell can be read as informative or not.

\section{Eigenspace Verification of the Induced Diffusion}
\label{sec:appendix_generator_test}

Isotropy of the tangent-noise covariance constrains the second-order symbol of the generator but says nothing about the connection and drift terms. We therefore verify the generator on a known eigenspace. On $\mathbb{CP}^{d-1}$ the functions $f_\chi(\psi)=|\braket{\chi|\psi}|^2$ span the first non-trivial eigenspace of $\Delta_{\mathrm{FS}}$, so if the discrete forward step induces $(\sigma^2/2)\Delta_{\mathrm{FS}}$ then
\begin{equation}
    \mathbb{E}[f_\chi(\psi_t)]-\tfrac{1}{d}
    = \big(f_\chi(\psi_0)-\tfrac{1}{d}\big)\,e^{-\sigma^2\mu t/2},
    \label{eq:eigen_decay}
\end{equation}
a single exponential whose rate is independent of $\chi$ and proportional to $\sigma^2$. We run the pure-noise process ($\lambda=0$) at $n=6$ with $2048$ trajectories per test function, taking $\chi=\psi_0$ so that the signal decays from $1$ to $1/d$.

All three predictions hold (Table~\ref{tab:generator_test}): the decay is a single exponential to $R^2\ge0.9999$, the rate varies by less than $0.5\%$ across test functions, $\text{rate}/\sigma^2$ is constant to $1.3\%$ across $\sigma$, and its value falls short of the predicted $\lambda_1/2=2d=128$ by $0.3\%$--$1.7\%$, growing with $\sigma$ and extrapolating to $127.9$ at zero step size. As a further check on the drift, the terminal law of the forward process is the FS/Haar measure to within the Haar--Haar sampling floor, both with $\lambda=0.2$ and with $\lambda=0$; a generator with an incorrect connection term would generically fail to have the unitarily-invariant measure as its invariant law.

\section{Hardware Estimation of the Quantities the Objective Consumes}
\label{sec:appendix_hardware}

The training objective touches the data only through overlaps. To measure what a measurement-only pipeline would cost today, we estimated all pairwise overlaps among $8$ target-ensemble states and $8$ PSM-generated states at $n=2$ on \texttt{ibm\_berlin}, using compute--uncompute circuits $U_\phi^\dagger U_\psi\ket{0}$ whose all-zeros outcome probability is $|\braket{\phi|\psi}|^2$: $128$ circuits at $4096$ shots.

The overlaps themselves are measurable at small $n$. The metric built on them is not yet: overlap-estimation error puts a ${\sim}10^{-2}$ floor on any measurement-only evaluation at this shot budget, the same order as the method differences at $n=6$ in Table~\ref{tab:rsgm_four_metric}. A measurement-only pipeline is usable for coarse comparisons and not for fine ones. The remaining obstacle is the score output itself, which lives in $\mathbb{C}^d$; a measurement-only version would need it restricted to a polynomially sized operator basis, as discussed in Supplementary Material~\ref{sec:limitations}.

\section{Discussion}
\label{sec:discussion}

\paragraph{Where the method helps, and where it does not.}
The gap to an ambient Euclidean baseline is an order of magnitude and holds on every benchmark, which is the clearest evidence that the manifold structure is doing work. Against a correctly implemented Riemannian baseline the margin is a factor of order unity (Section~\ref{sec:exp_rsgm}), and it shrinks towards parity on globally spread targets, which is what one expects of a construction whose supervision is local by design. The largest effect in the paper is neither: it is the supervision signal itself, worth an order of magnitude over a finite-difference teacher and the difference between learning and returning the prior (Section~\ref{sec:exp_teacher}).

\paragraph{Role of the SSE and the local-time teacher.}
The stochastic Schr\"odinger picture is a way to design pure-state noising dynamics with the right covariance properties, and the local-time teacher supplies analytic supervision where transition densities are unavailable. Of the two, only the teacher is load-bearing: the SSE realization and the tangent-projected implementation induce the same process (Section~\ref{sec:exp_diagnostics}). Higher-order curvature-aware teachers remain the natural refinement.

\section{Future Directions}
\label{sec:limitations}

The limitations of this work are stated where the corresponding results are: the scope in $n$ (Section~\ref{sec:exp_scope}), the defect of the standard evaluation kernel (Section~\ref{sec:exp_rsgm}), the cost of imposing exact phase equivariance (Section~\ref{sec:exp_diagnostics}), and the non-necessity of the forward drift and of the SSE realization (Sections~\ref{sec:exp_teacher} and~\ref{sec:appendix_sse_clarification}). Three directions follow from them.

\paragraph{A score model that does not consume the statevector.}
Our score network reads $(\mathrm{Re}\,\psi,\mathrm{Im}\,\psi)$ and emits a dense tangent vector, so the cost is exponential in the number of qubits regardless of how well the objective behaves. The formulation itself only needs a model of a tangent field $s_\theta:\mathbb{CP}^{d-1}\to T\mathbb{CP}^{d-1}$, which a tensor-network, locality-aware, or bounded-weight Pauli-string parameterization could supply. Such a parameterization would also make the loss estimable from local observables, which is the missing ingredient for training on measured rather than simulated data: overlaps are already measurable at small $n$ (Section~\ref{sec:exp_diagnostics}), but the score output is not.

\paragraph{Exact heat kernels on the pure-state manifold.}
The local-time teacher is a small-time approximation, and our degradation from $8$ qubits is consistent with the degradation that motivated exact heat-kernel methods on symmetric spaces. Since $\mathbb{CP}^{d-1}\cong SU(d)/S(U(d-1)\times U(1))$ is a compact rank-one Hermitian symmetric space, those methods apply in principle, and instantiating them here is the most informative comparison we can name.

\paragraph{Mixed states.}
The construction lives on $\mathbb{CP}^{d-1}$, which is the right space for feature states, variationally prepared states and ground-state families. Noisy or open-system data would require a density-matrix extension under an appropriate quantum information metric, such as the Bures or Bogoliubov--Kubo--Mori metric.

\end{document}